\documentclass[acmsmall]{acmart}

\AtBeginDocument{%
  }

\usepackage{microtype}
\usepackage{graphicx}
\usepackage{algorithm}
\usepackage{algpseudocode}
\usepackage{diagbox}
\usepackage{booktabs} 
\usepackage{amsmath}
\usepackage{amsthm}
\usepackage{color, colortbl}
\usepackage{xcolor}
\usepackage{hyperref}
\usepackage[capitalize]{cleveref}
\usepackage{nicefrac}

\usepackage{enumitem}
\usepackage{bm}
\usepackage{bbm}
\usepackage{graphicx}
\usepackage{thmtools}
\usepackage{thm-restate}
\usepackage{mathtools}
\usepackage[skip=0pt]{caption}
\usepackage[skip=0pt]{subcaption}
\usepackage{booktabs} 
\usepackage{enumitem}
\usepackage{accents}

\usepackage{threeparttable, array, float} 

\definecolor{red}{HTML}{E51400}  
\definecolor{blue}{HTML}{0050EF} 
\definecolor{green}{HTML}{008A00} 
\definecolor{purple}{HTML}{AA00FF} 
\definecolor{dark-red}{rgb}{0.4, 0.15, 0.15}
\definecolor{dark-blue}{rgb}{0.15, 0.15, 0.4}
\definecolor{medium-red}{rgb}{0.5, 0, 0}
\definecolor{medium-blue}{rgb}{0, 0, 0.5}
\definecolor{light-red}{rgb}{0.7, 0, 0}
\definecolor{light-blue}{rgb}{0, 0, 0.7}

\hypersetup{
    colorlinks=true,
    linkcolor=blue,
    filecolor=magenta,      
    urlcolor=blue,
}

\newtheorem{theorem}{\bf Theorem}
\newtheorem{lemma}{\bf Lemma}

\newtheorem{condition}{Condition}

\newtheorem{proposition}{\bf Proposition}

\newtheorem{assumption}{Assumption}

\renewcommand*{\proofname}{\textnormal{\bf Proof}}

\theoremstyle{definition}
\newtheorem{remark}{\bf Remark}

\definecolor{red}{HTML}{E51400} 
\definecolor{blue}{HTML}{0050EF} 
\definecolor{green}{HTML}{008A00} 
\definecolor{purple}{HTML}{AA00FF} 
\definecolor{orange}{HTML}{FF7F00}
\definecolor{gray}{HTML}{848482}
\definecolor{Gray}{gray}{0.85}
\definecolor{LightGray}{gray}{0.96}

\usepackage{xspace}
\usepackage{comment}

\usepackage{fontawesome5}

\usepackage[colorinlistoftodos,prependcaption,textsize=tiny,textwidth=2cm,color=orange]{todonotes}

\newcommand{\rev}[1]{{{#1}}}

\newcommand{\ccmab}{{C$^2$MAB}}

\DeclareMathOperator*{\argmin}{argmin}
\DeclareMathOperator*{\argmax}{argmax}

\newcommand{\norm}[1]{\left\lVert#1\right\rVert}

\newcommand{\cS}{\mathcal{S}}
\newcommand{\abs}[1]{\left| #1 \right|}

\newcommand{\R}{\mathbb{R}}

\newcommand{\E}{\mathbb{E}}

\newcommand{\Var}{{\rm Var}}

\newcommand{\I}{\mathbb{I}}

\newcommand{\bA}{\boldsymbol{A}}

\newcommand{\bB}{\boldsymbol{B}}

\newcommand{\bI}{\boldsymbol{I}}

\newcommand{\bL}{\boldsymbol{L}}
\newcommand{\bM}{\boldsymbol{M}}
\newcommand{\bH}{\boldsymbol{H}}
\newcommand{\bG}{\boldsymbol{G}}
\newcommand{\bg}{\boldsymbol{g}}

\newcommand{\bS}{\boldsymbol{S}}

\newcommand{\bV}{\boldsymbol{V}}

\newcommand{\bX}{\boldsymbol{X}}

\newcommand{\bx}{\boldsymbol{x}}

\newcommand{\bZ}{\boldsymbol{Z}}

\newcommand{\bmu}{\boldsymbol{\mu}}

\newcommand{\btheta}{\boldsymbol{\theta}}
\newcommand{\boldeta}{\boldsymbol{\eta}}
\newcommand{\boldzeta}{\boldsymbol{\zeta}}
\newcommand{\bphi}{\boldsymbol{\phi}}

\newcommand{\cA}{\mathcal{A}}
\newcommand{\cB}{\mathcal{B}}
\newcommand{\cC}{\mathcal{C}}
\newcommand{\cD}{\mathcal{D}}
\newcommand{\cE}{\mathcal{E}}
\newcommand{\cF}{\mathcal{F}}

\newcommand{\cH}{\mathcal{H}}

\newcommand{\cL}{\mathcal{L}}

\newcommand{\cM}{\mathcal{M}}

\newcommand{\cQ}{{\mathcal{Q}}}

\newcommand{\defeq}{\vcentcolon=}

\makeatletter
\newcommand*{\rom}[1]{\expandafter\@slowromancap\romannumeral #1@}
\makeatother

\newcommand{\lbd}{d\log \left(\nicefrac{4(1+TK)}{\delta}\right)}
\newcommand{\lbdt}{d\log \left(\nicefrac{4(1+tK)}{\delta}\right)}
\newcommand{\lbdd}{d\log \left(\nicefrac{4(1+(T+1)K)}{\delta}\right)}
\newcommand{\lbddt}{d\log \left({4T(1+(T+1)K)}\right)}

\newcommand{\ts}[1]{}

\newcommand{\compilefullversion}{true}
\ifthenelse{\equal{\compilefullversion}{false}}{%
	\newcommand{\OnlyInFull}[1]{}
	\newcommand{\OnlyInShort}[1]{#1}
}{%
	\newcommand{\OnlyInFull}[1]{#1}%
	\newcommand{\OnlyInShort}[1]{}%
}%

\newcommand{\compilehidecomments}{false}
\ifthenelse{ \equal{\compilehidecomments}{true} }{%
	\newcommand{\wei}[1]{}
	\newcommand{\xutong}[1]{}
	\newcommand{\jinhang}[1]{}
	\newcommand{\siwei}[1]{}
        \newcommand{\carlee}[1]{}

}{
\newcommand{\wei}[1]{{\color{blue}{[Wei: #1]}}}
\newcommand{\xutong}[1]{{\color{green} [Xutong: #1]}}
\newcommand{\jinhang}[1]{{\color{orange} [\text{Jinhang:} #1]}}
\newcommand{\siwei}[1]{{\color{red} [\text{Siwei:} #1]}}

}

\allowdisplaybreaks

\newenvironment{talign*}
 {\csname align*\endcsname}
 {\endalign}



\begin{document}

\title[Combinatorial Logistic Bandits]{Combinatorial Logistic Bandits}

\author{Xutong Liu}
\affiliation{
\institution{The Chinese University of Hong Kong}
\city{Hong Kong}
\country{China}}
\email{liuxt@cse.cuhk.edu.hk}

\author{Xiangxiang Dai}
\affiliation{%
    \institution{The Chinese University of Hong Kong}
    \city{Hong Kong}
    \country{China}
}
\email{xxdai23@cse.cuhk.edu.hk}

\author{Xuchuang Wang}
\affiliation{%
  \institution{University of Massachusetts Amherst}
    \city{Amherst}
    \state{MA}
    \country{USA}}
    \email{xuchuangw@gmail.com}

\author{Mohammad Hajiesmaili}
\affiliation{%
    \institution{University of Massachusetts Amherst}
    \city{Amherst}
    \state{MA}
    \country{USA}}
     \email{hajiesmaili@cs.umass.edu}

\author{John C.S. Lui}
\affiliation{%
    \institution{The Chinese University of Hong Kong}
    \city{Hong Kong}
    \country{China}}
    \email{cslui@cse.cuhk.edu.hk}



\renewcommand{\shortauthors}{Xutong Liu et al.}

\begin{abstract}
We introduce a novel framework called combinatorial logistic bandits (CLogB), where in each round, a subset of base arms (called the super arm) is selected, with the outcome of each base arm being binary and its expectation following a logistic parametric model. The feedback is governed by a general arm triggering process. Our study covers CLogB with reward functions satisfying two smoothness conditions, capturing application scenarios such as online content delivery, online learning to rank, and dynamic channel allocation.
We first propose a simple yet efficient algorithm, CLogUCB, utilizing a variance-agnostic exploration bonus. Under the 1-norm triggering probability modulated (TPM) smoothness condition, CLogUCB achieves a regret bound of $\tilde{O}(d\sqrt{\kappa KT})$, where $\tilde{O}$ ignores logarithmic factors, $d$ is the dimension of the feature vector, $\kappa$ represents the nonlinearity of the logistic model, and $K$ is the maximum number of base arms a super arm can trigger. This result improves on prior work by a factor of $\tilde{O}(\sqrt{\kappa})$.
We then enhance CLogUCB with a variance-adaptive version, VA-CLogUCB, which attains a regret bound of $\tilde{O}(d\sqrt{KT})$ under the same 1-norm TPM condition, improving another $\tilde{O}(\sqrt{\kappa})$ factor. VA-CLogUCB shows even greater promise under the stronger triggering probability and variance modulated (TPVM) condition, achieving a leading $\tilde{O}(d\sqrt{T})$ regret, thus removing the additional dependency on the action-size $K$.
Furthermore, we enhance the computational efficiency of VA-CLogUCB by eliminating the nonconvex optimization process when the context feature map is time-invariant while maintaining the tight $\tilde{O}(d\sqrt{T})$ regret. Finally, experiments on synthetic and real-world datasets demonstrate the superior performance of our algorithms compared to benchmark algorithms. The code is accessible at the following link: https://github.com/xiangxdai/Combinatorial-Logistic-Bandit.

\end{abstract}

%
%

\ccsdesc[500]{Theory of computation~Online learning theory}
\ccsdesc[300]{Computing methodologies~Planning under uncertainty}
\ccsdesc[500]{Computing methodologies~Learning from implicit feedback}
\ccsdesc[500]{Networks~Network performance analysis}

%
\keywords{Multi-armed bandits, combinatorial multi-armed bandits, logistic bandits, nonlinear environments, variance-adaptive, regret}

\received{August 2024}
\received[revised]{September 2024}
\received[accepted]{October 2024}

\maketitle

\section{Introduction}\label{sec:intro}
The stochastic multi-armed bandit (MAB) problem~\cite{robbins1952some,auer2002finite} is a fundamental sequential decision-making problem that has been widely studied (cf.~\citet{slivkins2019introduction,lattimore2020bandit}).
As a noteworthy extension of MAB, the combinatorial multi-armed bandit (CMAB) problem has drawn considerable attention due to their rich applications in recommendation systems \cite{kveton2015cascading,li2016contextual,lattimore2018toprank,agrawal2019mnl}, social networks~\cite{vaswani2015influence,wen2017online,wang2017improving}, and cyber-physical systems \cite{gyorgy2007line,kveton2015combinatorial,li2019combinatorial,liu2023variance}.
In CMAB, the learning agent chooses a subset of base arms (often referred to as the super arm) to be pulled simultaneously in each round.
Once base arms are pulled, each pulled arm, following some unknown distribution, will return a random outcome that can be observed as feedback (typically known as semi-bandit feedback).
Then, the learner receives a reward, which can be a general function of the pulled arms' outcomes, with the summation as the most common function.
The agent's goal is to minimize the expected \textit{regret}, which quantifies the difference in expected cumulative rewards between always selecting the best action (i.e., the super arm with the highest expected reward) and following the agent's own policy. CMAB poses the challenge of balancing exploration and exploitation while dealing with the possible combinatorial explosion of the action space.

Motivated by large-scale applications with a huge number of base arms, recent advances in the combinatorial multi-armed bandit (CMAB) model have led to the development of contextual combinatorial bandits (\ccmab) ~\cite{qin2014contextual,li2016contextual,takemura2021near}. This model integrates contextual information and the well-established linear bandit parameterization into CMAB, allowing for scalability and providing regret bounds independent of the number of base arms $m$.
Building on this, \citet{liu2023contextual} introduced the framework of \ccmab~with probabilistically triggered arms (\ccmab-T), which extends the deterministic semi-bandit feedback model of \ccmab. This framework is designed for scenarios where selected arms may not always yield feedback, while unselected arms might provide feedback, stochastically based on the outcomes of other arms. By incorporating various smoothness conditions for the nonlinear reward function, \ccmab-T covers a broader range of applications, including contextual cascading bandits \cite{kveton2015cascading,li2016contextual,vial2022minimax,liu2022batch} and contextual influence maximization~\cite{lei2015online,vaswani2015influence,vaswani2017diffusion,wen2017online}.

Despite the scalability and generality achieved by the above \ccmab~frameworks, they rely on a linear parametric model, which typically assumes that arm outcomes are continuous (e.g., sales revenue, network delay) and that the mean value is linear in relation to its corresponding feature vector. This assumption prevents this framework from being applied to many real-world scenarios where arm outcomes are \textit{binary}, such as ad clicks \cite{zoghi2017online}, user purchases \cite{trovo2015multi}, video streaming success or failure \cite{liu2023variance}, or any other binary success/failure outcomes.
Moreover, in complex and nonlinear application scenarios, the relationship between the mean value (e.g., purchase probability) and the feature vector (e.g., product rating, price, quality) is not necessarily linear. For instance, the probability of a user purchase may increase dramatically as the rating rises from 3 to 4 but shows little change from 1 to 2 or 4 to 5.
To address these limitations, we pose the following question:

\textit{Can we design a contextual CMAB framework with a nonlinear parametric model that is suitable for large-scale and \textbf{nonlinear} environments with \textbf{binary} outcomes?}

\subsection{Our Contributions}

To answer the above question, we propose a new combinatorial logistic bandit (CLogB) framework, where each base arm outcome is modeled using a logistic parametric model. This model, also utilized in logistic bandits (LogB) for single-arm selection scenarios \cite{faury2020improved,abeille2021instance,faury2022jointly}, provides a rigorous theoretical framework to analyze the nonlinearity of parametric bandits and has been empirically demonstrated to outperform linear bandits in cases involving binary outcomes \cite{li2012unbiased}.
In the CLogB framework, the contextual information at each round \( t \in [T] \) is captured through a time-varying feature map \( \phi_t \). The mean outcome of each arm \( i \in [m] \) is determined by a sigmoid function applied to the inner product of the feature vector \( \bphi_t(i) \in \R^d \) and an unknown vector \( \btheta^* \in \R^d \) (with \( d \ll m \) to manage large-scale applications). At the super arm level, the framework inherits the arm triggering process and focuses on nonlinear reward functions that satisfy various smoothness conditions. This approach can accommodate a diverse range of application scenarios, including online content delivery \cite{liu2023variance}, online learning to rank \cite{li2016contextual}, dynamic channel allocation \cite{gai2010learning}, and online packet routing \cite{kveton2015cascading}.
With this formulation, CLogB combines the nonlinear parameterization capabilities of LogB with the scalability, diverse reward functions, and general feedback models of C$^2$MAB-T. The detailed comparison with the related works is postponed to \cref{apdx_sec:related_work}. 

\begin{table*}[t]
	\caption{Summary of the main results for CLogB and the additional results for CLogB with time-invariant feature maps (CLogB-TI). }	\label{tab:clogb_res}	
	\centering
	\resizebox{1.\columnwidth}{!}{
	\centering
	\begin{threeparttable}
	\begin{tabular}{|cccccc|}
 \hline
		\textbf{CLogB}&\textbf{Algorithm}&\textbf{Condition}& \textbf{Coefficient} & \textbf{Regret Bound} & \textbf{Per-round Cost}\\
  \hline
    & CLUB-cascade$^{*}$   \cite{li2018online}  & - & -   & $\tilde{O}\left( { d \kappa\sqrt{KT}} \right)$ & $\tilde{O}\left(dK^2T^2 + T_{\text{nc}}+T_{\alpha}\right)^{**}$ \\
                  \rowcolor{Gray}
       \textbf{(Main Result \ref{thm:var_ag_thm})} & CLogUCB  (\Cref{alg:CLogUCB}) & 1-norm TPM& $B_1$  & $\tilde{O}\left(B_1 d\sqrt{\kappa KT}\right)$ & $\tilde{O}\left(dK^2T^2 +T_{\alpha} \right)$\\
                         \rowcolor{Gray}
       \textbf{(Main Result \ref{thm:var_ad_thm1})} & VA-CLogUCB  (\Cref{alg:VA_CLogB}) & 1-norm TPM& $B_1$  & $\tilde{O}\left(B_1 d\sqrt{KT}+B_1\kappa d^2\right)$ & $\tilde{O}\left(dK^2T^2 +T_{\text{nc}} +T_{\alpha}\right)$\\
 \rowcolor{Gray}
	\textbf{(Main Result \ref{thm:var_ad_thm2})}& 	VA-CLogUCB 
 (\Cref{alg:VA_CLogB})  & TPVM  & $B_v$ $^\dagger$, $\lambda\ge1$$^{\ddagger}$ & $\tilde{O}\left(B_v d\sqrt{T} +B_1\kappa d^2\right)$ & $O\left(dK^2T^2 +T_{\text{nc}} +T_{\alpha}\right)$\\
    \hline
  \hline
    \textbf{CLogB-TI} &\textbf{Algorithm}&\textbf{Condition}& \textbf{Coefficient} & \textbf{Regret Bound}\\
    \hline
       \rowcolor{Gray}
   \textbf{(Additional Result 1)}&	EVA-CLogUCB~(\Cref{alg:EVA_CLogB})& 1-norm TPM & $B_1$  & $\tilde{O}\left(B_1 d\sqrt{KT}+B_1\kappa K d^2\right)$ & $O\left(dK^2T^2 +T_{\alpha}\right)$ \\
     \rowcolor{Gray}
   \textbf{(Additional Result 2)}&	EVA-CLogUCB~(\Cref{alg:EVA_CLogB})& TPVM & $B_v, \lambda\ge 1$  & $\tilde{O}\left(B_vd\sqrt{T}+B_1\kappa Kd^2\right)$ &$\tilde{O}\left(dK^2T^2 +T_{\alpha}\right)$ \\
	\hline
	\end{tabular}
	  \begin{tablenotes}[para, online,flushleft]
	\footnotesize
 This table assumes $T\gg m \ge K \gg d$.
 \item[]\hspace*{-\fontdimen2\font}$^{*}$ This work is specified for cascading bandits without considering the general nonlinear reward functions that satisfy smoothness conditions.
  \item[]\hspace*{-\fontdimen2\font}$^{**}$ $T_{\text{nc}}$ and $T_{\alpha}$ are the time to solve a nonconvex projection problem and an $\alpha$-approximation for the combinatorial optimization problem, respectively.
 \item[]\hspace*{-\fontdimen2\font}$^\dagger$ Generally,  
 coefficient $B_v=O(B_1\sqrt{K})$ and the existing regret bound is improved when $B_v=o(B_1\sqrt{K})$
 \item[]\hspace*{-\fontdimen2\font}$^{\ddagger}$ $\lambda$ is a coefficient in TPVM condition: when $\lambda$ is larger, the condition is stronger with smaller regret but can include fewer applications.
 
	\end{tablenotes}
			\end{threeparttable}
	}
 \vspace{-0.1in}
\end{table*}


Our main technical results are summarized in \cref{tab:clogb_res}. First, we study CLogB, whose reward function satisfies the 1-norm triggering probability modulated (TPM) smoothness condition with coefficient $B_1$, a general condition originally proposed in the foundational CMAB framework~\cite{wang2017improving}. In this context,
we develop the variance-agnostic CLogUCB algorithm, which achieves a regret upper bound of \(\tilde{O}(B_1d\sqrt{\kappa KT})\) with per-round time complexity \(O(dK^2T^2 + T_{\alpha})\), where \(T_{\alpha}\) represents the time required to solve an \(\alpha\)-approximate solution to the underlying combinatorial optimization problem (Main Result 1 in \cref{tab:clogb_res}). 
Unlike the linear parametric C$^2$MAB framework, which can easily construct an exploration bonus using a closed-form solution from linear regression for each base arm, CLogB faces the challenge of efficiently creating an estimator without closed-form solutions and taking into account the nonlinearity to construct the exploration bonus.
To address this, we employ the maximum likelihood estimation (MLE) for the unknown parameter $\btheta^*$ and establish a novel concentration inequality with an enlarged confidence radius tailored for the CLogB setting. 
Our approach diverges from previous works \cite{li2018online, faury2020improved} by constructing the exploration bonus directly around the MLE estimator, thus avoiding nonconvex projections.
With these techniques, our result improves the regret bound over previous work \cite{li2018online} by a factor of $O(\sqrt{\kappa})$ and eliminates the additional computation cost $T_{\text{nc}}$ associated with the nonconvex projection.
Despite the above improvement, CLogUCB still suffers from a nonlinearity term $\tilde{O}(\sqrt{\kappa})$ and a action-size term $\tilde{O}(\sqrt{K})$ due to its conservative variance-agnostic exploration. 

Second, to further improve the results obtained by CLogUCB, we propose a new variance-adaptive VA-CLogUCB algorithm, in light of the recent variance-adaptive principle that helps to reduce either the nonlinearity dependency $\kappa$ for LogB \cite{abeille2021instance,faury2022jointly} or the action-size dependency $K$ for C$^2$MAB \cite{merlis2019batch,liu2023contextual}.
Under the same 1-norm TPM smoothness condition, VA-CLogUCB achieves a regret of $\tilde{O}(B_1 d\sqrt{KT}+B_1\kappa d^2)$, which eliminates the $O(\sqrt{\kappa})$ dependency from the leading $O(\sqrt{T})$ term compared with CLogUCB (Main Result 2 in \cref{tab:clogb_res}).
Furthermore, under the variance-aware TPVM smoothness condition with coefficients $(B_v,B_1,\lambda)$ \cite{merlis2019batch,liu2022batch}, VA-CLogUCB achieves the "best-of-both-worlds" result, which achieves a regret of $\tilde{O}(B_v d\sqrt{T}+B_1 \kappa d^2)$, removing both $O(\sqrt{\kappa})$ and $O(\sqrt{K})$ dependencies in the leading term compared to CLogUCB (Main Result 3 in \cref{tab:clogb_res}). 
The main challenge is that the variance-adaptive exploration bonus, associated with the MLE estimator \(\hat{\btheta}_t\), may not decrease steadily as more feedback is collected. Inspired by~\cite{faury2020improved}, we project MLE $\hat{\btheta}_t$ to a bonus-vanishing region to control the growth of the exploration bonus, which secures the desired regret guarantee at the cost of an additional $T_{\text{nc}}$ time complexity. Although this projection only occurs when the MLE does not lie within the bonus-vanishing region, it is nonconvex and could be NP-hard to solve \cite{faury2020improved}.

Third, to improve computational efficiency, we propose the EVA-CLogUCB algorithm under the mild assumption that the feature map is time-invariant. EVA-CLogUCB achieves a regret of $\tilde{O}(B_1 d\sqrt{KT}+B_1 \kappa Kd^2)$ under the 1-norm TPM condition and $\tilde{O}(B_v d\sqrt{T}+B_1 \kappa Kd^2)$  under the TPVM condition, with only $\tilde{O}(dK^2T^2 + T_{\alpha})$ computation cost (Additional Results 1 \& 2 in \cref{tab:clogb_res}). These results eliminate the \(T_{\text{nc}}\) computation cost while maintaining the same leading regret as VA-CLogUCB. 
The key challenge is to control the growth of the bonus without projecting the MLE estimator onto the nonconvex bonus-vanishing region. 
Our solution involves a burn-in stage with \(O(\log T)\) rounds to construct a \textit{convex} nonlinearity-restricted region. After this stage, we directly optimize the MLE within this region, improving the computational efficiency.

Finally, we validate our theoretical results through experiments on both synthetic and real-world datasets, demonstrating at least 53\% regret improvement compared to baseline algorithms.

\textbf{Notations.}
We use $[n]$ to represent set $\{1,...,n\}$. We use boldface lowercase letters and boldface CAPITALIZED letters for column vectors and matrices, respectively. $\norm{\bx}_p$ denotes the $\ell_p$ norm of vector $\bx$. For any symmetric positive semi-definite (PSD) matrix $\bM$ (i.e., $\bx^{\top} \bM \bx \ge 0, \forall \bx$), $\norm{\bx}_{\bM}=\sqrt{\bx^{\top} \bM \bx}$ denotes the matrix norm of $\bx$ regarding matrix $\bM$. The notation $\dot{f}$ (resp. $\ddot{f}$) denotes the first (resp. second) derivative of $f$.

\section{Problem Setting}\label{sec: problem setting}

\subsection{Combinatorial Logistic Bandit Model}\label{sec:CLogB_model}
In this section, we introduce our model for the combinatorial logistic bandit (CLogB) problem.  
CLogB is a learning game between a player (the learner) and the environment in $T$ rounds, with the following components introduced below. 

\begin{figure}[t]
    \centering
    \begin{subfigure}[t]{0.355\linewidth}
    \includegraphics[width=\linewidth]{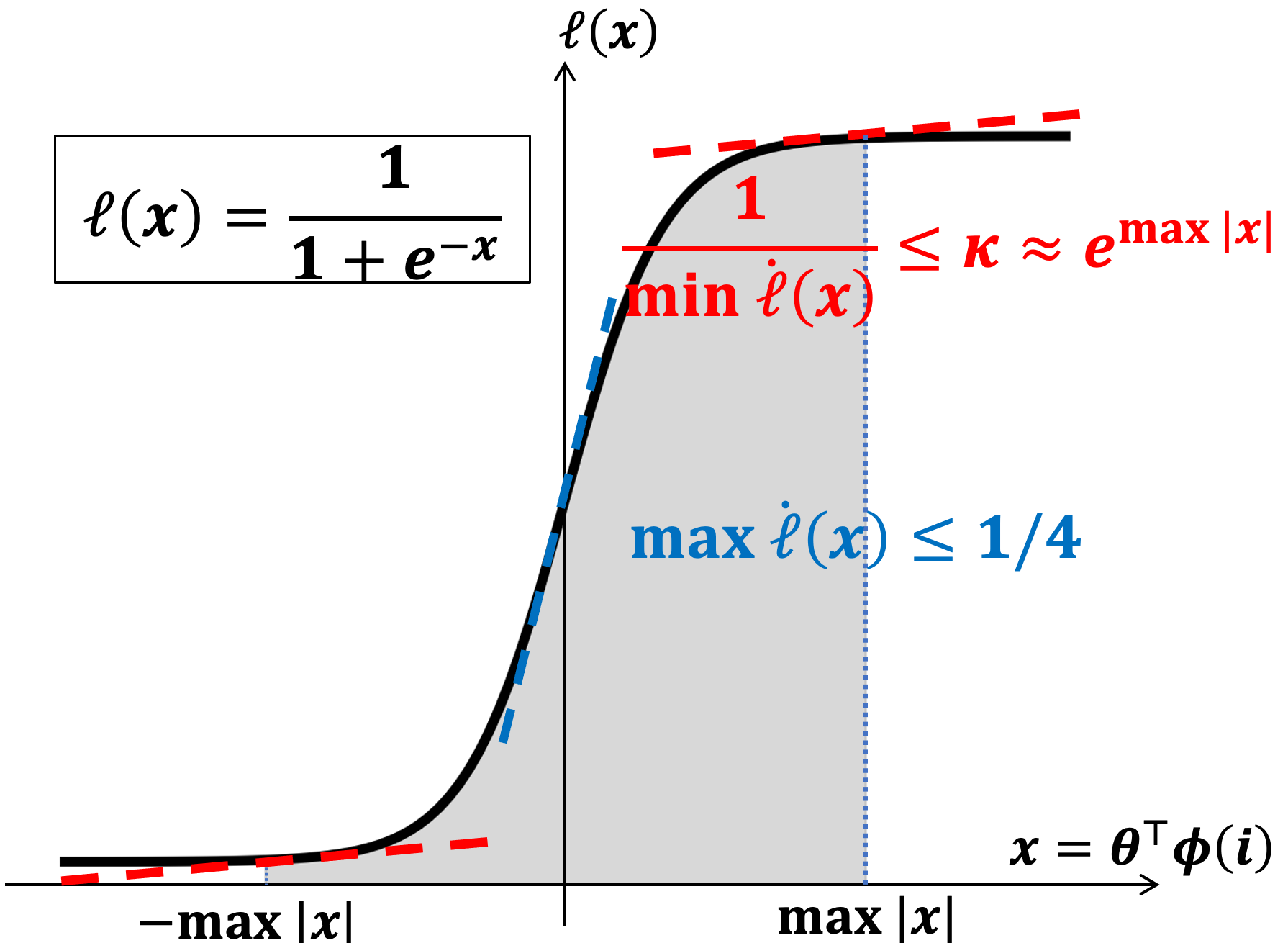}
    \caption{Sigmoid function}
    \label{fig:logistic}
    \end{subfigure}
    \begin{subfigure}[t]{0.453\linewidth}
    \includegraphics[width=\linewidth]{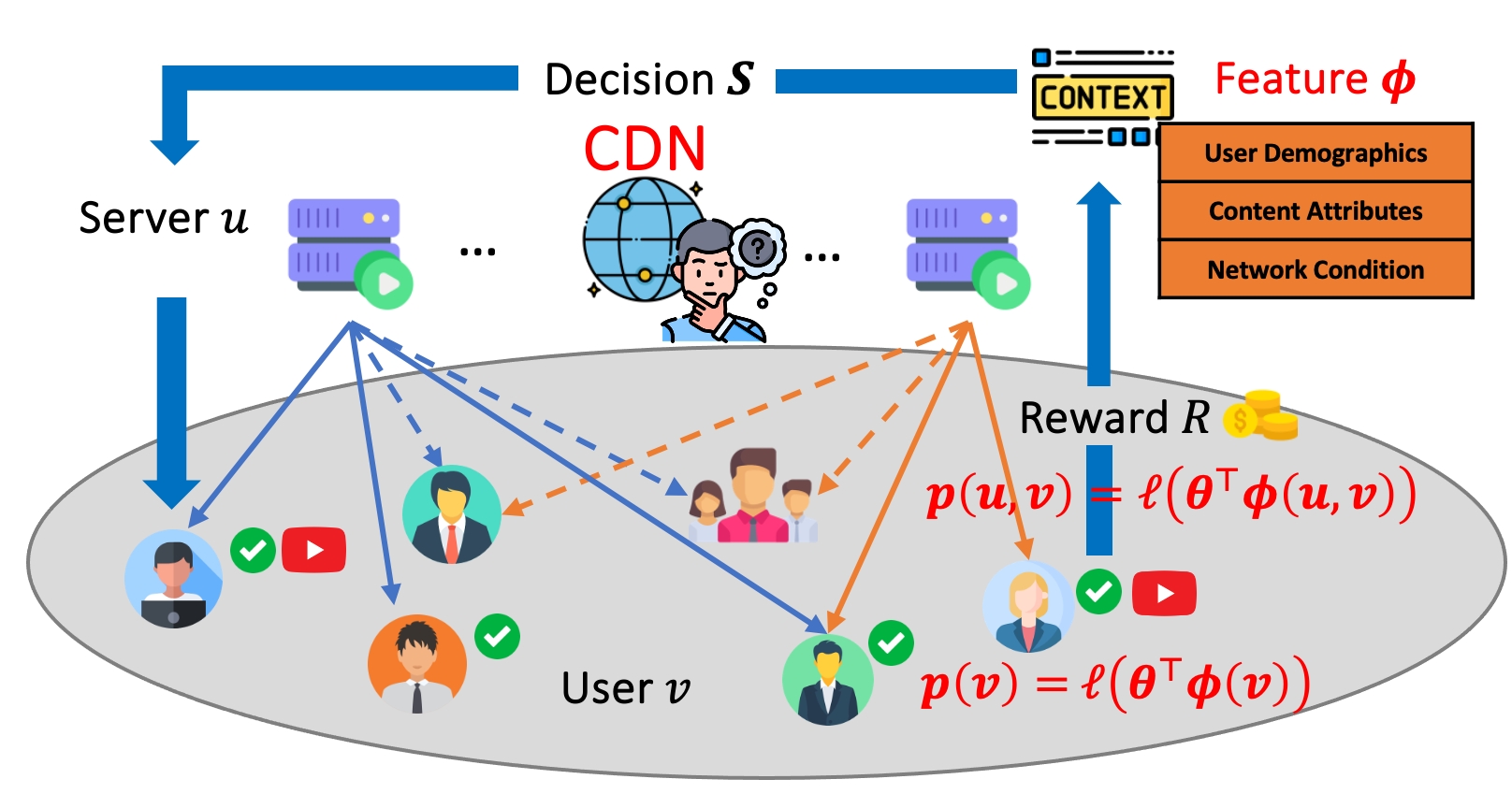}
    \caption{Content delivery network}
    \label{fig:cdn}
    \end{subfigure}
    \caption{\textbf{Left:} illustration of a sigmoid function with linear predictor $x=\btheta^{\top} \bphi(i)$ as input. The larger the $|x|$, the flatter the curve is, and the higher the nonlinearity level $\kappa$, where $\kappa$ grows exponentially fast w.r.t $|x|$. \textbf{Right:} CLogB for content delivery networks, the decision maker chooses servers based on contextual features, successfully covers users (green check marks) via edges (solid lines) with probability $p(u,v)$, and gains rewards if the user consumes the content (red play buttons) with probability $p(v)$.  }
    \label{fig:regret}
    \vspace{-0.1in}
\end{figure}

\textbf{Base arms.}
The environment has a set of $[m]=\{1,2,...,m\}$ base arms. 
Before the game starts, the environment chooses an unknown parameter $\btheta^* \in \Theta$, where $\Theta$ is the set of feasible parameters.
At each round $t \in [T]$, the environment first reveals a feature map $\phi_t \in \Phi$ to the learner, where $\phi_t$ is a function $[m]\rightarrow \R^d$ that maps an arm $i$ to a $d$-dimensional feature vector and $\Phi$ is the set of feasible feature maps.
Therefore, the leaner can observe feature vectors $\bphi_t(1), ..., \bphi_t(m) \in \R^d$ for all arms.  
Given $\phi_t$, the environment then draws Bernoulli outcomes $\bX_t=(X_{t,1},...X_{t,m})\in \{0,1\}^m$ with mean $\mu_{t,i}\defeq\E[X_{t,i} | \cH_t]=\ell(\btheta^{*\top}\bphi_{t}(i))$ for each base arm $i$. 
Here, $\btheta^{*\top}\bphi_{t}(i)$ is the linear predictor, $\ell: \R\rightarrow \R_{+}$ is the \textit{sigmoid} function $\ell(x)\defeq (1+e^{-x})^{-1}$  as shown in \cref{fig:logistic} that links the linear predictor and the mean $\mu_{t,i}$ in a nonlinear manner, and $\cH_t$ denotes the history information before the learner chooses the next round's action, which will be specified shortly after. 
For notational convenience, we use $\bmu_t\defeq(\mu_{t,1},...,\mu_{t,m})$ to denote the mean vector at round $t$ and ${\cM\defeq\{\ell\left(\btheta^{\top} \bphi(i)\right)_{i \in [m]}: \phi \in \Phi, \btheta \in \Theta\}}$ to denote all possible mean vectors generated by $\Phi$ and $\Theta$. 

\textbf{Combinatorial actions.}
At each round $t$, after observing the feature vectors $\bphi_t(1), ..., \bphi_t(m)$, the learner selects a combinatorial action $S_t\in \cS$, where $\cS$ is the set of feasible actions.
Typically, $S_t$ is composed of a set of individual base arms $S\subseteq [m]$, which we refer to as a super arm. 
However, $S_t$ can be more general than the super arm, possibly continuous, such as resource allocations \cite{zuo2021combinatorial}, which we will emphasize if needed.

\textbf{Probablistically triggering arm feedback.}
Motivated by the properties of real-world applications that will be introduced in detail in~\cref{sec:repr_app}, we consider a feedback process that involves scenarios where each base arm in a super arm $S_t$ does not always reveal its outcome, even probabilistically. For example, a user might leave the system randomly at some point before examining the entire recommended list $S_t$,  resulting in unobserved feedback for the unexamined items.
To handle such probabilistic feedback, we assume that after the action $S_t$ is selected, the base arms in a random set $\tau_t \sim D_{\text{trig}}(S_t, \bX_t)$ are triggered depending on the outcome $\bX_t$, where $D_{\text{trig}}(S, \bX)$ is the probabilistic triggering function on the subsets $2^{[m]}$ given $S$ and $\bX$.
This means that the outcomes of the arms in $\tau_t$, i.e., $(X_{t,i})_{i\in \tau_t}$ are revealed as feedback to the learner, which could also be involved in determining the reward of action $S_t$ as we introduce later. 
We let $p_i^{\bmu, S}$ denote the probability that base arm $i$ is triggered when the action is $S$, the mean vector is $\bmu$.
To allow the algorithm to estimate the underlying parameter $\btheta^*$ directly from samples, we assume the outcome does not depend on whether the arm $i$ is triggered, i.e., $\E_{\bX \sim D, \tau \sim D_{\text{trig}}(S,\bX)}[X_i | i\in \tau]=\E_{\bX\sim D}[X_i]$.

Note that triggering probabilities $p_i^{\bmu,S}$ are crucial for the triggering probability modulated bounded smoothness conditions to be defined in \cref{sec:cond}.
We denote \textit{{action size}} $K$ as the maximum number of arms that can be triggered, i.e., $K=\max_{S \in \cS,\bmu \in \cM}\sum_{i\in [m]}\I\{p_i^{\bmu, S}>0\}$, which will appear as important coefficients in our regret bounds.
For online learning to rank applications in \cref{sec:app_ranking}, for example, $p_i^{\bmu, S}$ corresponds to the probability that the user will examine item $i$ in a ranked list $S$ and action-size $K$ corresponds to the length of the ranked list. To this end, we can give the formal definition of history $\cH_t=(\phi_s, S_s, \tau_s, (X_{s,i})_{i\in \tau_s})_{s<t} \bigcup \phi_t$ , which contains all information before round $t$, as well as the feature map $\phi_t$ at round $t$.



\textbf{Reward function.}
At the end of round $t$, the learner receives a nonnegative reward $R(S_t, \bX_t, \tau_t)$, determined by action $S_t$, outcome
$\bX_t$, and triggered arm set $\tau_t$. 
Similarly to~\cite{wang2017improving}, we assume the expected reward to be $r(S_t;\bmu_t)\defeq \E[R(S_t,\bX_t,\tau_t)]$, a function of
the unknown mean vector $\bmu_t$, where the expectation is taken over the randomness of $\bX_t$ and $\tau_t \sim D_{\text{trig}}(S_t,\bX_t)$. \rev{In this paper, we assume that the \textit{expected} reward function, $r: \cS \times [0,1]^m \rightarrow \R_{\ge 0}$, is known. That is, given any action $S$ and any base arm mean vector $\bmu$, the learner can compute the expected reward $r(S; \bmu)$. However, the random reward function $R:\cS \times [0,1]^m \times [m]\rightarrow \R_{\ge 0}$ remains unknown to the learner.}

\textbf{Offline approximation oracle and approximate regret.}
The goal of CLogB is to accumulate as much reward as possible over $T$ rounds by learning the underlying parameter $\btheta^*$.
The performance of an online learning algorithm $A$ is measured by its \textit{regret}, defined as the difference in the expected cumulative reward between always playing the best action $S^*_t \defeq \argmax_{S \in \cS}r(S;\bmu_t)$ in each round $t$ and playing the actions chosen by algorithm $A$.
For many reward functions, it is NP hard to compute the exact $S^*_t$ even when $\bmu_t$ is known, so similar to~\cite{chen2013combinatorial,wang2017improving,liu2022batch,liu2023contextual}, we assume that algorithm $A$ has access to an offline $\alpha$-approximation oracle ORACLE, which takes any mean vector $\bmu \in [0,1]^m$ as input, and outputs an $\alpha$-approximate solution $S\in \cS$, i.e., 
\begin{align}\label{eq:def_oracle}
    S=\text{ORACLE}(\bmu) \text{ s.t. } r(S;\bmu)\ge \alpha \cdot \max_{S'\in \cS}r(S';\bmu)
\end{align}
The $T$-round $\alpha$-approximate regret is then defined as
\begin{equation}\label{eq:def_reg}\textstyle
    \text{Reg}(T)=\E\left[\sum_{t=1}^T \left(\alpha  \cdot r(S^*_t;\bmu_t)-r(S_t;\bmu_t)\right)\right],
\end{equation}
where the expectation is taken over the randomness of outcomes $\bX_1, ..., \bX_T$, the triggered sets $\tau_1, ..., \tau_T$, as well as the randomness of the algorithm itself.

To summarize, we can specify a CLogB instance by a tuple $([m], \Phi, \Theta, \cS, D_{\text{trig}}, R)$, where $[m]$ are base arms, $\cS$ are combinatorial actions,
$\Phi$ are feature maps, $\Theta$ is the parameter space,  
$D_{\text{trig}}$ is the probabilistic triggering function, and $R$ is the reward function.

\subsection{Connections to Existing Models}\label{sec:connect}
Our model is mainly based on \citet{wang2017improving,liu2023contextual}, which handles the nonlinear reward function, efficient offline approximation, and probabilistic arm triggering feedback for combinatorial decision-making. 
But different from \citet{wang2017improving,liu2023contextual}, at the base arm level, we introduce a nonlinear parameterization for the binary outcome of each base arm based on the logistic bandit model~\cite{filippi2010parametric,faury2020improved}, generalizing the non-parametric CMAB \cite{wang2017improving} and the linear parametric C$^2$MAB~\cite{liu2023contextual}, respectively.

\noindent\textbf{Comparison between CLogB and C$^2$MAB \cite{liu2023contextual}}.
The main difference between the current work and C$^2$MAB \cite{liu2023contextual} is the parameterization of the base arm. 
If we set the sigmoid link function $\ell(x)=(1+e^{-x})^{-1}$ to be the linear function $\ell(x)=x$, then CLogB reproduces the setting of \citet{liu2023contextual}. 
While linear parameterization is good at handling continuous reward outcomes, it falls short in handling binary outcomes (e.g., click or not, network function works or not, and generally any success/failure outcomes) whose mean reward (i.e., success probability) lies within $[0,1]$. 
Moreover, C$^2$MAB simply assumes the linear relationship between the mean and the feature vectors, which is less meaningful when dealing with probabilities (e.g., the success probability grows from 50\% to 100\% when the value of the feature vector doubles). 
In contrast, the sigmoid link function naturally has a range of $[0,1]$ and captures the nonlinear relationships between the feature vectors and the success probability, thus resulting in significantly improved performance in real-world applications that are complex and nonlinear~\cite{li2012unbiased}.

\noindent\textbf{Extension to generalized linear model}.
In this work, we consider the logistic model in which the outcome of the base arm $X_{t,i}$ follows the Bernoulli distribution with the sigmoid link function $\ell(x)=(1+e^{-x})^{-1}$. Essentially, we utilize two important properties of this logistic model for efficient learning and decision-making: (1) the \textit{self-concordant} property of the sigmoid link function \cite{bach2010self}, i.e., a smoothness property where the second-order derivative is bounded by its first-order derivative $|\ddot{\ell}(x)|\le \dot{\ell}(x)$ for any $x\in \R$, which enables a tight confidence region that uses local nonlinearity information instead of the loose global nonlinearity coefficient; and (2) the variance property of the distribution, i.e. $\Var[X_{t,i}|\cH_t]=\dot{\ell}({\btheta^*}^{\top} \bphi_t(i))$, which links the local nonlinearity information with the variance so that enables variance-adaptive algorithm design. Our result carries over to any model that satisfies the above properties. In fact, those two properties are very mild and hold for many instances of the generalized linear model whose outcome distribution belongs to canonical exponential families, such as the Poisson distribution with $\ell(x)=e^{x}$.

\noindent\rev{\textbf{Overview of the technical challenges.} There are several unique challenges in combining two prominent research directions: the logistic bandit (LogB) and combinatorial bandits (CMAB). 
First, unlike previous contextual CMAB works~\cite{qin2014contextual,takemura2021near,liu2023contextual} using the linear regression that has a straightforward closed-form solution, we utilize the maximum likelihood estimation (MLE) to estimate the underlying parameter $\btheta^*$ and establish a new concentration inequality for MLE. While similar MLE-based techniques are used in the LogB literature, their concentration inequalities are not applicable when multiple arms are pulled simultaneously and the feedback follows a probabilistic triggering process. This necessitated our new concentration inequality with non-trivial analysis, as detailed in \cref{apdx_sec:proof_concen_ineq}.
Second, the combinatorial action space introduces significant computational challenges and regret analysis challenges to standard LogB algorithms. Most LogB algorithms employ the optimism-in-the-face-of-uncertainty (OFU) approach, which jointly searches the action and the parameter space within the confidence region and leads to exponentially high time complexity. The UCB-based method, though less explored in the LogB literature, is central to our approach. We introduce a variance-adaptive UCB-based algorithm, which provides a "best-of-both-worlds" result, eliminating both the $O(\sqrt{\kappa})$ and $O(\sqrt{K})$ compared to directly using the existing techniques from LogB.
Finally, we further enhance computational efficiency by designing a novel burn-in stage to remove the NP-hard projection step. This component is also novel for UCB-based LogB methods and significantly improves their efficiency. The key challenge we faced was determining how to design this burn-in stage to balance the trade-off between regret minimization and computational efficiency.
}

\subsection{Key Assumptions and Conditions}\label{sec:cond}

Due to the nonlinear parameterization of base arms and the complex structure of the reward function for combinatorial actions, it is essential to consider some assumptions/conditions for the parameter and the reward function to achieve meaningful regret bounds~\cite{chen2013combinatorial, chen2016combinatorial, wang2017improving,merlis2019batch,liu2022batch}. 
We consider the following assumptions/conditions at the arm and reward level.

\noindent\textbf{Arm-level assumptions.} \Cref{cond:bounded_para} bounds the range of plausible feature vectors for each base arm and unknown parameters. \Cref{cond:arm_nonlinear} bounds the nonlinearity level of the base arm mean regarding all plausible linear predictor $\btheta^{\top}\bphi(i)$.

\begin{assumption}[\textbf{Bounded parameter and arm feature}]\label{cond:bounded_para}
    There exists a known constant $L>0$ such that for any $\btheta \in \Theta$, $\norm{\btheta}_2\le L$. For any $\phi\in \Phi$ and $i\in [m]$, it holds that $\norm{\bphi(i)}_2\le 1$.
\end{assumption}

\begin{assumption}[\textbf{Arm-level nonlinearity}]\label{cond:arm_nonlinear}
    There exists a known upper bound $\kappa>0$ such that
        $\left(\min_{i \in [m], \phi\in \Phi, \btheta^{*} \in \Theta }\dot{\ell}(\btheta^{\top}\bphi(i))\right)^{-1}\le \kappa.$
\end{assumption}

\begin{remark} [Intuition of \Cref{cond:arm_nonlinear}] 
    Intuitively, for any linear predictor $x=\btheta^{\top}\bphi(i)$, $\dot{\ell}(x)$ represents the sensitivity of its mean $\ell(x)$, i.e., the speed of change when its linear predictor $x$ changes. When $\ell(x)=x$ is perfectly linear, the sensitivity is constant $\dot{\ell}(x)=1$ for all plausible $x$. When $\ell(x)$ is not linear, the level of nonlinearity can be measured by the ratio between the maximum sensitivity $\max_{x} \dot{\ell}(x)$ and the minimum sensitivity $\min_{x} \dot{\ell}(x)$. For our logistic model, as shown in \cref{fig:logistic}, the maximum sensitivity is at most some constant (i.e., $\max_x\dot{\ell}(x)=\max_x \ell(x)(1-\ell(x))=1/4$ when $x=0$), so the level of nonlinearity can be measured by $1/\min_x\dot{\ell}(x)$, which is bounded by $\kappa$. Also, one can verify that $\kappa \ge \exp(|x|)$ for any plausible $x$, so $\kappa$ can be exponentially large when $x<0$ is small (or equivalently when $\ell(x)$ is very small). This means that if there exists an arm with a very small mean (e.g., low purchase probability), then $\kappa$ will be exponentially large.
\end{remark}

\rev{
\begin{remark}[Dealing with unknown $\kappa$]
    Directly knowing the exact nonlinearity parameter $\kappa$ is often unrealistic. Instead, we can assume to have access to an \textit{upper bound} $\bar{\kappa}$ of $\kappa$. Based on the definition of $\kappa$ in \cref{cond:arm_nonlinear}, there are two ways to compute $\bar{\kappa}$:
(1) If we know an upper bound on the $\ell_2$ norm of the underlying parameter, i.e., $\norm{\btheta}_2 \le L$ for $\btheta \in \Theta$, then $\bar{\kappa}$ can be bounded by $4\exp(L)$. Note that the knowledge of $L$ is a very standard assumption in linear contextual bandits \cite{chu2011contextual,abbasi2011improved} and logistic bandits \cite{faury2020improved,abeille2021instance,faury2022jointly}.  Without knowledge of $L$, none of these algorithms can provide any theoretical guarantees either.
(2) If there exists a constant $\gamma > 0$ such that the base arm means are bounded within $\mu_i \in [\gamma, 1-\gamma]$ for $i \in [m]$, then $\bar{\kappa} \le \frac{1}{(1-\gamma)\gamma}$.
For the algorithm and the regret bound, we can replace $\kappa$ with the larger $\bar{\kappa}$, and the regret bound will hold with $\bar{\kappa}$ in place of $\kappa$.
Additionally, we include experiments using the upper bound $\bar{\kappa}$ (the first method with $\bar{\kappa}=4\exp(L)$) to demonstrate how inaccuracies in $\kappa$ affect the algorithm's regret for all algorithms proposed in this paper in Appendix \ref{app:extended}.
\end{remark}
}

\noindent\textbf{Reward-level conditions.} Condition~\ref{cond:mono} indicates the reward is monotonically increasing when the parameter $\bmu$ increases. In online learning to rank application (\cref{sec:app_ranking}), for example, \Cref{cond:mono} means that when the purchase probability for each item increases, the total number of purchases for recommending any list of items also increases. Condition~\ref{cond:TPM} and \ref{cond:TPVMm} both bound the reward smoothness/sensitivity, i.e., the amount of the reward change caused by the parameter change from $\bmu$ to $\bmu'$. In online learning to rank (\cref{sec:app_ranking}), for example, these conditions upper bounds the difference in total number of purchases when the purchase probability for the items changes from $\bmu$ to $\bmu'$. In general, \Cref{cond:TPVMm} is stronger and can yield better results. On the other hand, \Cref{cond:TPM} is easier to satisfy and can cover more applications.   

\begin{condition}[\textbf{Monotonicity}]\label{cond:mono}
A CLogB problem instance satisfies monotonicity condition if for any action $S \in \cS$, any mean vectors $\bmu,\bmu' \in [0,1]^m$ s.t. $\mu_i \le \mu'_i $ for all $i \in [m]$, we have $ r(S;\bmu) \le r(S;\bmu') $. 
\end{condition}

\begin{condition}[\textbf{1-norm TPM bounded smoothness}, \cite{wang2017improving}]\label{cond:TPM}
We say that a CLogB problem instance satisfies the 1-norm triggering probability modulated (TPM) $B_1$-bounded smoothness condition, if there exists $B_1>0$, such that for any action $S \in \cS$, any mean vectors $\bmu, \bmu' \in [0,1]^m$, we have $|r(S;\bmu')-r(S;\bmu)|\le B_1\sum_{i \in [m]}p_{i}^{\bmu,S}|\mu_i-\mu'_i|$.
\end{condition}


\begin{condition}[\textbf{TPVM bounded smoothness}, \cite{liu2022batch}]\label{cond:TPVMm}
We say that a CLogB problem instance satisfies the triggering probability and variance modulated (TPVM) $(B_v, B_1,\lambda)$-bounded smoothness condition, if there exists $B_v,B_1,\lambda >0$ such that for any action $S \in \cS$, any mean vector $\bmu, \bmu' \in (0,1)^m$, for any $\boldzeta, \boldeta \in [-1,1]^m$ s.t. $\bmu'=\bmu+\boldzeta+\boldeta$, we have
$|r(S;\bmu')-r(S;\bmu)|\le B_v\sqrt{\sum_{i\in [m]}(p_i^{\bmu,S})^{\lambda}\frac{\zeta_i^2 }{(1-\mu_i)\mu_i}} + B_1 \sum_{i\in[m]}p_i^{\bmu,S}\abs{\eta_i}$. 
\end{condition}

\begin{remark}[Intuition of \Cref{cond:TPM}]
For Condition~\ref{cond:TPM}, the key feature is that the parameter change in each base arm $i$ is modulated by the triggering probability $p_i^{\bmu,S}$.
Intuitively, for base arm $i$ that is unlikely to be triggered/observed (small $p_i^{\bmu, S}$), Condition~\ref{cond:TPM} ensures that a large change in $\mu_i$ (due to insufficient observation) only causes a small change (multiplied by $p_i^{\bmu,S}$) in reward, saving a $p_{\min}$ factor in \cite{wang2017improving} where $p_{\min}$ is the minimum positive triggering probability.
In online learning to rank application, for example, since users will never purchase an item if it is not examined, increasing or decreasing the purchase probability of an item that is unlikely to be examined (i.e., with small $p_i^{\bmu,S}$) does not significantly affect the total number of purchases.
\end{remark}


\begin{remark}[Intuition of \Cref{cond:TPVMm}]
\label{rmk:TPVM}
For Condition~\ref{cond:TPVMm}, the leading $B_v$ term is modulated by the inverse of the variance $\Var[X_{t,i}]=(1-\mu_i)\mu_i$, and thus allows applications to reduce their $B_v$ coefficient to a coefficient independent of $K$, leading to significant savings in the regret bound for applications like cascading bandits and PMC bandits in \cref{sec:repr_app}.
Intuitively, the inverse variance modulation adds more importance to the arm whose success probability is approaching the boundary value $1$ or $0$. In PMC bandits (\cref{sec:app_PMC}), for example, whenever there exists a critical node in $S$ that is very likely to cover ($\mu_i\approx 1$) the target node, increasing or decreasing other nodes' probability will not affect the fact the target node is covered, making them less important. In packet routing (\cref{sec:app_routing}), on the other hand, if there exists a critical edge in the path that is very likely to be a broken edge ($\mu_i\approx 0$), the packet transmission will probably fail regardless other edges.
Finally, Condition~\ref{cond:TPVMm} combines both the triggering-probability modulation from Condition~\ref{cond:TPM} and the variance modulation mentioned above. The exponent $\lambda$ of $p_i^{\bmu,S}$ gives additional flexibility to trade-off between the strength of the condition and the regret, i.e., with a larger $\lambda$, one can obtain a smaller regret bound, while
with a smaller $\lambda$, the condition is easier to satisfy to include more applications. 
\end{remark}

\begin{remark}[Relationship between \cref{cond:TPM} and \cref{cond:TPVMm}]\label{rmk:connect_TPM_TPVM}
    In general, Condition~\ref{cond:TPVMm} is stronger than Condition~\ref{cond:TPM} as the former degenerates to the latter conditions by setting $\boldzeta=\boldsymbol{0}$. Conversely, by applying the Cauchy-Schwartz inequality, one can verify that if a reward function is 1-norm TPM $B_1$-bounded smooth, then it is TPVM $(B_1\sqrt{K}/2, B_1, \lambda)$-bounded smooth for any $\lambda\le 2$.
\end{remark}

In light of the conditions described above that significantly advance the non-parametric CMAB-T and linear-parametric C$^2$MAB-T, the goal of the subsequent sections is to design algorithms and conduct analysis to derive improved results for the nonlinear-parametric CLogB setting. Before delving into the algorithm and analysis, we first demonstrate how our CLogB model with these conditions applies to various applications, such as cascading bandits and probabilistic maximum coverage bandits, to showcase the broad applicability of our framework.

\section{Representative Applications}\label{sec:repr_app}

\begin{table*}[t]
	\caption{Summary of the coefficients, regret bounds, and computation costs for various applications.}	\label{tab:app}	
         \centering 
		\resizebox{1.
  \columnwidth}{!}{
	\centering
	\begin{threeparttable}
	\begin{tabular}{|ccccc|}
		\hline
		\textbf{Application}&\textbf{Condition}& \textbf{$(B_v, B_1, \lambda)$} & \textbf{Regret}& \textbf{Per-round cost}\\
	\hline
	Online content delivery (\cref{sec:app_PMC})& TPVM & $(3\sqrt{2|V|},1,2)$ & $\tilde{O}\left(d\sqrt{|V|T} + \kappa d^2\right)$ & $\tilde{O}\left(dK^2|V|^2T^2 + K^2 |U||V|\right)^*$\\
 			\hline
    		Online learning to rank (\cref{sec:app_ranking}) &  TPVM & $(1,1,1)$ & $\tilde{O}\left(d\sqrt{T} + \kappa d^2\right)$ & $\tilde{O}(dK^2T^2 + m)^\dagger$ \\
			\hline
  Dynamic channel allocation (\cref{sec:app_matching}) & 1-norm TPM &$(-,1,-)$   & $\tilde{O}\left( d\sqrt{\kappa |U|T} \right)$ & $\tilde{O}(d|U|^2T^2 + |U|^2|V| )^{**}$\\
	    \hline
		Reliable packet routing (\cref{sec:app_routing}) & TPVM & $(1,1,1)$ & $\tilde{O}\left(d\sqrt{T} + \kappa d^2\right)$ & $\tilde{O}(d|E|^2T^2+|V|+|E|)^\ddagger$\\
	    \hline
     	   
	\end{tabular}
	  \begin{tablenotes}[para, online,flushleft]
	\footnotesize
   \item[]\hspace*{-\fontdimen2\font}$^*$ $K, |U|, |V|$ denote the number of selected servers, the number of candidate servers, and the number of target users, respectively;
   	\item[]\hspace*{-\fontdimen2\font}$^\dagger$ $K$ and $m$ denote the length of the ranked list, the number of candidate items, respectively;
   \item[]\hspace*{-\fontdimen2\font}$^{**}$ $|U|, |V|, |E|$ denote the number of users, the number of channels, the number of candidate user-channel pairs, respectively;
  \item[]\hspace*{-\fontdimen2\font} $^\ddagger$ $|V|$ and $|E|$ denote the number of nodes and edges in a network, respectively.
	\end{tablenotes}
			\end{threeparttable}
	}
\end{table*}

In this section, we introduce two representative applications (with two additional applications in \cref{apdx_sec:app}) that can fit into our CLogB framework, as summarized in \cref{tab:app}. 

\subsection{Probabilistic Maximum Coverage Bandits: Online Content Delivery in CDN}\label{sec:app_PMC}
The probabilistic maximum coverage (PMC) problem is a simple yet powerful model that covers many practical network applications, such as network content delivery \cite{yang2018content} and mobile crowdsensing \cite{wang2023online}.
Typically, the PMC problem takes a bipartite graph $G=(U,V,E)$ as input, where $U$ are nodes to be selected, $V$ are nodes to be covered, each edge $(u,v)\in E$ is associated with a probability $p(u,v)$, and each node $v \in V$ is associated with another probability $p(v)$. A target node $v \in V$ can be covered by a node $u \in U$ with an independent probability $p(u,v)$ and any successfully covered node $v$ would have probability $p(v)$ to yield a reward of $1$. In this work, we consider two cases the non-triggering case when $p(v)=1$ and the triggering case where $p(v)\neq 1$ for some $v$.
The decision maker's goal is to select at most $k$ nodes from $U$ so as to maximize the total rewards given by the covered nodes in $V$. 

In a content delivery network (CDN) as shown in \cref{fig:cdn}, for example, contents (e.g., pictures, videos) are cached across mirror servers so that end users can access the contents swiftly via the nearby server ~\cite{dai2024axiomvision,chen2018spatio}. How to strategically choose a set of servers (of size $k$) to improve the user experience can be modeled by PMC, where $U$ models the set of candidate mirror servers sending contents, $V$ represents the set of users that consume contents. For each edge $(u,v) \in E$, $\mu_{(u,v)}$ models the probability that the content can be successfully delivered from server $u$ to node $v$ in time (i.e., $u$ covers $v$), and $\mu_v$ is the probability that $v$ ultimately consumes the contents. The goal of PMC is to provide the best possible user experience, i.e., to maximize the total number of users that ultimately consume the content $r(S;\boldsymbol{\mu})=\sum_{v \in S}\mu_v\left(1-\prod_{i\in S} (1-\mu_{(u,v)})\right)$. This reward function is a submodular function regarding $S$ and so the greedy algorithm can yield a $\alpha=(1-\nicefrac{1}{e})$-approximation in $T_{\alpha}=O(|V|^2|E|)$ \cite{kempe2003maximizing}.

PMC bandit is the online learning version of the probabilistic maximum coverage problem. 
At each round $t \in [T]$, the learner aims to select $K$ servers $S_t \subseteq U$ (i.e., super arm) that can cache the content $t$ and deliver the content to users through the CDN network. 
Before making the decision, the learner collects contextual information $c_t$, which includes users' demographics (gender, age, location, etc.), content attributes (video category, quality, reviews), and network conditions (network latency, bandwidth, jitter, etc.) between each server and user \cite{ye2023data, paganelli2014context}.
Then, each user $v\in V$ will attempt to prefetch the content from servers $S_t$ to their device. 
After receiving the request, the selected servers $u\in S_t$ will independently deliver contents with unknown success probability $\mu_{t,(u,v)}=\ell({\btheta^*}^{\top} \bphi_t(u,v))$, depending on the unknown parameter $\btheta^*\in \R^d$ and the feature vector $\bphi_t(u,v)$ that relates to the network condition in the contextual information $c_t$.
By ``success", we mean the content is delivered in a \textit{high-quality and timely} manner, which can be modeled as a Bernoulli random variable $X_{t,(u,v)} \in \{0,1\}$. 
Here the sigmoid function $\ell$ captures the nonlinear relationship between the success probability and the network factors like bandwidth, an S-shape curve  (\cref{fig:logistic}) where a small increase in the bandwidth can significantly boost the video quality in some critical ranges but plateaus at both low and high extremes \cite{lee2005non}.
If there exists any $u\in S_t$ such that $X_{t,(u,v)}=1$, then the content is successfully prefetched to the user $v$.
Then, the user decides whether or not to consume the content (e.g., view the video), which can be modeled by a Bernoulli random variable $X_{t,v}$ with success probability $\mu_{t,v}=\ell({\btheta^*}^{\top} \bphi_t(v))$, depending on the unknown parameter $\btheta^*\in \R^d$ and the feature vector $\bphi_t(u,v)$ that relates to the user $v$'s demographics and the content $t$'s attributes.
Again, the sigmoid function models the nonlinearity between the consumption probability and factors like user ratings, i.e., the probability of consumption increases dramatically from the rating of 3 to 4 but slowly increases from 1 to 2 or 4 to 5.
The expected number of users that consume the content, i.e., the reward function, can be expressed as $r(S_t;\boldsymbol{\mu}_t)=\sum_{v \in S_t}\mu_{t,v}\left(1-\prod_{i\in S_t} (1-\mu_{t,(u,v)})\right)$. 
The learner's goal is to maximize the maximum number of users who consume the content to maximize the user experience. 
As for the feedback, the learner can observe whether the contents are successfully delivered from the selected servers, i.e., $X_{t,(u,v)}$ for $u \in S_t, v\in V$. 
Additionally, for any user $v$ that successfully prefetches the content, the learner can observe whether $v$ consumes the content or not, i.e., $X_{t,v}$ for $\{v: \exists u\in S_t \text{ s.t. }X_{t,(u,v)}=1\}$. 
For this application, it fits the CLogB framework, satisfying the 1-norm TPVM smoothness condition with \( B_v = 3\sqrt{2|V|}, B_1 = 1, \lambda = 2 \) as in \cite{liu2023variance,liu2024learning}.

\subsection{Cascading Bandit: Online Learning to Rank in Recommendation Systems}\label{sec:app_ranking}

Learning to rank \cite{liu2009learning} is an approach used to improve the ordering of items (e.g., products, ads) in recommendation systems based on user interactions in real time. 
This approach is crucial for various types of recommendation systems, such as search engines \cite{hu2018reinforcement}, e-commerce platforms \cite{karmaker2017application}, and content recommendation services \cite{karatzoglou2013learning}. 


The cascading bandit problem \cite{kveton2015cascading,li2016contextual,vial2022minimax} addresses the \textit{online} learning to rank problem under the cascade model \cite{craswell2008experimental}. Specifically, we consider a $T$-round sequential decision-making process. At each round \( t \in [T] \), a user \( t \) comes to the recommendation system (like Amazon), and the learner aims to recommend a ranked list \( S_t \) of length \( K \) (i.e., a super arm) from a total of \( m \) candidate products (i.e., base arms). The learner first observes the feature map \( \bphi_t \), which maps each item \( i \in [m] \) to a feature vector \( \bphi_t(i) \in \mathbb{R}^d \), considering factors such as item attributes like relevance score, price, and user reviews. Based on \( \bphi_t \), the learner selects the ranked list \( S_t = (a_{t,1}, ..., a_{t,K}) \subseteq [m] \), where each item \( i \in S_t \) has a probability \( \mu_{t,i} = \ell({\btheta^*}^{\top}, \bphi_t(i)) \) of being satisfactory and purchased by user \( t \). Here, \( \btheta^* \in \mathbb{R}^d \) is an unknown parameter related to user demographics (e.g., age, occupation), and the sigmoid function $\ell$ (\cref{fig:logistic}) models the nonlinear relationship between the purchase probability and the features. For instance, the sigmoid function effectively captures the impact of relevance, where a small increase in relevance score can significantly boost the likelihood of purchase in a certain range of relevance scores but plateaus at both low and high extremes.
The user examines the list from \( a_{t,1} \) to \( a_{t,K} \) until they purchase the first satisfactory item (and leave the system) or exhaust the list without finding a satisfactory item. If the user purchases an item (suppose the \( j_t \)-th item), the learner receives a reward of 1 and observes Bernoulli outcomes of the form \( (0,...,0, 1, \text{x}, ..., \text{x}) \), meaning the first \( j_t-1 \) items are unsatisfactory (denoted as 0), the \( j_t \)-th item is satisfactory (denoted as 1), and the outcomes of the remaining items are unobserved (denoted as x).
If the user exhausts the list and finds no satisfactory items, the learner receives a reward of 0 and observes Bernoulli outcomes \( (0,0,...,0) \), indicating all \( K \) items are unsatisfactory. The expected reward is \( r(S_t; \boldsymbol{\mu}_t) = 1 - \prod_{i \in S_t} (1 - \mu_{t,i}) \).
For this application, it fits into the CLogB framework, satisfying the 1-norm TPVM smoothness condition with \( B_v = 1, B_1 = 1, \lambda = 1 \) as in \cite{liu2023contextual}. The offline oracle is essentially to find the top-k items regarding $\mu_{t,i}$, which maximizes $r(S_t; \boldsymbol{\mu}_t)$ in $O(m\log K)$ using the max-heap.

\section{Variance-Agnostic {CL\lowercase{og}UCB} Algorithm and Regret Analysis under 1-norm TPM Smoothness Condition}\label{sec:vag}
In this section, we first introduce the parameter learning process, which utilizes the maximum likelihood estimation (MLE) and lays the foundations of our combinatorial UCB-based algorithms throughout the paper.
Leveraging this MLE, we introduce the variance-agnostic confidence region that allows a variance-agnostic exploration bonus for each base arm.
Based on this exploration bonus, we design a simple yet efficient CLogUCB algorithm and prove the first regret bound for applications under the 1-norm TPM smoothness condition.

\subsection{Maximum Likelihood Estimation}

Unlike linear parametric C$^2$MAB works \cite{qin2014contextual,liu2023contextual} that uses the linear regression with a closed-form solution, we leverage the maximum likelihood estimation in combinatorial logistic bandits to estimate the unknown parameter $\btheta^*$ based on historical data $\cH_t=(\bphi_s, S_s, \tau_s, (X_{s,i})_{i\in \tau_s})_{s<t} \bigcup \bphi_t$. 
Specifically, we consider the following regularized log-likelihood (or cross-entropy loss) for $t\in [T]$:
\begin{equation}\label{eq:log-likelihood}
\begin{array}{r}
\mathcal{L}_t(\btheta)\defeq-\sum_{s=1}^{t-1}\sum_{i\in \tau_s}\left[X_{s,i} \log \ell\left(\btheta^{\top}\bphi_s(i) \right)+\left(1-X_{s,i}\right)\right. 
\left.\cdot \log \left(1-\ell\left(\btheta^{\top}\bphi_s(i) \right)\right)\right]+\frac{\lambda_t}{2}\|\btheta\|_2^2.
\end{array}
\end{equation}
where $\lambda_t=O(d\log (Kt))$ is a time-varying regularizer that will be specified later on in Algorithms~\ref{alg:CLogUCB}-\ref{alg:EVA_CLogB}. Our MLE estimator is defined as 
\begin{align}\label{eq:MLE_minimizer}\textstyle
    \hat{\btheta}_t\defeq\arg\min_{\btheta \in \R^d}\mathcal{L}_t(\btheta).
\end{align}

For this loss function $\mathcal{L}_t(\btheta)$, it is convenient to define a mapping $\bg_t:\R^d\rightarrow \R^d$
\begin{align}\textstyle
    \bg_t(\btheta)\defeq\sum_{s=1}^{t-1}\sum_{i\in \tau_s}\ell\left(\btheta^{\top}\bphi_s(i) \right)\bphi_s(i)+\lambda_t\btheta.
\end{align}

Then we can express the gradient $\nabla_{\btheta}\mathcal{L}_t(\btheta)$ at $\btheta$ as:
\begin{align}\label{eq:grad_MLE}\textstyle
    \nabla_{\btheta}\mathcal{L}_t(\btheta)=\bg_t(\btheta) - \sum_{s=1}^{t-1}\sum_{i\in \tau_s} X_{s,i} \bphi_s(i).
\end{align}

Lastly, we define Hessian $\bH_t(\btheta)$ of the log-loss and the covariance matrix $\bV_t$ as two important quantities used in our algorithm design and analysis. 
\begin{align}\textstyle
    \bH_t(\btheta)&\defeq\sum_{s=1}^{t-1}\sum_{i\in \tau_s}\dot{\ell}\left(\btheta^{\top}\bphi_s(i) \right)\bphi^{\top}_s(i)\bphi_s(i)+\lambda_t\bI_d ,\label{eq:H_mat}\\
    \bV_t&\defeq\sum_{s=1}^{t-1}\sum_{i\in \tau_s}\bphi^{\top}_s(i)\bphi_s(i)+\kappa\lambda_t\bI_d.\label{eq:V_mat}
\end{align}

Based on the fact that Hessian $\bH_t(\btheta)\succ \boldsymbol{0}$, we know that $\mathcal{L}_t(\btheta)$ in \cref{eq:log-likelihood} is a $\lambda_t$-strongly convex function and there exists a unique minimizer, which is exactly our estimated parameter in \cref{eq:log-likelihood}.

\begin{remark}[Variance-awareness]\label{rmk:variance_aware}
    Note that $\bH_t(\btheta)$ is different from $\bV_t$ mainly in that each covariant $\bphi_s(i)\bphi_s^{\top}(i)$ of $\bH_t(\btheta)$ is weighted by the first order derivative $\dot{\ell}(\btheta^{\top}\bphi_s(i))=\ell(\btheta^{\top}\bphi_s(i))(1-\ell(\btheta^{\top}\bphi_s(i)))$. 
Interestingly, for any Bernoulli random variable $X\sim \text{Bernoulli}(\ell(x))$, the variance $\Var[X]=(1-\ell(x))\ell(x)=\dot{\ell}(x)$. 
With this connection, $\bH_t(\btheta)$ actually involves the variance information $\dot{\ell}(\btheta^{\top}\bphi_s(i))=\Var[\text{Bernoulli}(\ell(\btheta^{\top}\bphi_s(i)))]$, which is the key to obtain the variance-adaptive exploration bonus for tighter regret bounds in \cref{sec:vad} and \cref{sec:eva}.
\end{remark}

We conclude this section by showing that $\hat{\btheta}_t$ is a good estimator by bounding the distance between $\hat{\btheta}_t$ and $\btheta^*$ via mapping $g_t$. The proof of \cref{lem:est_dist} is in \cref{apdx_sec:proof_concen_ineq}.

\begin{lemma}[Concentration inequality for MLE]\label{lem:est_dist}
Let $\hat{\btheta}_t$ be the MLE as defined in \cref{eq:MLE_minimizer}, it holds with probability at least $1-\delta$ that:
$$ \left\|\bg_t\left(\hat{\btheta}_t\right)-\bg_t\left(\btheta^*\right)\right\|_{\bH_t^{-1}(\btheta^*)} \leq \gamma_t(\delta), \quad \forall t \geq 1,
$$
where  $\gamma_t(\delta)\defeq\left(L+\nicefrac{3}{2}\right)\sqrt{d\log \left(\nicefrac{4(1+tK)}{\delta}\right)}$ is the confidence radius. Therefore, it holds with probability at least $1-\delta$ that $\{\forall t \geq 1, \btheta^* \in \cA_t(\delta)\}$, where $\cA_t(\delta)\defeq \left\{\btheta \in \Theta:
\left\|\bg_t\left(\hat{\btheta}_t\right)-\bg_t\left(\btheta\right)\right\|_{\bH_t^{-1}(\btheta)} \leq \gamma_t(\delta) \right\}.$
\end{lemma}

\begin{remark}[Comparison with Lemma 1 in \cite{faury2020improved}]
Different from \citet{faury2020improved} where a single arm is pulled in each round, at most $K$ base arms could be triggered as feedback, resulting in an enlarged confidence radius $\gamma_t(\delta)$ in \cref{lem:est_dist}. An important note is that $K$ only appears in the logarithmic term, preserving the order of confidence radius comparable to the single-arm case. 
\end{remark}

\begin{remark}[Computation cost for MLE]\label{rmk:MLE_cost}
    To compute $\hat{\btheta}_t$ in \cref{eq:MLE_minimizer}, an efficient way is to use (projected) gradient descent \cite{hazan2016introduction}. Specifically, for a $\nu$-strongly convex and $L'$-smooth loss function, it takes $O\left(\nicefrac{L'}{\nu}\log \varepsilon\right)$ steps to yield a solution with $\varepsilon$-accuracy for any $\varepsilon >0$. In our case, the Hessian matrix in \cref{eq:H_mat} gives $\nu\ge\lambda_t>1$ and $L'\le KT+\lambda_t$, so it takes $O(Kt \log T)$ steps to give a $\nicefrac{1}{T}$-accurate solution, which is sufficient for the theoretical guarantee. For each step, it takes $O(dKt)$ to compute the gradient in \cref{eq:grad_MLE}. 
    So the overall per-round cost for MLE is $\tilde{O}(dK^2T^2)$.
\end{remark}



\subsection{Variance-Agnostic Confidence Region and Exploration Bonus}\label{sec:vag_conf_bonus}

There are two types of optimistic algorithms that navigate the exploration of unknown parameters: (1) algorithms based on the upper confidence bound (UCB)~\cite{faury2020improved}, which adds an exploration bonus directly to each arm; and (2) algorithms based on optimism in the face of uncertainty (OFU)~\cite{abeille2021instance,faury2022jointly}, which jointly searches the action and the parameter space within the confidence region. When considering the classical logistic bandit where a single arm is pulled in each round, OFU-based is favorable since it (1) removes some unnecessary algorithmic complexity (e.g., some nonconvex projection steps); and (2) gives a neater analysis. For the combinatorial logistic bandit setting, however, the OFU-based algorithm is highly inefficient since it jointly searches ${(S_t,\btheta_t)= \argmax_{S \in \mathcal{\cS},\btheta \in \cC_t(\delta)} r(S;\mu_t(\btheta))}$ where ${\mu_t(\btheta)\defeq\left(\ell\left(\btheta^{\top}\bphi_t(i)\right)\right)_{i\in[m]}\in [0,1]^d}$, and needs to enumerate the combinatorial space $\cS$ for each possible parameter $\btheta$ in the confidence region $\cC_t(\delta)$, resulting in exponentially large time complexity. 
Therefore, we focus on the UCB-based approach which is more tractable in our CLogB setting.

To devise a UCB-based algorithm, the key is to build an exploration bonus for each arm.
Using \cref{lem:est_dist}, we first construct the following variance-agnostic confidence region around the MLE $\hat{\btheta}_t$, where the parameter distance is directly bounded without the mapping $g_t$ in \cref{lem:est_dist}. By "variance-agnostic", we mean that the confidence region is only related to covariance matrix $\bV_t$ without any variance information as discussed in \cref{rmk:variance_aware}. The proof of \cref{lem:var_ag_conf} is in \cref{apdx_sec:proof_ag_conf}.  
\begin{lemma}[Variance-agnostic confidence region]\label{lem:var_ag_conf}
    Let $\delta \in (0,1]$ and set the confidence radius $\beta_t(\delta)\defeq \left(L^2+4L+\nicefrac{19}{4}\right) \sqrt{\kappa d\log \left(\nicefrac{4(1+tK)}{\delta}\right)}$. The following region
    \begin{equation}
\mathcal{B}_t(\delta):=\left\{\btheta \in \Theta :\left\|\btheta-\hat{\btheta}_t\right\|_{\mathbf{\bV}_{\mathrm{t}}} \leq \beta_t(\delta)\right\}, 
\end{equation}
 is an anytime confidence region for $\btheta^*$ with probability at least $1-\delta$, i.e., 
    \begin{align}
        \Pr\left(\forall t\ge 1, \btheta^* \in \mathcal{B}_t(\delta)\right)\ge 1-\delta.
    \end{align}
\end{lemma}

Based on the above confidence region, we can now construct our variance-agnostic exploration bonus $\rho_{t,V}(i)$ in order to upper bound the true base-arm reward ${\mu_{t,i}=\ell\left({\btheta^*}^{\top}\bphi_t(i)\right)}$ using the estimated base-arm reward $\hat{\mu}_{t,i}=\ell\left(\hat{\btheta}_t^{\top}\bphi_t(i)\right)$. The proof of \cref{lem:var_ag_bonus} is in \cref{apdx_sec:var_ag_bonus}.

\begin{lemma}[Variance-agnostic exploration bonus]\label{lem:var_ag_bonus}
    Let $B_t(\delta)$ be the confidence region with the confidence radius $\beta_t(\delta)$ as defined in \cref{lem:var_ag_conf}. 
    Let the exploration bonus be
    \begin{align}\label{eq:var_ag_conf_V}\textstyle
        \rho_{t,V}(i)\defeq\frac{1}{4}\beta_t(\delta)\norm{\bphi_t(i)}_{\bV_t^{-1}}.
    \end{align} 
    Under the event $\left\{\forall t\ge 1, \btheta^* \in \mathcal{B}_t(\delta)\right\}$, it holds that, for any $i\in[m], t\ge 1$,
    \begin{align}\textstyle
        \ell\left({\btheta^*}^{\top}\bphi_t(i)\right)\le \ell\left(\hat{\btheta}_t^{\top}\bphi_t(i)\right)+\rho_{t,V}(i)\le \ell\left({\btheta^*}^{\top}\bphi_t(i)\right)+2\rho_{t,V}(i).
    \end{align}
\end{lemma}

Finally we can use the variance-agnostic upper confidence bound $\bar{\mu}_{t,i}\defeq \ell\left(\hat{\btheta}_t^{\top}\bphi_t(i)\right)+\rho_{t,V}(i)$ as our optimistic estimation of the true mean $\mu_{t,i}$ to balance the exploration-exploitation tradeoff.

\subsection{Variance-Agnostic CLogUCB Algorithm and Regret Bound under 1-Norm TPM Smoothness Condition}

\begin{algorithm}[t]
	\caption{CLogUCB: \textbf{C}ombinatorial \textbf{Log}isitc  \textbf{U}pper \textbf{C}onfidence \textbf{B}ound Algorithm for CLogB}\label{alg:CLogUCB}
			\resizebox{1\columnwidth}{!}{
\begin{minipage}{\columnwidth}
	\begin{algorithmic}[1]
	    \State {\textbf{Input:}} Base arms $[m]$, dimension $d$, parameter space $\Theta$, nonlinearity coefficient $\kappa$, probability $\delta=\nicefrac{1}{T}$, offline ORACLE. 
	   \For{$t=1, ...,T$ }
	   \State Compute MLE $\hat{\btheta}_{t}= \argmax_{\btheta\in \R^d}\mathcal{L}_{t}(\btheta)$ according to \cref{eq:log-likelihood} with $\lambda_t=\lbdt$.     \label{line:clogb_estimate}
        \State Compute the covariance matrix $\bV_t$ according to \cref{eq:V_mat}.\label{line:comp_V}
	   \For{$i \in [m]$}
	    \State $\bar{\mu}_{t,i}=\ell\left(\hat{\btheta}_{t}^{\top}\bphi_t(i)\right)+\beta_t(\delta)\norm{\bphi_t(i)}_{\bV_t^{-1}}$ with $\beta_t(\delta)=\left(L^2+4L+\nicefrac{19}{4}\right) \sqrt{\kappa d\log \left(\nicefrac{4(1+tK)}{\delta}\right)}$.\label{line:clogb_ucb}
	    \EndFor
	   \State $S_t=\text{ORACLE}(\bar{\mu}_{t,1}, ..., \bar{\mu}_{t,m})$ as in \cref{eq:def_oracle}. 
	   \State Play $S_t$ and observe triggering arm set $\tau_t$ with their outcomes $(X_{t,i})_{i\in \tau_{t}}$.
	    	   \EndFor
		\end{algorithmic}
           		\end{minipage}}
\end{algorithm}

In this section, we provide the first CLogB-UCB algorithm (\cref{alg:CLogUCB}) for the combinatorial logistic bandits, which generalizes the logistic-UCB-1 algorithm \cite{faury2020improved} to the combinatorial setting meanwhile improving its efficiency by removing the nonconvex projection.
In \cref{line:clogb_estimate}, \cref{alg:CLogUCB} estimates $\btheta^*$ using the MLE in \cref{eq:log-likelihood}. 
We would like to emphasize that, different from logistic-UCB-1, we do not need to project $\hat{\btheta}_t$ back to $\Theta$ through a nonconvex minimization routine $\argmin_{\btheta \in \Theta}\norm{\bg_t(\btheta)-\bg_t(\hat{\btheta}_t)}_{\bH_t^{-1}(\btheta)}$ when $\hat{\btheta}_t\notin \Theta$. This is done using a slightly enlarged confidence radius $\beta_t(\delta)$.
By removing the projection, the computational efficiency of \cref{alg:CLogUCB} is largely improved. 
In \cref{line:comp_V}, we compute the covariance matrix $\bV_t$ in order to compute the exploration bonus $\rho_{t,V}(i)$ defined as in \cref{eq:var_ag_conf_V}.
In \cref{line:clogb_ucb}, we construct an upper confidence bound $\bar{\mu}_{t,i}$ for each arm $i$ based on \cref{lem:var_ag_conf}, where $\ell\left(\hat{\btheta}_{t}^{\top}\bphi_t(i)\right)$ is the MLE estimation of $\mu_{t,i}$ and $\beta_t(\delta)\norm{\bphi_t(i)}_{\bV_t^{-1}}$ is the exploration bonus in the direction $\bphi_t(i)$.
After computing the UCB values $\bar{\mu}_t$, the learner selects action $S_t$ through the offline oracle with $\bar{\bmu}_t$ as input.
Then, the base arms in $\tau_t$ are triggered, and the learner receives the observation set $(X_{t,i})_{i\in \tau_{t}}$ as feedback to improve future decisions.
Now we provide the regret upper bound for applications under the 1-norm TPM condition.

\begin{theorem}\label{thm:var_ag_thm}
     For a CLogB instance that satisfies monotonicity (\cref{cond:mono}) and 1-norm TPM smoothness (\cref{cond:TPM}) with coefficient $B_1$, CLogUCB (\Cref{alg:CLogUCB}) with an $\alpha$-approximation oracle achieves an $\alpha$-approximate regret bounded by $  O\left(B_1  d\sqrt{\kappa KT}\log (KT)\right).$
\end{theorem}

\textbf{Discussion.} The regret bound in \cref{thm:var_ag_thm} is independent of the number of arms $m$ and has a sublinear dependence on the nonlinearity level $\kappa$ and the action-size $K$. Considering the contextual cascading bandits in \cref{sec:app_ranking}, our result improves \cite{li2018online} up to a factor of $O(\sqrt{\kappa})$. Consider the degenerate case when $K=1$, our result matches the lower bound $\Omega(d\sqrt{\nicefrac{T}{\kappa}})$ \cite{abeille2021instance} up to a factor of $\tilde{O}(\kappa)$. 
As for the per-round computation cost, \cref{alg:CLogUCB} contains three parts: (1) computing MLE takes $\tilde{O}(dK^2T^2)$ as discussed in \cref{rmk:MLE_cost}, (2) computing the covariance matrix $\bV_t$ takes $O(d^2)$ and computing UCB for all $m$ base arms takes $O(d^3 + d^2m)$, (3) the $\alpha$-approximation oracle takes $T_{\alpha}$. So, the overall cost per round is $O(dK^2T^2 + T_{\alpha})$, which saves a $T_{\text{nc}}$ (which corresponds to a nonconvex projection problem that could be NP-hard) compared to directly using the logistic-UCB-1 algorithm in \cite{faury2020improved}. The proof of \cref{thm:var_ag_thm} is in \cref{apdx_sec:regret_ag}.

\begin{proof}[Proof idea]
Recall the history $\cH_t=(\bphi_s, S_s, \tau_s, (X_{s,i})_{i\in \tau_s})_{s<t} \bigcup \bphi_t$ and let $\E_t=\E[\cdot | \cH_t]$.
We bound the instantaneous regret at round $t$ by $\E_t[\alpha\cdot r(S^*_t;\bmu_t)-r(S_t;\bmu_t)]{\le}\E_t[\alpha\cdot r(S^*_t;\bar{\bmu}_t)-r(S_t;\bmu_t)]\le\E_t[r(S_t;\bar{\bmu}_t)-r(S_t;\bmu_t)]$ based on monotonicity (\cref{cond:mono}), $\bar{\bmu}_t \ge \bmu_t$ by \cref{lem:var_ad_bonus}, and $S_t$ being an $\alpha$-approximation. Leveraging the 1-norm TPM condition, the fact that $p_i^{\bmu_t,S_t}=\E_t[i\in \tau_t]$, and \cref{lem:var_ad_bonus}, we have $\E_t[r(S_t;\bar{\bmu}_t)-r(S_t;\bmu_t)]\le \E_t\left[\sum_{i \in [m]}B_1p_i^{\bmu_t,S_t}(\bar{\mu}_{t,i}-\mu_{t,i})\right]=\E_t\left[\sum_{i \in \tau_t}B_1(\bar{\mu}_{t,i}-\mu_{t,i})\right]\le \E_t\left[\sum_{i \in \tau_t}2B_1\rho_{t,V}\right]$. Consider over all $t\in [T]$ and the definition of exploration bonus $\rho_{t,V}$, we have $\text{Reg}(T)\le \E[\sum_{t\in [T]}\sum_{i\in \tau_t}2B_1\rho_{t,V}]\le \frac{1}{2}B_1\beta_T(\delta)\E\left[\sum_{t \in [T]}\sum_{i \in \tau_t}\norm{\bphi_t(i)}_{\bV_{t}^{-1}}\right]$. 
We conclude the theorem by $\sum_{t \in [T]}\sum_{i \in \tau_t}\norm{\bphi_t(i)}_{\bV_{t}^{-1}}=\tilde{O}(\sqrt{dKT})$ using the elliptical potential lemma (\Cref{apdx_lem:elliptical_potential}) and $\beta_T(\delta)=\tilde{O}({\sqrt{d \kappa}})$ by \cref{lem:var_ad_bonus}. 
\end{proof}

\section{Variance-Adaptive {VA-CL\lowercase{og}UCB} Algorithm and Regret Analysis under 1-Norm TPM and TPVM Smootheness Conditions}\label{sec:vad}
In this section, we introduce a variance-adaptive confidence region that gives an improved variance-aware exploration bonus for each base arm. 
Based on this improved exploration bonus, we devise a variance-adaptive algorithm called VA-CLogUCB and prove an improved regret bound that removes the nonlinearity $\kappa$ dependence in the leading term under the 1-norm TPM smoothness condition.
Finally, for applications that satisfy the stronger TPVM smoothness conditions, we prove that VA-CLogUCB can further improve the regret bound and achieve the "best-of-both-worlds" result, removing both the $\kappa$ term and the action-size dependence $K$ in the leading term of the regret.

\subsection{Variance-Adaptive Confidence Region and Exploration Bonus}\label{sec:vad_conf_bouns}
In this section, we propose a more refined confidence region that entails the variance information. 
Before we introduce this new confidence region, let us modify MLE $\hat{\btheta}_t$ by projecting it onto the following bonus-vanishing region $\cQ_t$ in order to achieve our learning guarantee:
\begin{align}\label{eq:Q_region}\textstyle
    \cQ_t = \left\{\btheta \in \Theta: \abs{\btheta^{\top}\bphi_s(i)} \le \max_{\btheta' \in \cA_s(\delta)}\abs{{\btheta'}^{\top}\bphi_s(i)} \text{ for all } i \in \tau_s, s\in [t]\right\},
\end{align}
where $\cA_s(\delta)$ is the confidence region as defined in \cref{lem:est_dist}.

Given the above region $\cQ_t$, the new MLE $\hat{\btheta}_{t,H}$ is computed by the following nonconvex projection
\begin{align}\textstyle
    \hat{\btheta}_{t,H}=\argmin_{\btheta \in \cQ_t} \norm{\bg_t(\btheta)-\bg_t(\hat{\btheta}_t)}_{\bH_t^{-1}(\btheta)}.
\end{align}

At a high level, this projection serves two purposes: (1)  constrains the MLE $\hat{\btheta}_t$ to feasible region $\Theta$ so that \cref{lem:var_ad_conf} holds with high probability, and (2) makes the exploration bonus shrink in a steady rate so that the final regret can be bounded. 
\begin{lemma}[Variance-adaptive confidence region]\label{lem:var_ad_conf}
    Let $\delta \in (0,1]$ and set the confidence radius $\sigma_t(\delta):=\left(2L+1\right) \left(2L+3\right)\sqrt{d\log \left(\nicefrac{4(1+tK)}{\delta}\right)}$. The following region
    \begin{equation}
\mathcal{C}_t(\delta):=\left\{\btheta \in \Theta: \left\|\btheta-\hat{\btheta}_{t,H}\right\|_{\mathbf{\bH}_{\mathrm{t}}(\hat{\btheta}_{t,H})} \leq \sigma_t(\delta)\right\}, 
\end{equation}
 is an anytime confidence region for $\btheta^*$ with probability at least $1-\delta$, i.e., 
    \begin{align}
        \Pr\left(\forall t\ge 1, \btheta^* \in \mathcal{C}_t(\delta)\right)\ge 1-\delta.
    \end{align}
\end{lemma}

Based on the above confidence region, we can now construct our variance-adaptive exploration bonus $\rho_{t,H}(i)$ in order to upper bounds the true base-arm reward ${\mu_{t,i}=\ell\left({\btheta^*}^{\top}\bphi_t(i)\right)}$ using the estimated reward ${\hat{\mu}_{t,i}=\ell\left(\hat{\btheta}_{t,H}^{\top}\bphi_t(i)\right)}$. The proof of \cref{lem:var_ad_conf} and \cref{lem:var_ad_bonus} are in \cref{apdx_sec:proof_vad_conf} and \cref{apdx_sec:proof_vad_bonus}, repsectively.
\begin{lemma}[Variance-adaptive exploration bonus]\label{lem:var_ad_bonus}
    Let $C_t(\delta)$ be the confidence region with the confidence radius $\sigma_t(\delta)$ as defined in \cref{lem:var_ad_conf}. 
    Let the exploration bonus be
    \begin{align}\label{eq:var_ad_conf_H}
        \rho_{t,H}(i)\defeq \sigma_t(\delta)\dot{\ell}\left(\bphi_t(i)^{\top}\hat{\btheta}_{t,H}\right) \norm{\bphi_t(i)} _{\bH_t^{-1}(\hat{\btheta}_{t,H})}+\frac{1}{8} \kappa \sigma_t^2(\delta) \norm{\bphi_t(i)} ^2  _{\bV_t^{-1}}.
    \end{align} 
    Under the event $\left\{\forall t\ge 1, \btheta^* \in \mathcal{C}_t(\delta)\right\}$, it holds that, for any $i\in[m], t\ge 1$,
    \begin{align}
        \ell\left({\btheta^*}^{\top}\bphi_t(i)\right)\le \ell\left(\hat{\btheta}_{t,H}^{\top}\bphi_t(i)\right)+\rho_{t,H}(i)\le \ell\left({\btheta^*}^{\top}\bphi_t(i)\right)+2\rho_{t,H}(i).
    \end{align}
\end{lemma}

Finally, we can use the variance-adaptive upper confidence bound $\bar{\mu}_{t,i}\defeq \ell\left(\hat{\btheta}_{t,H}^{\top}\bphi_t(i)\right)+\rho_{t,H}(i)$ as our optimistic estimation of the true mean $\mu_{t,i}$. Compared with the variance-agonistic exploration bonus $\rho_{t,E}(i)$, the first term of $\rho_{t,H}(i)$ is weighted by the estimated variance $\dot{\ell}\left(\bphi_t(i)^{\top}\hat{\btheta}_{t,H}\right)$, which more aggressively explores those arms that have low variance, resulting improved regret results.

\subsection{Variance-Adaptive VA-CLogUCB Algorithm and Improved Regret Bound Under 1-norm TPM and TPVM Conditions}


\begin{algorithm}[t]
	\caption{VA-CLogUCB: \textbf{V}ariance-\textbf{A}daptive \textbf{C}ombinatorial \textbf{Log}isitc  \textbf{U}pper \textbf{C}onfidence \textbf{B}ound Algorithm for CLogB}\label{alg:VA_CLogB}
			\resizebox{\columnwidth}{!}{
\begin{minipage}{\columnwidth}
	\begin{algorithmic}[1]
	    \State {\textbf{Input:}} Base arms $[m]$, dimension $d$, parameter space $\Theta$, probability $\delta=\nicefrac{1}{T}$,  offline ORACLE. 
     \State Set bonus-vanishing region $\cQ_1=\Theta$.
	   \For{$t=1, ...,T$ }
	   \State Compute MLE $\hat{\btheta}_{t}\defeq \argmax_{\btheta\in \R^d}\mathcal{L}_{t}(\btheta)$ according to \cref{eq:log-likelihood} with $\lambda_t=\lbdt$. \label{line:va_clogb_estimate}
    \State If $\hat{\btheta}_t\notin \cQ_t$, then compute projected MLE $\hat{\btheta}_{t,H} = \arg\min_{\btheta \in \cQ_t}\norm{ g_t(\btheta) - g_t(\hat{\btheta}_t) } _{\bH_t^{-1}(\btheta)}$.\label{line:proj_Q}
	   \For{$i \in [m]$}
	    \State $\bar{\mu}_{t,i}=\ell\left(\hat{\btheta}_{t,H}^{\top}\bphi_t(i)\right)+\rho_{t,H}(i)$ according to  \cref{eq:var_ad_conf_H}.\label{line:va_clogb_ucb}
	    \EndFor
	   \State $S_t=\text{ORACLE}(\bar{\mu}_{t,1}, ..., \bar{\mu}_{t,m})$ as in \cref{eq:def_oracle}. 
	   \State Play $S_t$ and observe triggering arm set $\tau_t$ with outcomes $(X_{t,i})_{i\in \tau_{t}}$.
    \State Update bonus-vanishing region $\cQ_{t+1}=\cQ_t \cap_{i\in \tau_t} \left\{\btheta: \abs{\btheta^{\top}\bphi_{t}(i)}\le \sup_{\btheta \in \cA_t(\delta)}\abs{\btheta^{\top}\bphi_{t}(i)}\right\}$.\label{line:update}
	    	   \EndFor
		\end{algorithmic}
           		\end{minipage}}
\end{algorithm}

In this section, we provide our new variance-adaptive algorithm VA-CLogUCB in \cref{alg:VA_CLogB}.
Compared with \cref{alg:CLogUCB}, VA-CLogUCB has two key differences: 
First, in \cref{line:va_clogb_ucb}, we use the variance-adaptive exploration bonus $\rho_{t,H}(i)$ where the leading $\dot{\ell}(\hat{\btheta}_{t,H}^{\top}\bphi_t(i))\norm{\bphi_t(i)}_{\bH_t^{-1}(\hat{\btheta}_{t,H})}$ term is weighted by the estimated variance $\Var[X_{t,i}]=\dot{\ell}(\hat{\btheta}_{t,H}^{\top}\bphi_t(i))$ as mentioned in \cref{rmk:variance_aware}, saving unnecessary explorations towards arms with low variance.
Second, in \cref{line:proj_Q}, we obtain the projected MLE $\hat{\btheta}_{t,H}$ by projecting the original MLE $\hat{\btheta}$ from \cref{line:va_clogb_estimate} onto the bonus-vanishing region $\cQ_t$. The intuition for the projection is that if we directly use the original MLE $\hat{\btheta}$ to replace \cref{line:va_clogb_estimate}, then the Hessian matrix $\bH_t(\hat{\btheta}_{t,H})$ could be very small, resulting in large exploration bonus $\rho_{t,H}(i)$ even at the very end of the learning process. 
To tackle this, we use the bonus-vanishing region $\cQ_t$, which is equivalent to $ \cQ_t = \left\{\btheta \in \Theta: \dot\ell(\btheta^{\top}\bphi_s(i)) \ge \min_{\btheta' \in \cC_s(\delta)}\dot{\ell}({\btheta'}^{\top}\bphi_s(i)) \text{ for all } i \in \tau_s, s\in [t]\right\}$, so that the Hessian matrix $\bH_t(\hat{\btheta}_{t,H})=\sum_{s=1}^{t-1}\sum_{i\in \tau_s}\dot{\ell}\left(\btheta^{\top}\bphi_s(i) \right)\bphi^{\top}_s(i)\bphi_s(i)+\lambda_t\bI_d$ can be lower bounded by some matrix that steadily increases over time, resulting in steadily vanishing exploration bonus.

Given the projected MLE $\hat{\btheta}_{t,H}$ and the variance-adaptive exploration bonus $\rho_{t,H}(i)$, \cref{alg:VA_CLogB} computes the UCB $\bar{\mu}_{t,i}$ for each arm $i$ in \cref{line:va_clogb_ucb}.
After computing the UCB values $\bar{\bmu}_t$, the learner selects action $S_t$ via the offline oracle with $\bar{\bmu}_t$ as input.
By playing $S_t$, the base arms in $\tau_t$ are triggered, and the learner receives observation set $(X_{t,i})_{i\in \tau_{t}}$ as feedback to improve decisions for future rounds.

Now, we give the improved regret bound that removes the $\kappa$ dependence under the 1-norm TPM smoothness condition (\cref{cond:TPM}). 
\begin{theorem}\label{thm:var_ad_thm1}
     For a CLogB instance that satisfies monotonicity (Condition~\ref{cond:mono}) and 1-norm TPM smoothness (Condition~\ref{cond:TPM}) with coefficient $B_1$, VA-CLogUCB (\Cref{alg:VA_CLogB}) with an $\alpha$-approximation oracle achieves an $\alpha$-approximate regret bounded by $ O\left(B_1 d\sqrt{KT}\log (KT)+ B_1 \kappa d^2\log ^2(KT)\right).$
\end{theorem}
Under the stronger 1-norm TPVM condition (\cref{cond:TPVMm}), we can further improve the regret guarantee by removing the batch-size dependence $K$. The detailed proofs of \cref{thm:var_ad_thm1} and \cref{thm:var_ad_thm2} are in \cref{apdx_sec:proof_thm_ad}.
\begin{theorem}\label{thm:var_ad_thm2}
    For a CLogB instance that satisfies monotonicity (Condition~\ref{cond:mono}) and the TPVM smoothness (Condition~\ref{cond:TPVMm}) with coefficient $(B_v,B_1, \lambda)$, if $\lambda \ge 1$, then VA-CLogUCB with an $\alpha$-approximation oracle achieves an $\alpha$-approximate regret bounded by $ O\left(B_v d\sqrt{T}\log (KT)+ B_1 \kappa d^2\log ^2(KT)\right)$.
\end{theorem}
\textbf{Discussion.} 
In \cref{thm:var_ad_thm1}, the leading regret term is $\tilde{O}(B_1 d\sqrt{KT})$. 
Compared with the variance-agnostic CLogUCB algorithm (\cref{alg:CLogUCB}), the VA-CLogUCB algorithm improves the leading regret of CLogUCB up to a factor of $\tilde{O}(\sqrt{\kappa})$.
Looking at \cref{thm:var_ad_thm2}, the leading regret is $\tilde{O}(B_v d\sqrt{T})$, which is totally independent of the nonlinearity level $\kappa$ and the action-size $K$.
When $B_v,B_1=O(1)$, VA-CLogUCB improves the leading regret of CLogUCB up to a factor of $\tilde{O}(\sqrt{\kappa K})$.
Consider the cascading bandit application in \cref{sec:app_ranking} where $B_v=B_1=1$, VA-CLogUCB achieves a regret bound of $\tilde{O}(d\sqrt{T}+\kappa d^2)$, which improves the result of \citet{li2018online} up to a factor of $
\tilde{O}(\kappa \sqrt{K})$.
Consider the degenerate case when $K=1$, we have $B_v=B_1=1$ and thus VA-CLogUCB matches the lower bound in~\cite{abeille2021instance} up to a factor of $\tilde{O}(\sqrt{\kappa})$.

\rev{For the dependency on $\kappa$, in the degenerate LogB problem with $K=1$, \citet{abeille2021instance} provide a lower bound of $\tilde{O}(d\sqrt{T/\kappa})$. Intuitively, as $T \to \infty$, the algorithm predominantly selects arms near the optimal arm $i^* = \argmax_{i \in [m]} \ell({\btheta^*}^{\top} \bphi(i))$, where the feature vector $\bphi(i^*)$ lies in the top-right region of \cref{fig:logistic}. When $\kappa$ is large, as shown in \cref{fig:logistic}, the reward curve around $i^*$ becomes relatively flat (with slope $\kappa^{-1}$). Therefore, the regret from selecting suboptimal arms near $i^*$ scales proportionally with $\kappa^{-1}$. 
For our variance-adaptive algorithms, VA-CLogUCB and EVA-CLogUCB, the current dependency on $\kappa$ for the leading regret term is $O(1)$, rather than the expected $\tilde{O}(\sqrt{1/\kappa})$. We hypothesize that this is primarily due to a technical artifact in the analysis. For instance, in \cref{apdx_eq:not_tight_1} to \cref{apdx_eq:not_tight_2} of our current proof, we take a step that upper-bounds $\dot{\ell}\left(\tilde{\btheta}_{t,i}^{\top} \bphi_t(i)\right)$ by $\sqrt{\dot{\ell}\left(\tilde{\btheta}_{t,i}^{\top} \bphi_t(i)\right)}$ for each base arm, potentially discarding a factor of the order $O(\sqrt{1/\kappa})$. However, this term $\sqrt{\dot{\ell}\left(\tilde{\btheta}_{t,i}^{\top} \bphi_t(i)\right)}$ could be $O(1)$, particularly in the early time slots. If we could fully exploit this term in the analysis, it might lead to a tighter regret upper bound with a dependency of $O(\sqrt{1/\kappa})$. This remains a challenging direction for future research.}

\rev{For the dependency on $K$, while it is true that with larger $K$, we gather more observations, allowing for more accurate estimation of each arm (since MLE is more accurate), the regret incurred by selecting any single base arm should theoretically decrease. However, in combinatorial bandits where $K$ base arms are selected and assume the reward function is $r(S; \bmu) = \sum_{i \in S} \mu_i$, the total regret per round is the summation of the regret from all $K$ base arms, which can be $K$ times larger. Consequently, even though individual arm accuracy improves with larger $K$, the total regret does not necessarily decrease. This is further validated by the lower bound of $\Omega(\sqrt{mKT})$ for classical combinatorial multi-armed bandits with a linear reward function, as shown in \citet{kveton2015tight}.}

As for the per-round computation cost, \cref{alg:VA_CLogB} contains four parts: (1) computing MLE takes $\tilde{O}(dK^2T^2)$ as discussed in \cref{rmk:MLE_cost}, (2) computing the nonconvex projection takes $T_{\text{nc}}$ if $\hat{\btheta}_{t,H} \notin \cQ_t$ (3) computing the Hessian matrix $\bH_t(\hat{\btheta}_{t,H})$ takes $O(d^2KT)$ and computing UCB for all $m$ base arms takes $O(d^3 +d^2m)$, (4) the $\alpha$-approximation oracle takes $T_{\alpha}$. So the overall per-round cost is $O(dK^2T^2 + T_{\text{nc}}+T_{\alpha})$, which has an additional $T_{\text{nc}}$ compared to CLogUCB algorithm (\cref{alg:CLogUCB}).

\begin{proof}[Proof idea]
For \cref{thm:var_ad_thm1}, we use the proof of \cref{thm:var_ag_thm}: $\text{Reg}(T)\le \E[\sum_{t\in [T]}\sum_{i\in \tau_t}2B_1 \rho_{t,H}]$ $\lesssim\E\left[\sum_{t \in [T]}\sum_{i \in \tau_t}B_1\left( \sigma_t(\delta)\dot{\ell}\left(\hat{\btheta}_{t,H}^{\top}\bphi_t(i)\right)\|\bphi_t(i)\|_{\mathbf{H}_t^{-1}(\hat{\btheta}_{t,H})} +\kappa \sigma_t^2(\delta)\|\bphi_t(i)\|_{\mathbf{V}_t^{-1}}^2\right)\right] $. The main challenge is that ${\mathbf{H}_t^{-1}(\hat{\btheta}_{t,H})}$ could be very large but thanks to that fact that $\hat{\btheta}_{t,H}\in \cQ_t$, where $\cQ_t$ is the bonus-vanishing region in \cref{eq:Q_region}, we have the lower bound of ${\mathbf{H}_t^{-1}(\hat{\btheta}_{t,H})} \gtrsim {\mathbf{H}_t^{-1}(\tilde{\btheta}_{t,i})}$, where  $\tilde{\btheta}_{t,i}\defeq \arg\min_{\btheta \in \mathcal{A}_t(\delta)}\dot{\ell}(\bphi_t(i)^{\top}\btheta)$. Let $\tilde{\bphi}_t(i)\defeq\dot{\ell}^{1/2}\left(\tilde{\btheta}_{t,i}^\top\bphi_{t}(i)\right)\bphi_t(i)$ and $\bL_t\defeq\sum_{s=1}^{t-1}\sum_{i\in\tau_s}\tilde{\bphi}_s(i)\tilde{\bphi}_s(i)^{\top}+\lambda_t \mathbf{I}_d$, we have $\text{Reg}(T)\lesssim\E\left[\sum_{t \in [T]}\sum_{i \in \tau_t}B_1\left( \sigma_t(\delta)\|\tilde{\bphi}_t(i)\|_{\mathbf{L}_t^{-1}} +\kappa \sigma_t^2(\delta)\|\bphi_t(i)\|_{\mathbf{V}_t^{-1}}^2\right)\right] $. We conclude the \cref{thm:var_ad_thm1}  by $\sum_{t \in [T]}\sum_{i \in \tau_t}\norm{\tilde{\bphi}_t(i)}_{\bL_{t}^{-1}}=\tilde{O}(\sqrt{dKT})$ using the elliptical potential lemma and $\sigma_T(\delta)=\tilde{O}({\sqrt{d}})$ by \cref{lem:var_ad_bonus} (saving a $\sqrt{\kappa}$ factor compared to $\beta_t(\delta)$ in \cref{thm:var_ag_thm}). 

For \cref{thm:var_ad_thm2}, we use the stronger TPVM condition (\cref{cond:TPVMm}) and obtain $\text{Reg}(T)\lesssim \sum_{t=1}^T B_v\sigma_t(\delta) \sqrt{\E\left[\sum_{i \in [\tau_t]} \frac{\dot{\ell}^2\left( \bphi_t(i)^{\top} \tilde{\btheta}_{t,i} \right)} {\mu_{t,i}(1-\mu_{t,i})} \norm{\bphi_t(i)}^2 _{\bL_t^{-1}}\right]}$ $+ \E\left[\sum_{t\in [T]}\sum_{i\in [\tau_t]} B_1 \kappa \sigma_t^2(\delta) \norm{\bphi_t(i)} ^2  _{\bV_t^{-1}}\right]$. Since the variance $\Var[X_{t,i}]=\mu_{t,i}(1-\mu_{t,i})=\dot{\ell}\left(\bphi_t(i)^{\top}\btheta^*\right)$ and given the fact $\btheta^*\in \cQ_t$ meaning $\dot{\ell}\left(\bphi_t(i)^{\top}\btheta^*\right)\ge \dot{\ell} \left(\bphi_t(i)^{\top}\tilde{\btheta}_{t,i}\right)$, we have $\text{Reg}(T)\lesssim  B_v\sqrt{T}\sigma_t(\delta) \sqrt{\E\left[\sum_{t=1}^T\sum_{i \in [\tau_t]} \|\tilde{\bphi}_t(i)\|^2 _{\bL_t^{-1}}\right]}+\\ \E\left[\sum_{t\in [T]}\sum_{i\in [\tau_t]} B_1 \kappa \sigma_t^2(\delta) \norm{\bphi_t(i)} ^2  _{\bV_t^{-1}}\right]$. The proof is concluded by $\sum_{t \in [T]}\sum_{i \in \tau_t}\norm{\tilde{\bphi}_t(i)}^2_{\bL_{t}^{-1}}=\tilde{O}(\sqrt{d})$ using the elliptical potential lemma (\Cref{apdx_lem:elliptical_potential}) and $\sigma_T(\delta)=\tilde{O}({\sqrt{d}})$ by \cref{lem:var_ad_bonus}. 
\end{proof}

\section{Improving Computation Efficiency for {VA-CL\lowercase{og}UCB} Algorithm}\label{sec:eva}
In this section, we consider the CLogB-TI setting where the feature maps $\bphi_t=\bphi$ are time-invariant. 
For this setting, we are able to modify the variance-adaptive algorithm by introducing a burn-in stage that restricts the learning space around $\btheta^*$ whose reward sensitivity is bounded.
After this burn-in stage, we can remove the time-consuming nonconvex projection in the previous section, improving the computation efficiency while maintaining the tight regret bound. The proposed EVA-CLogUCB algorithm is provided in \cref{alg:EVA_CLogB}.

\begin{algorithm}[t]
	\caption{EVA-CLogUCB: \textbf{E}fficient \textbf{V}ariance-\textbf{A}daptive \textbf{C}ombinatorial \textbf{Log}isitc  \textbf{U}pper \textbf{C}onfidence \textbf{B}ound Algorithm for CLogB}\label{alg:EVA_CLogB}
			\resizebox{1.\columnwidth}{!}{
\begin{minipage}{\columnwidth}
	\begin{algorithmic}[1]
	    \State {\textbf{Input:}} Base arms $[m]$, dimension $d$, parameter space $\Theta$, time-invariant feature map $\bphi$, probability $\delta=\nicefrac{1}{T}$, offline ORACLE.
     \State \textbf{Initialize:} $T_0=\left(4L^2+16L+19\right)^2 \kappa d^2 \log^2\left( \nicefrac{4(2+T)}{\delta}\right)$, $\lambda_{0}=d \log\left( \nicefrac{4(2+T_0)}{\delta}\right)$.
     \For{$t= 1,..., T_0$} \Comment{Burn-in stage}\label{line:burn-in_start}
    \State Compute $\bV_{t}=\sum_{s=1}^{t-1}\bphi(i_s) \bphi(i_s)^{\top} + \kappa\lambda_0 \bI_d$. \label{line:eva_Vt}
     \State Play $S_t \in \cS$ such that $i_t\in S_t$ where $i_t=\argmax_{i\in [m]} \norm{\bphi(i)}_{\bV_{t}^{-1}}$ and observe $X_{t,i_t}$.\label{line:pure_explore}
     \EndFor
         \State Compute $\bV_{T_0+1}=\sum_{s=1}^{T_0}\bphi(i_s) \bphi(i_s)^{\top} + \kappa\lambda_0 \bI_d$. \label{line:compute_V_T0}
     \State Compute $\hat{\btheta}_{T_0+1}=\argmin_{\btheta\in \R^d}-\sum_{s=1}^{T_0}\left[X_{s,i_s} \log \ell\left(\btheta^{\top}\bphi(i_s) \right)+\left(1-X_{s,i_s}\right)\right. 
\left. \log \left(1-\ell\left(\btheta^{\top}\bphi(i_s) \right)\right)\right]+\nicefrac{\lambda_0}{2}\|\btheta\|_2^2$.\label{line:compute_MLE_T0}
     \State Construct nonlinearity-restricted region $\cQ=\left\{\btheta\in \R^d: \norm{\btheta-\hat{\btheta}_{T_0+1}}_{\bV_{T_0+1}}\le \left(L^2+4L+\nicefrac{19}{4}\right)\sqrt{\kappa\lambda_{0}} \right\}$. \label{line:nonlinear_region}
	   \For{$t=T_0+1, ...,T$ } \Comment{Learning stage}
    \State $\hat{\btheta}_{t}\defeq \argmax_{\btheta\in \cQ}\mathcal{L}_{t}(\btheta)$ in \cref{eq:log-likelihood} with $\lambda_t=\lbdt$. \label{line:constrained_opt}
	   \For{$i \in [m]$}
	    \State $\bar{\mu}_{t,i}=\ell\left(\hat{\btheta}_{t}^{\top}\bphi_t(i)\right)+\rho_{t,E}(\bphi_t(i),\hat{\btheta}_{t,E})$ according to \cref{eq:eva_conf_H}. \label{line:eva_ucb}
	    \EndFor
	   \State $S_t=\text{ORACLE}(\bar{\mu}_{t,1}, ..., \bar{\mu}_{t,m})$ as in \cref{eq:def_oracle}.\label{line:eva_oracle} 
	   \State Play $S_t$ and observe triggering arm set $\tau_t$ with outcomes $(X_{t,i})_{i\in \tau_{t}}$.\label{line:eva_obs}
	    	   \EndFor
		\end{algorithmic}
               		\end{minipage}}
\end{algorithm}

\subsection{Burn-In Stage}

To remove nonconvex projection that is NP-hard to solve, we introduce a burn-in stage of length $T_0$ that produces a nonlinearity-restricted region $\cQ$ in Lines~\ref{line:burn-in_start}-\ref{line:nonlinear_region}. 
Specifically, in each round $t=1, ..., T_0$, the learner selects any super arm $S_t$ that contains the arm $i_t$ that has the largest uncertainty in \cref{line:pure_explore}. 
Then at the end of $T_0$, we can compute the covariance matrix $\bV_{T_0+1}$ in \cref{line:compute_V_T0} and the MLE $\hat{\btheta}_{T_0+1}$ of the data from these $T_0$ rounds. 
Based on $\bV_{T_0+1}$ and $\hat{\btheta}_{T_0+1}$, we can construct a nonlinearity-restricted region in \cref{line:nonlinear_region} that has the following property.

\begin{lemma}[Nonlinearity-restricted region]\label{lem:Q_region}
Let $T_0=\left(4L^2+16L+19\right)^2 \kappa d^2 \log^2\left( \nicefrac{4(2+T)}{\delta}\right)$. The nonlinearity-restricted region
    $\cQ=\left\{\btheta\in \R^d: \norm{\btheta-\hat{\btheta}_{T_0+1}}_{\bV_{T_0+1}}^2\le \left(L^2+4L+\nicefrac{19}{4}\right)\sqrt{\kappa d \log\left( \nicefrac{4(2+T_0)}{\delta}\right)}\right\}$
is a confidence region such that with probability at least $1-\delta$, we have $\btheta^* \in \cQ$ and
the diameter $\text{diam}(\cQ)\defeq\max_{\btheta_1, \btheta_2 \in \cQ, i\in[m]}\abs{\bphi(i)^{\top}(\btheta_1-\btheta_2)}$ is bounded, i.e., 
    $\text{diam}(\cQ) \le 1$.
\end{lemma}

At a high level, the above lemma ensures that $\cQ$ contains the true parameter $\btheta^*$. Additionally, for any two feasible parameters $\btheta_1,\btheta_2 \in \cQ$, their corresponding reward sensitivity for any arm $i$ is bounded by each other, i.e., $\dot{\ell}\left(\btheta_1^{\top}\bphi(i)\right)\le \exp(1) \dot{\ell}\left(\btheta_2^{\top}\bphi(i)\right)$. Therefore, we can restrict our learning space to $\cQ$ where any $\btheta\in \cQ$ has bounded first-order derivative $e^{-1}\dot{\ell}({\btheta^*}^\top \bphi_t(i))\le \dot{\ell}(\btheta^\top \bphi_t(i)) \le  e\dot{\ell}({\btheta^*}^\top \bphi_t(i))$. That is, any Hessian matrix $\bH_t(\btheta)\succeq e^{-1}\bH_t(\btheta^*)$ for $\btheta \in \cQ$, which guarantees the shrinkage of the exploration bonus without the nonconvex projection in VA-CLogUCB algorithm.

\subsection{Learning Stage}\label{sec:eva_conf_bonus}
After the burn-in stage, we can replace the projected MLE $\hat{\btheta}_{t,H}$ that involves a nonconvex projection with the constrained MLE $\hat{\btheta}_{t,E}$ that only involves the \textit{convex} projection in~\cref{line:constrained_opt}. 
For this constrained MLE, we have the following variance-adaptive confidence region. The proof of \cref{lem:eva_conf} is in \cref{apdx_sec:proof_eva_conf}.

\begin{lemma}[Variance-adaptive confidence region after the burn-in stage]\label{lem:eva_conf}
    Let $\delta \in (0,1]$ and set the confidence radius $\nu_t(\delta):=3(L+\nicefrac{3}{2})\sqrt{\lbdt}$. The following region
    \begin{equation}\textstyle
\mathcal{D}_t(\delta):=\left\{\btheta \in \Theta: \left\|\btheta-\hat{\btheta}_{t,E}\right\|_{\mathbf{\bH}_{\mathrm{t}}(\btheta)} \leq \nu_t(\delta)\right\}, 
\end{equation}
 is an anytime confidence region after the burn-in stage for $\btheta^*$ with probability at least $1-\delta$, i.e., 
    \begin{align}
        \Pr\left(\forall t\ge T_0+1, \btheta^* \in \mathcal{D}_t(\delta)\right)\ge 1-\delta.
    \end{align}
\end{lemma}

Based on the above confidence region, we can now construct our variance-adaptive exploration bonus $\rho_{t,E}(i)$ as follows. The proof of \cref{lem:eva_bonus} is in \cref{apdx_sec:proof_eva_bonus}.
\begin{lemma}[Variance-adaptive exploration bonus after the burn-in stage]\label{lem:eva_bonus}
    Let $\mathcal{D}_t(\delta)$ be the confidence region with the confidence radius $\nu_t(\delta)$ as defined in \cref{lem:eva_conf}. 
    Let the exploration bonus be
    \begin{align}\label{eq:eva_conf_H}
        \rho_{t,E}(i)\defeq \sqrt{e}\dot{\ell}(\hat{\btheta}_{t,E}^{\top}\bphi(i))\nu_t(\delta)\norm{\bphi(i)}_{\bH_t^{-1}(\hat{\btheta}_{t,E})}+\frac{1}{8} \kappa \nu_t^2 (\delta) \norm{\bphi(i)}^2_{\bV_t^{-1}}.
    \end{align} 
    Under the event $\left\{\forall t\ge T_0+1, \btheta^* \in \mathcal{D}_t(\delta)\right\}$, it holds that, for any $i\in[m], t\ge T_0+1$,
    \begin{align}
        \ell\left({\btheta^*}^{\top}\bphi(i)\right)\le \ell\left(\hat{\btheta}_{t,E}^{\top}\bphi(i)\right)+\rho_{t,E}(i)\le \ell\left({\btheta^*}^{\top}\bphi(i)\right)+2\rho_{t,E}(i).
    \end{align}
\end{lemma}
Next we can use the variance-agnostic upper confidence bound $\bar{\mu}_{t,i}\defeq \ell\left(\hat{\btheta}_{t,E}^{\top}\bphi(i)\right)+\rho_{t,E}(i)$ as our optimistic estimation of the true mean $\mu_{t,i}$ in \cref{line:eva_ucb}. Then in \cref{line:eva_oracle}, the learner selects action $S_t$ via the offline oracle with $\bar{\bmu}_t$ as input.
By playing $S_t$, the base arms in $\tau_t$ are triggered, and the learner receives observation set $(X_{t,i})_{i\in \tau_{t}}$ as feedback.

\subsection{Regret Bound with Improved Computational Efficiency}
We now give the regret bound under for CLogB with a time-invariant feature map (CLogB-TI).

\begin{theorem}\label{thm:eva_thm1}
     For a CLogB-TI instance that satisfies monotonicity (Condition~\ref{cond:mono}) and 1-norm TPM smoothness (Condition~\ref{cond:TPM}) with coefficient $B_1$, EVA-CLogUCB (\Cref{alg:EVA_CLogB}) with an $\alpha$-approximation oracle achieves an $\alpha$-approximate regret bounded by 
     $ O\left(B_1 d\sqrt{KT}\log (KT)+ B_1 \kappa K d^2\log ^2(T)\right).$
     For a CLogB-TI instance that satisfies monotonicity (Condition~\ref{cond:mono}) and the TPVM smoothness (Condition~\ref{cond:TPVMm}) with coefficient $(B_v,B_1, \lambda)$, if $\lambda \ge 1$, then EVA-CLogUCB (\Cref{alg:EVA_CLogB}) with an $\alpha$-approximation oracle achieves an $\alpha$-approximate regret bounded by $ O\left(B_v d\sqrt{T}\log (KT)+ B_1 \kappa K d^2\log ^2(T)\right).$
\end{theorem}

\textbf{Discussion.} 
Looking at \cref{thm:eva_thm1}, the leading regrets are $\tilde{O}(B_1 d\sqrt{KT})$ and $\tilde{O}(B_v d\sqrt{T})$, matching the leading regret of VA-CLogUCB. For the lower order terms, EVA-CLogUCB has an additional $\tilde{O}(K)$ factor compared with VA-CLogUCB.
As for the per-round computation cost, \cref{alg:EVA_CLogB} contains two stages. For the burn-in stage, the main computational cost is computing the MLE in \cref{line:compute_MLE_T0}, which takes $O(dK^2T^2)$ as discussed in \cref{rmk:MLE_cost}. But since it is computed only once in $T_0$, the per-round cost is $O(dK^2 T)$. For the learning stage, it contains three parts: (1) computing the MLE over the ellipsoidal region $\cQ$ still takes $\tilde{O}(dK^2T^2)$ since in each projected gradient descent iteration (and there are $\tilde{O}(KT)$ iterations), it takes $O(dKT)$ to compute the gradient and $\tilde{O}(d^3)$ to project back the updated variable back to $\cQ$ by solving a one-dimensional convex problem [Lemma 13 in \cite{faury2022jointly}], (2) computing the Hessian matrix $\bH_t(\hat{\btheta}_{t,H})$ takes $O(d^2KT)$ and computing UCB for all $m$ base arms takes $O(d^3 + d^2m)$, (3) the $\alpha$-approximation oracle takes $T_{\alpha}$. So the overall per-round cost is $\tilde{O}(dK^2T^2 +T_{\alpha})$, removing the additional $T_{\text{nc}}$ compared to VA-CLogUCB algorithm (\cref{alg:VA_CLogB}).

\begin{proof}[Proof idea]
    The proof consists of two parts: the regret caused by the burn-in stage $T_0$ and the learning stage after $T_0$. For the burn-in stage, the regret is bounded by $\tilde{O}(B_1\kappa d^2 K)$ since $T_0=\tilde{O}(\kappa d^2)$ and the maximum regret incurred in each round is $\alpha r(S_t^*;\bmu_t)\le B_1 K$. For the learning stage after the burn-in stage, we can restrict our learning space to $\cQ$ in \cref{eq:Q_region} where any $\btheta\in \cQ$ has bounded first-order derivative $e^{-1}\dot{\ell}({\btheta^*}^\top \bphi_t(i))\le \dot{\ell}(\btheta^\top \bphi_t(i)) \le  e\dot{\ell}({\btheta^*}^\top \bphi_t(i))$. Since $\hat{\btheta}_{t,E}\in \cQ$, we have both the lower and upper bound $e^{-1}\bH_t({\btheta^*}) \le \bH_t(\hat{\btheta}_{t,E})\le e\bH_t({\btheta^*})$, which can replace the critical lower bound ${\mathbf{H}_t^{-1}(\hat{\btheta}_{t,H})} \gtrsim {\mathbf{H}_t^{-1}(\tilde{\btheta}_{t,i})}$ by ${\mathbf{H}_t^{-1}(\hat{\btheta}_{t,H})} \gtrsim {\mathbf{H}_t^{-1}({\btheta^*})}$, effectively maintaining the shrinkage of the exploration bonus without using the nonconvex projection. Applying this trick to the rest of the proofs of VA-CLogUCB yields the desired regret bounds. The detailed proof is in \cref{apdx_sec:proof_thm_eva}.
\end{proof}

\section{Experiments}
\label{sec:simulations}

In this section, we conduct experiments to evaluate the performance of our proposed algorithms on both synthetic and real-world datasets. We compare our algorithms CLogUCB, VA-CLogUCB, and EVA-CLogUCB with four baselines: VAC$^2$-UCB \cite{liu2023contextual}, the state-of-the-art variance-adaptive linear contextual algorithm for C$^2$AMB; C$^3$UCB~\cite{li2016contextual}, the state-of-the-art contextual combinatorial bandit algorithm that is not variance-adaptive; $\epsilon$-greedy, which selects a random action with fixed probability $\epsilon$ for exploration (set at $\epsilon=0.2$ for all tests), and otherwise chooses the empirically optimal action selection but without the exploratory component; and CUCB~\cite{chen2016combinatorial}, the state-of-the-art non-parametric algorithm for CMAB. All experiments are averaged over five trials and performed on a desktop with an Apple M3 Pro processor and 18 GB of RAM. The code is accessible at the following link: https://github.com/xiangxdai/Combinatorial-Logistic-Bandit.

\begin{figure*}[t]
    \centering
    \begin{subfigure}[t]{0.245\textwidth}
    \includegraphics[width=0.99\textwidth]{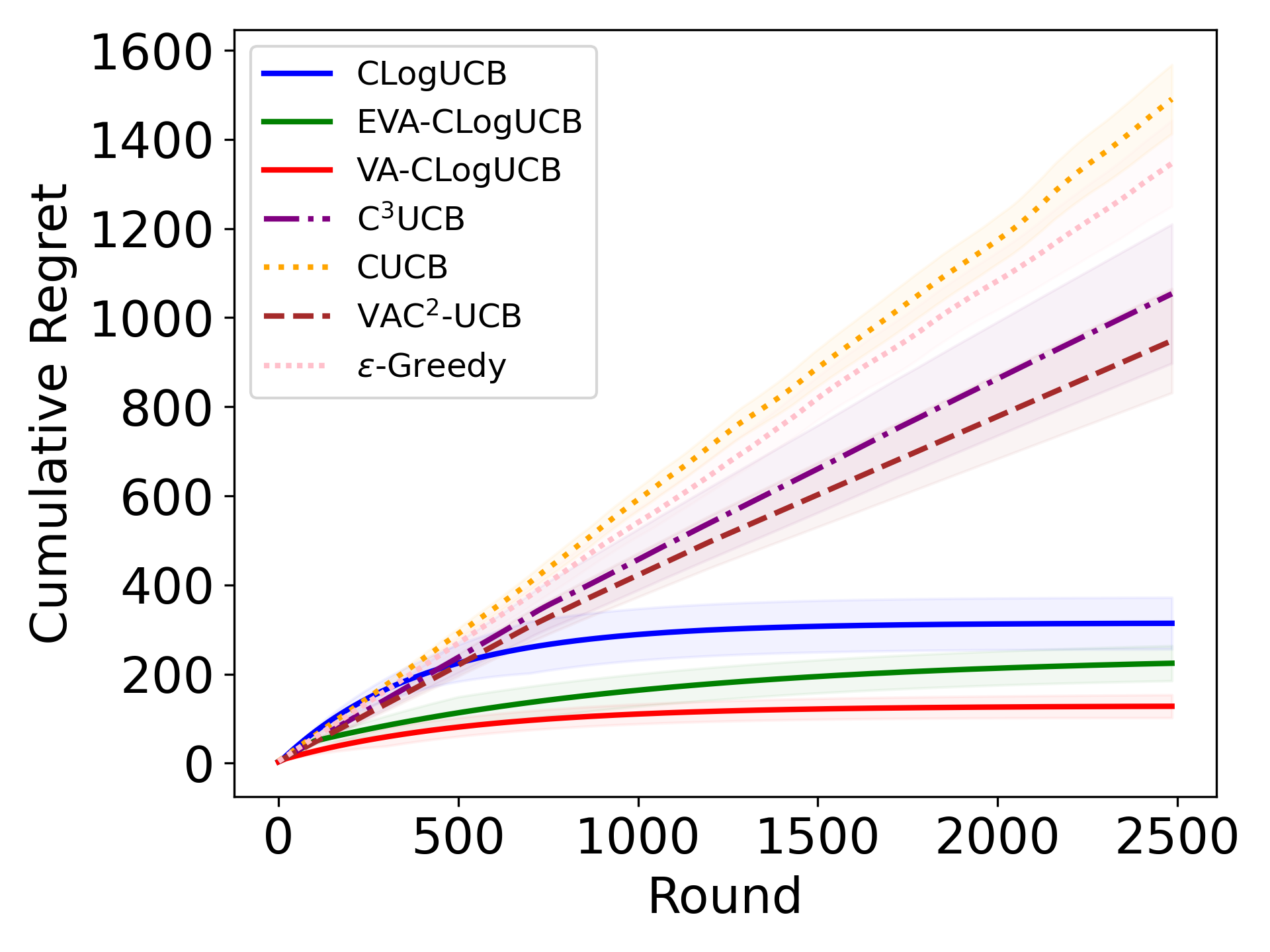}
    \caption{Cumulative Regret on Cascading Bandits}
    \label{fig:regret_cascading1}
    \end{subfigure}
    \begin{subfigure}[t]{0.245\textwidth}
    \includegraphics[width=0.99\textwidth]{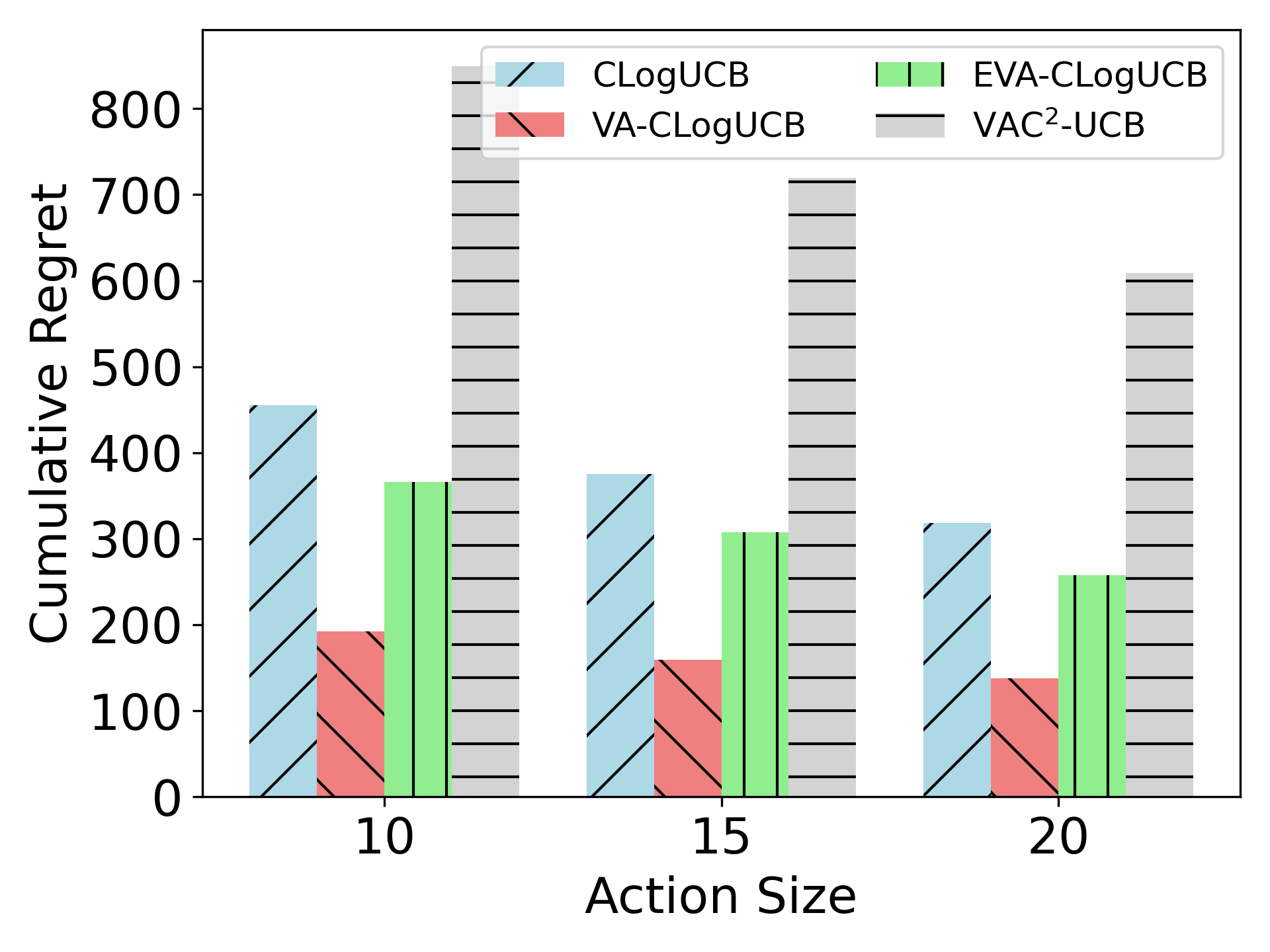}
    \caption{Regret under Varying  $K$ with $T=2000$}
    \label{fig:regret_k}
    \end{subfigure}
    \begin{subfigure}[t]{0.245\textwidth}
    \includegraphics[width=0.99\textwidth]{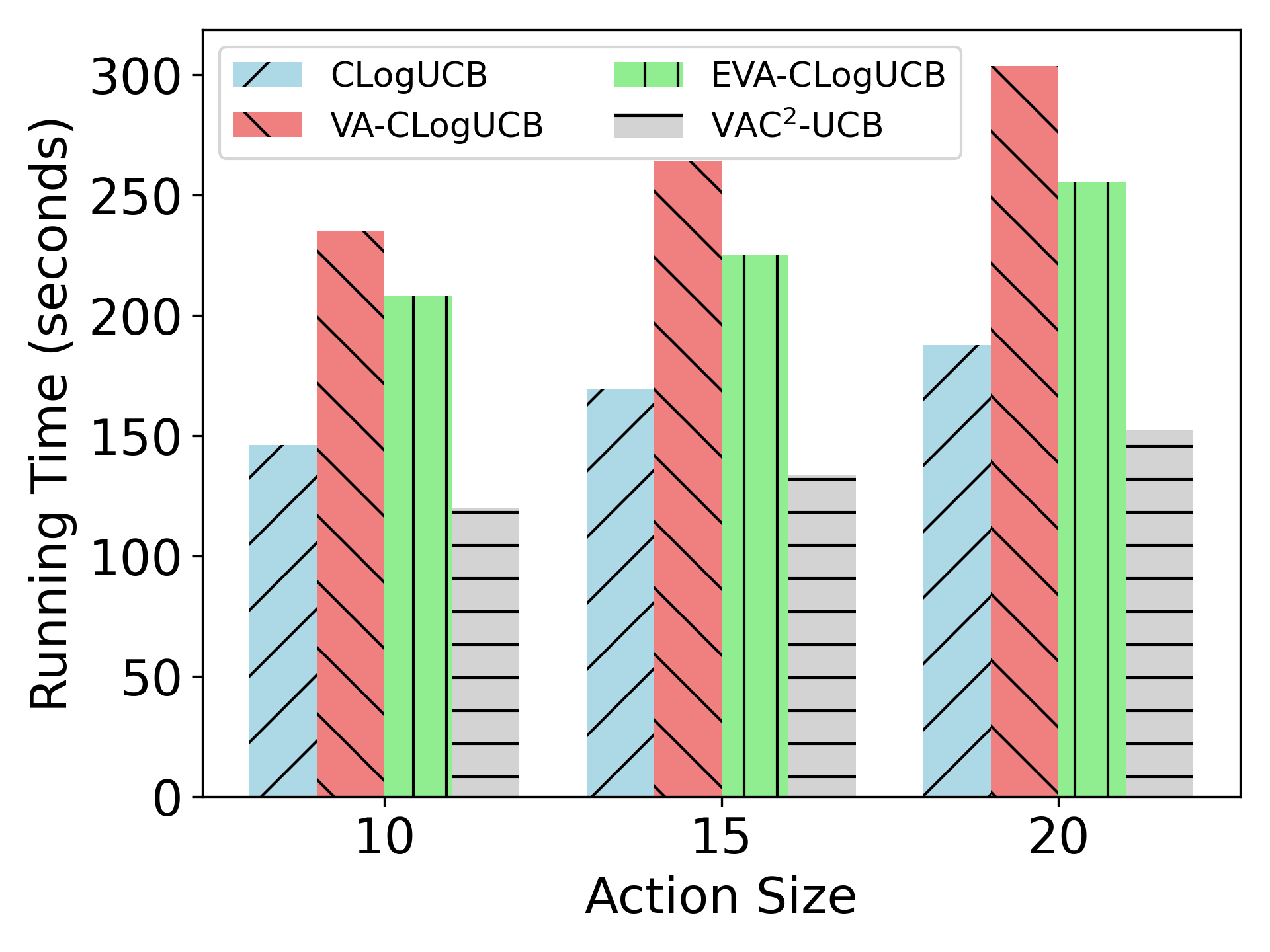}
    \caption{Running Time under Varying $K$ with $T=2000$}
    \label{fig:running_time1}
    \end{subfigure}
    \begin{subfigure}[t]{0.245\textwidth}
    \includegraphics[width=0.99\textwidth]{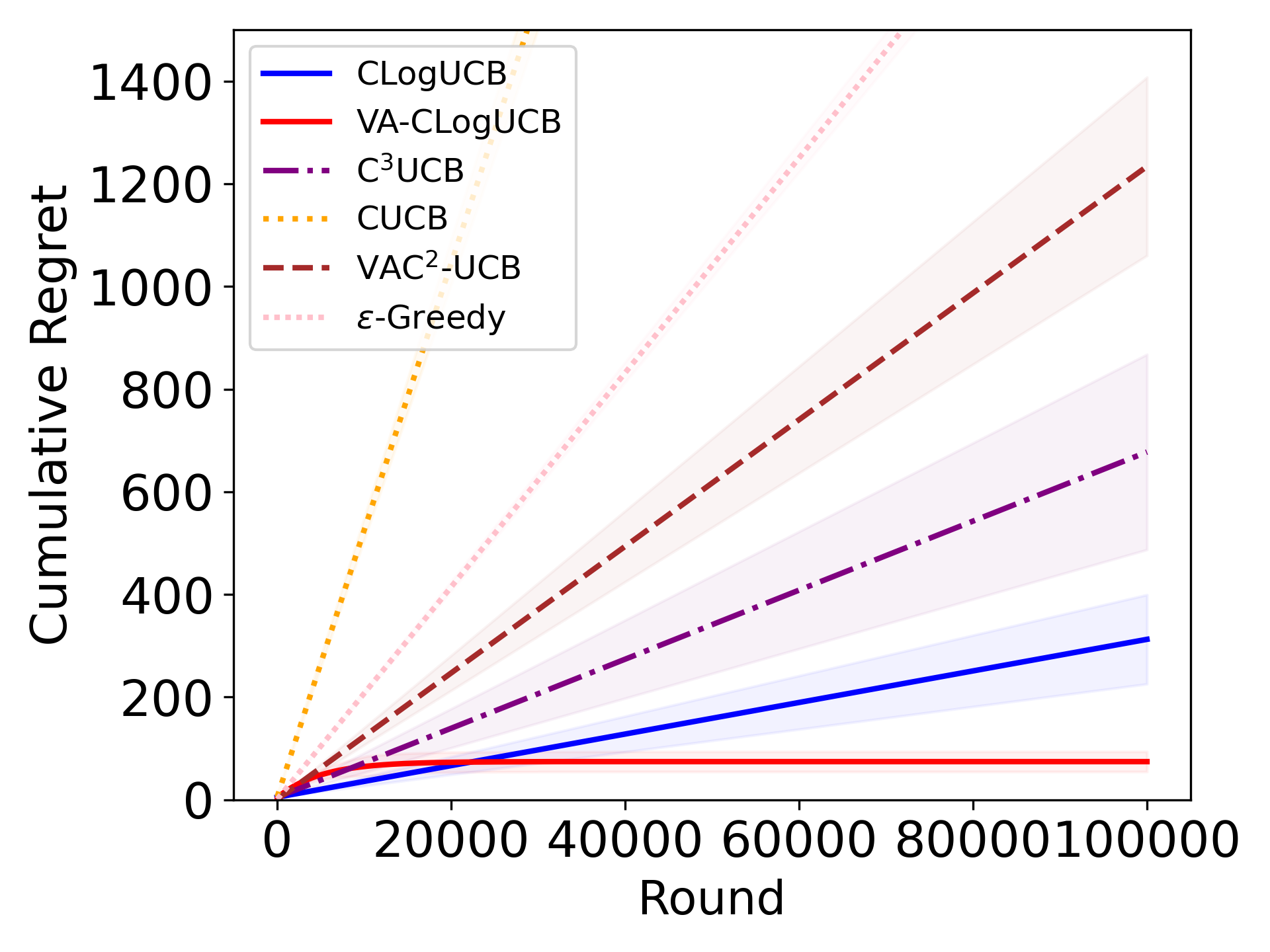}
    \caption{Cumulative Regret on PMC Bandits}
    \label{fig:regret_content}
    \end{subfigure}
  \caption{
(a)-(c) show the results of cascading bandits on the synthetic dataset; (d) shows the results of probabilistic maximum coverage bandits on the real-world dataset (All our algorithms represented by solid lines). }
  \label{fig:reg1}
  \vspace{-0.1in}
\end{figure*}

\subsection{Cascading Bandits on Synthetic Dataset}\label{sec:Cascading}
\rev{We conduct experiments on cascading bandits for the online learning-to-rank application in \cref{sec:app_ranking}, aiming to select \(K=15\) items from a set of \(m=500\) to maximize the reward. An additional experiment with a different instance scale is provided in Appendix \ref{app:extended}.} We set the dimension \(d=10\), and the feature vectors \(\phi(i) \in \mathbb{R}^d\) are time-invariant, with each dimension drawn from the uniform distribution \(U(-1, 1)\). The unknown parameter \(\btheta^*\) is also drawn from the uniform distribution \(U(-1, 1)\). For each item \(i \in [m]\), the purchase probability is \(\mu_{i} = \ell(\btheta^*, \bphi(i))\), where \(\ell\) is the sigmoid function. In each round \(t\), the algorithm selects a ranked list \(S_t = (a_{t,1}, \ldots, a_{t,K}) \subseteq [m]\), and the outcome \(X_{t,i}\) for \(i \in S_t\) is generated from a Bernoulli distribution with mean \(\mu_{i}\). The reward at round \(t\) is 1 if there exists an item \(a_{t,k}\) with index \(k\) whose outcome is \(X_{t,a_{t,k}}=1\), and the learner observes the outcomes for the first \(k\) items of \(S_t\); otherwise, the reward is 0, and the learner observes all \(K\) item outcomes to be \(X_{t,i}=0\) for \(i \in S_t\).

\rev{Fig.~\ref{fig:regret_cascading1} shows the cumulative regrets of the algorithms over \(T=2000\) rounds. All three algorithms proposed in this work outperform the baseline algorithms by at least 78.94\%, 76.68\%, 70.21\%, and 66.89\% less regret than the CUCB, \(\epsilon\)-greedy, C\(^3\)UCB, and VAC\(^2\)-UCB algorithms, respectively, demonstrating our superior performance in a nonlinear environment.} Unlike our algorithms, CUCB and \(\epsilon\)-greedy independently learn each arm's mean, while VAC\(^2\)-UCB and C\(^3\)UCB leverage the correlation between arms but through a mismatched linear model, resulting in worse performance.
\rev{Fig.~\ref{fig:regret_k} and Fig.~\ref{fig:running_time1} compare the regret and total running time of our CLogUCB, VA-CLogUCB, and EVA-CLogUCB algorithms against the lowest-regret benchmark, VAC$^2$-UCB, across different action sizes $K$ under an enlarged setting of $m=600$ to evaluate robustness (a comprehensive comparison with all benchmarks and an ablation study on dimension $d$ are provided in Appendix \ref{app:extended}). Note that with the change of $K$, the optimal reward adjusts accordingly, which explains why the regret of smaller values of $K$ tends to be larger than that for larger $K$. VA-CLogUCB achieves the lowest regret compared to CLogUCB and EVA-CLogUCB. On the other hand, CLogUCB achieves the lowest running time. EVA-CLogUCB, positioned in the middle, strikes a good balance between the regret and the running time. Compared to VAC$^2$-UCB, CLogUCB increases the average running time by 32 seconds but reduces significantly regret by 47.32\%.}

\subsection{ Probabilistic Maximum Coverage Bandits on Real-World Dataset} 

We conduct experiments on probabilistic maximum coverage (PMC) bandits. As described in \cref{sec:app_PMC}, PMC bandits focus on the online content delivery application in CDN (see \cref{fig:cdn}), which dynamically adapts content delivery based on the capabilities and preferences of the users' network \cite{ye2023data, yang2018content,dai2025variance}. Our experiment utilizes data from 10 Microsoft Azure CDN point of presence locations in North America\footnote{https://docs.microsoft.com/en-us/azure/cdn/cdn-pop-locations}, which reflect the geographic distribution of users. We use Quality of Service (QoS) scores as features \cite{dai2024quantifying}. Specifically, for each $(u,v) \in \cE$, the feature is a $d_1=4$ dimensional score vector $\bphi_{t}((u,v))=[L, J, D, B] \in \R^{d_1}$, where $L$, $J$, $D$, and $B$ represent scores determined by packet loss, jitter, packet delay, and bandwidth for each link $(u,v)$ in the network between server $u$ and user $v$.  The network delay and jitter are collected from real network testing results between user $u$'s location and server $v$'s location\footnote{https://wondernetwork.com/pings}. The bandwidth and package loss are sampled from uniform distribution $U(0.9\text{Mbps},3\text{Mbps})$ and $U(0\%,1\%)$ \cite{dong2018pcc,prasad2003bandwidth}. To compute the score $\bphi_t(u,v)$, we normalize delay and jitter by $\max((200-\text{delay})/200,0)$ and $\max((10-\text{jitter})/10,0)$. Bandwidth and packet loss are linearly normalized to the range $(0,1)$ by dividing by their maximum values.  A weight vector $\btheta^*_1 = [0.149,0.151,0.111,0.598]$ is utilized to assign relative importance to these QoS metrics according to \cite{kim2012qoe}. The overall service quality is evaluated using the sigmoid function $\ell({\btheta^*}^{\top}_1 \bphi_{t}((u,v)))$, which aggregates the four QoS scores. As for the feedback, the content owner can observe whether the content is successfully delivered from the selected servers, tracked via $X_{t, (u,v)}$ for each server $u$ in the selected set $S_t$ and user $v$ in the set $V$. The probability of this feedback being observed is 1 for $u \in S_t$ and 0 otherwise.

After establishing the network capabilities, we proceed to model user preferences for the user-targeted CDN.
To analyze user preferences for movie content, following \cite{dai2024online}, we employ the MovieLens-1M dataset, which comprises 1 million ratings by 6,040 users for 3,952 movies.\footnote{https://grouplens.org/datasets/movielens/1m/} Unlike previous studies that applied Singular Value Decomposition (SVD) to explicit feedback data such as user-item ratings \cite{dai2024conversational,li2019improved}, we adopt the logistic SVD methodology in \cite{johnson2014logistic} to process implicit feedback, i.e., user clicks within the MovieLens dataset.
We begin by decomposing the user-item click matrix to derive a 10-dimensional user feature mapping, denoted as $\bphi(v)$. The dataset is split into training and test sets with a 70:30 ratio. A logistic regression model using the logistic function $\ell_2$ achieves 98\% accuracy on the test set. This model estimates the preference distribution of a user \(v\) for a randomly selected movie as $\mu_v = \ell_2({\btheta_2^*}^{\top} \bphi(v))$. In each round $t$, users randomly arrive to be served, with both content (movie) and 10 users across 10 geographic locations sampled uniformly at random in our experiments.  Additionally, Gaussian noise $\mathcal{N}(0, 1)$ is added to simulate real-world variability in user behavior.
In terms of feedback, if a user $v$ successfully receives the content, the content owner can verify consumption via the observed value of $X_{t,v}$ if $\exists v \text{ s.t. } X_{u,v} = 1$, which depends on other random outcomes (i.e., whether the content is successfully received), with the observation probability calculated as $p_v^{\bmu_t,S_t}=1-\prod_{u \in S_t}(1-\mu_{t, (u,v)})$.

Fig. \ref{fig:regret_content} shows the cumulative regrets of different algorithms for $T=100000$ rounds.
Due to the time-varying nature of the feature map under the application of CDN, we have removed EVA-CLogUCB (Algorithm \ref{alg:EVA_CLogB}). Additionally, as analyzed in Table \ref{tab:clogb_res}, the projection operation in VA-CLogUCB incurs a high computational cost (as shown by the running time in \cref{fig:running_time}). Therefore, we eliminate the projection operation (i.e., lines \ref{line:proj_Q} and \ref{line:update} in Algorithm \ref{alg:VA_CLogB}) to investigate whether VA-CLogUCB can still perform well on real-world data even without the time-consuming projection operation. As shown in Fig. \ref{fig:regret_content}, CLogUCB significantly outperforms baseline algorithms, achieving at least 93.99\%, 74.68\%, 53.88\%, and 85.04\% lower regret than CUCB, VAC$^2$-UCB, C$^3$UCB, and $\epsilon$-greedy algorithms, respectively. Notably, even without the projection operation, VA-CLogUCB surpasses CLogUCB with 76.25\% less regret under the CDN application.

\section{Conclusion}
In this paper, we propose a new combinatorial logistic bandit problem framework that can cover various applications satisfying different smoothness conditions. We design a variance-agnostic CLogUCB algorithm with $\tilde{O}(d\sqrt{\kappa K T})$ regret under the TPM condition. Then we devise the variance-adaptive VA-CLogUCB algorithm with improved $\tilde{O}(d\sqrt{T})$ regret under the stronger TPVM condition. Finally, we improve the computational efficiency of VA-CLogUCB while maintaining the regret results. An interesting future direction is to generalize the EVA-CLogUCB to handle time-varying feature maps while maintaining its tight regret bound and computational efficiency. 

\section*{Acknowledgement}
We thank our shepherd Sean Sinclair and the anonymous SIGMETRICS reviewers for their valuable
insight and feedback.

The work of Xutong Liu was partially supported by a fellowship award from the Research Grants Council of the Hong Kong Special Administrative Region, China (CUHK PDFS2324-4S04).
The work of John C.S. Lui was supported in part by the RGC GRF-14202923. The work is supported by the National Science Foundation under awards CNS-2102963, CAREER-2045641, CNS-2106299, CPS-2136199, and CNS-2325956. Xuchuang Wang is the corresponding author.

\newpage
\bibliographystyle{ACM-Reference-Format}
\bibliography{bibliography}

\newpage
\appendix
\clearpage
\onecolumn

\section*{Appendix}

The appendix is organized as follows. 
\begin{itemize}
    \item In \cref{apdx_sec:related_work}, we discuss the extended related works on
    \begin{itemize}
        \item generalized linear bandits,
        \item logistic bandits,
        \item combinatorial MAB with semi-bandit feedback
    \end{itemize}
    \item In \cref{apdx_sec:app}, we provide two additional applications that fit into CLogB framework.
    \begin{itemize}
        \item Dynamic channel allocation
        \item Reliable packet routing
    \end{itemize}
    \item In \cref{apdx_sec:notation}, we recall/define the notations useful for the proofs and introduce some important inequalities.
    \item In \cref{apdx_sec:confidence_region}, we prove that
    \begin{itemize}
        \item the concentration inequality of the MLE $\hat{\btheta}_t$,
        \item $\cB_t(\delta), \cC_t(\delta), \cD_t(\delta)$ are confidence regions for $\btheta^*$,
        \item the exploration bonuses $\rho_{t,V}, \rho_{t,H}, \rho_{t,E}$ ensure the optimism of the UCB values. 
    \end{itemize}
    \item In \cref{apdx_sec:regret_upper_bounds}, we prove the regret upper bounds for 
    \begin{itemize}
        \item CLogUCB algorithm under the 1-norm TPM smoothness condition,
        \item VA-CLogUCB under both 1-norm TPM and TPVM smoothness conditions,
        \item EVA-CLogUCB under both 1-norm TPM and TPVM smoothness conditions.
    \end{itemize}
    \item In \cref{apdx_sec:aux}, we prove auxiliary lemmas that serve as important ingredients for our analysis.
\item \rev{In \cref{app:extended}, we provide the additional experimental results.}
\end{itemize}

\section{Extended Related Works}\label{apdx_sec:related_work}


\textbf{Generalized Linear Bandits.}
Generalized linear bandits (GLB) were first introduced by \citet{filippi2010parametric}. This model generalizes the well-studied linear bandit framework by replacing the linear link function with general exponential family link functions under the generalized linear model. Notable instances include logistic bandits with the sigmoid link function to be introduced shortly after and Poisson bandits \cite{gisselbrecht2015policies,mutny2021no,lee2024unified} with the exponential link function.
GLB has advanced significantly, with the development of various algorithms such as Newton step-based algorithms \cite{zhang2016online,li2017provably,jun2017conversion} and Thompson sampling-based algorithms \cite{abeille2017thompson,dong2019thompson,kveton2020glm,ding2021efficient}. Additionally, GLB has been explored in different contexts, including the Bayesian setting \cite{dong2019thompson}, the non-stationary setting \cite{russac2021glm}, and pure exploration regimes \cite{kazerouni2021glm}.
As mentioned in \cref{sec:connect}, our results under the logistic model can potentially be extended to generalized linear models that satisfy self-concordant and variance properties, which presents an interesting direction for future research.

\textbf{Logistic Bandits.}
As a special case of generalized linear bandits where the link function is a sigmoid function, logistic bandits are crucial for studying the effects of nonlinear parameterization on the exploration-exploitation trade-off in parametric bandits. Logistic bandits were first formally proposed by \citet{faury2020improved}, who introduced two UCB-based algorithms: Logistic-UCB-1 and Logistic-UCB-2. These algorithms achieve regret bounds of \(\tilde{O}(d\sqrt{\kappa T})\) and \(\tilde{O}(d\sqrt{T} + \kappa d^2)\), respectively, based on the self-concordant analysis of logistic regression \cite{bach2010self}.
Following this, research began to focus on optimism in the face of uncertainty (OFU)-based algorithms. \citet{abeille2021instance} proposed OFULog, which jointly searches the action space and the confidence region to yield an optimistic action-parameter pair in each round, achieving a regret bound of \(\tilde{O}(d\sqrt{T/\kappa})\) and matching the lower bound proposed in the same paper. \citet{faury2022jointly} improved the computational complexity of OFULog by introducing OFU-ECOLog, an online Newton step-based algorithm with a per-round computation cost of \(\tilde{O}(d^2 m)\) while maintaining the same \(\tilde{O}(d\sqrt{T/\kappa})\) regret bound. Recently, \citet{lee2024improved} proposed the OFULog+ algorithm, which leverages the regret-to-confidence set conversion method to improve the dependency on the support \(S\).
Our CLogB framework is the first to generalize these works to handle scenarios where multiple arms can be triggered in each round. This setting presents unique challenges that make the direct application of the previous methods infeasible. For example, OFU-based algorithms become intractable due to the combinatorially large action space, and it is difficult to ensure the effectiveness of UCB-based algorithms in the presence of nonlinear reward functions while avoiding expensive nonconvex projections. Addressing these issues is the main focus and the contribution of our current work.

\textbf{Combinatorial Multi-Armed Bandits with semi-bandit feedback.} 
There is a vast body of literature on stochastic combinatorial multi-armed bandits (CMAB) \cite{gai2012combinatorial,cesa2012combinatorial,gyorgy2007line,kveton2015tight,combes2015combinatorial,chen2016combinatorial,wang2017improving,merlis2019batch,saha2019combinatorial,liu2022batch,liu2024combinatorial}. \citet{gai2012combinatorial} was the first to consider stochastic CMAB with semi-bandit feedback. Within this domain, two prominent lines of work are closely related to our study: CMAB with probabilistically triggered arms (CMAB-T) and contextual CMAB (\ccmab).

For CMAB-T, \citet{chen2016combinatorial} introduced the concept of arm triggering processes to address cascading bandits and influence maximization applications. They proposed the CUCB algorithm, which achieves an \(O(B_1\sqrt{mKT\log T}/p_{\min})\) regret bound under the 1-norm TPM smoothness condition with coefficient \(B_1\). Subsequently, \citet{wang2017improving} proposed a stronger 1-norm triggering probability modulated (TPM) \(B_1\) smoothness condition and employed triggering group analysis to eliminate the \(1/p_{\min}\) factor from the previous regret bound. More recently, \citet{liu2022batch} leveraged the variance-adaptive principle to propose the BCUCB-T algorithm, which further reduces the regret's dependency on action-size from \(O(K)\) to \(O(\log K)\) under the new variance and triggering probability modulated (TPVM) condition. Although our smoothness conditions are inspired by these works, they focus on non-parametric settings, and their algorithms do not yield meaningful results for our CLogB setting.

For \ccmab, most studies focus on linear parametric models. \citet{qin2014contextual} first proposed the C$^2$UCB algorithm, which considers reward functions under a 2-norm smoothness condition with coefficient \(B_2\). \citet{takemura2021near} later replaced the 2-norm smoothness condition with a 1-norm smoothness condition, achieving a \(\tilde{O}(B_1d\sqrt{KT})\) regret bound. Recently, \citet{liu2023contextual} studied the \ccmab-T setting, which combines linear \ccmab~with CMAB-T, and proposed the variance-adaptive VAC$^2$-UCB algorithm, achieving a \(\tilde{O}(B_v d\sqrt{T})\) regret bound under the TPVM smoothness condition with coefficient \(B_v\). As noted in \cref{sec:intro}, the linear parametric model falls short in handling binary outcomes and nonlinear environments. Our CLogB framework takes the first step toward addressing large-scale and nonlinear environments with binary outcomes.
Other \ccmab~studies have considered nonlinear parametric models, such as neural network models \cite{hwang2023combinatorial}, Lipschitz continuous models \cite{chen2018contextual,nika2020contextual}, and multinomial logit models \cite{choi2024cascading}. The first two models are very different from ours and are not directly comparable. \citet{choi2024cascading} studies the multinomial logit model~\cite{agrawal2019mnl,oh2021multinomial,lee2024nearly} that generalizes our binary outcomes to \(M+1\) outcomes, reducing to a logistic model when \(M=1\), but its reward function is specific to cascading rewards and employs OFU-based algorithms, which are generally intractable for \ccmab. Our work, in contrast, focuses on general reward functions with more efficient UCB-based algorithms.

\section{Additional Applications}\label{apdx_sec:app}

\subsection{Matching Bandits: Dynamic Channel Allocation in Wireless Networks}\label{sec:app_matching}

The maximum weighted matching problem in bipartite graphs is a fundamental problem in graph theory aiming to identify a matching where the sum of the edge weights is maximized. This problem has significant applications in various fields, such as optimizing channel allocation in wireless networks by considering the probability of successful communication \cite{gai2010learning} and improving task-worker matching in crowdsourcing based on the likelihood of successful task completion \cite{tong2017spatial}.


The matching bandit problem represents the online learning variant of the maximum weighted matching problem. In the context of channel allocation, consider a bipartite graph \( G=(U,V,E) \), where \( U \) represents users, \( V \) represents channels, and each edge \( (u,v) \in E \) (i.e., a base arm) models the data transmission process if user \( u \) is assigned to channel \( v \). In each round \( t \), the learner first obtains the feature map \( \bphi_t \), which translates factors such as user location, user mobility, channel bandwidth, and channel signal-to-noise ratio (SNR) into a feature vector. Specifically, this feature map assigns each possible pair \( (u,v) \in E \) to a feature vector \( \bphi_t(u,v) \in \mathbb{R}^d \).
Based on \( \bphi_t \), the learner assigns each user to a channel, forming a matching \( S_t \subseteq E \) with \( |S_t| = |U| \) (i.e., a super arm) that ensures no collisions (i.e., no two users share the same channel). Each user \( u \) then uses the allocated channel \( v \) to transmit data, with a probability \( \mu_{t,u,v} = \ell({\btheta^*}^{\top} \bphi_t(u,v)) \) that the transmission is successful for some downstream tasks (e.g., environmental sensing and federated learning). By "success," we mean that the transmission process is of high quality, meeting the requirements of downstream tasks. Here, \( \btheta^* \in \mathbb{R}^d \) is the unknown parameter (e.g., importance weights for different factors), and the sigmoid function \( \ell \) models the nonlinear relationship between the success probability and these factors (e.g., channel SNR and channel bandwidth).
For example, the success probability exhibits an S-shape, rapidly increasing around mid-SNR values while changing slowly under poor (low-SNR) and excellent (high-SNR) channel conditions.
The reward function is the total number of successful transmissions, \( r(S_t; \bmu_t) = \sum_{(u,v) \in S_t} \mu_{t,u,v} \), which is a linear function. For feedback, the learner observes the success or failure of each arm \( (u,v) \in S_t \) (i.e., semi-bandit feedback) to learn \( \btheta^* \) and improve future actions. For this application, it fits into the CLogB framework, satisfying the 1-norm TPM smoothness condition with \( B_1=1 \). For the offline oracle, we can use the Hungarian algorithm \cite{kuhn1955hungarian} since the underlying problem is a maximum weighted matching problem, which yields the optimal solution ($\alpha=1$) with $T_{\alpha}=O(|U|^2|V|)$ time complexity.

\subsection{Conjunctive Cascading Bandit: Reliable Packet Routing in Distributed Networks}\label{sec:app_routing}

Routing is a crucial process in networking that involves determining the optimal paths for data packets to travel from their source to their destination across a network. Broadly speaking, routing is vital in many types of networks, including local area networks (LANs), wide area networks (WANs), and the internet. It plays a particularly significant role in distributed networks, where multiple interconnected nodes work together to exchange information and resources.

In this context, we consider a packet routing scenario, focusing on how to reliably route packets of data through a multi-hop communication network. Such a network is typically represented by a directed graph \( G(V,E) \), where each node \( v \in V \) represents a router and any edge \( (u,v) \in E \) represents the packet transmission process between two end nodes \( u,v \in V \). Our goal is to send a stream of \( T \) packets from a source node to a destination node to maximize the number of packets that are successfully transmitted.

Conjunctive cascading bandit models \cite{kveton2015combinatorial} the online learning to route problem for packet routing. At each round \( t \in [T] \), the learner first observes the contextual information \( \bphi_t \) about the network conditions between any of the two routers \( (u,v) \in E \), such as network congestion and router bandwidth. Then, the learner selects a path \( S_t = (e_{t,1}, ..., e_{t,K}) \subseteq E \) connecting the source to the destination, and the packet \( t \) is sent along this path. 
Depending on the network traffic and the capacity of the routers, each edge \( e = (u,v) \in E \) has a probability \( \mu_{t,e} = \ell({\btheta^*}^{\top} \bphi_t(e)) \) of being a good edge that successfully sends data from end routers \( u \) to \( v \), and a probability \( 1 - \mu_{t,e} \) of being a broken edge, where \( \btheta^* \in \mathbb{R}^d \) is the unknown parameter. Here, the sigmoid function models the nonlinear relationship between the success probability and features such as congestion level.
If the path \( S_t \) does not contain any broken edges, then the packet is transmitted successfully from the source to the destination, obtaining a reward of 1. Otherwise, the packet transmission will stop at the first broken edge, and the learner will get zero reward. As for feedback, if the packet is transmitted successfully along path \( S_t \), the learner observes \( K \) Bernoulli outcomes \( (1, ..., 1) \), indicating all the edges in the path \( S_t \) are good edges. On the other hand, if the packet fails to transmit via this path, the learner observes the feedback of edges before (and including) the first broken edge (suppose its index is \( j_t \)-th), which is of the form \( (1, ..., 1, 0, \text{x}, ..., \text{x}) \), meaning the first \( j_t-1 \) edges are good (denoted as 1), the \( j_t \)-th edge is broken (denoted as 0), and the outcomes of the rest of the edges are unobserved (denoted as x).
The reward function for this application is \( r(S_t; \boldsymbol{\mu}_t) = \prod_{e \in S_t} \mu_{t,e} \). This application fits into the CLogB framework, satisfying the 1-norm TPVM smoothness condition with \( B_v = 1, B_1 = 1, \lambda = 1 \) as in \cite{liu2023contextual}. For the offline oracle, we can define $\mu'_{t,e}\defeq-\log (\mu_{t,e})$ for $e\in E$ and construct a new graph $G'(V,E)$ by changing the edge weight of $G(V,E)$ with $\mu'_{t,e}$ for all $e\in E$. Then we can use Dijkstra's shortest path algorithm on $G'$ to find the optimal path $S_t^*$ ($\alpha=1$) with time complexity $T_{\alpha}=O(|E| + |V|\log |V|)$.

\section{Notations and Preliminaries}\label{apdx_sec:notation}
We collect a list of definitions that will be used throughout the appendix. Recall that the regularized log-loss is defined as: 

\begin{equation}\label{apdx_eq:log-likelihood}
\begin{array}{r}
\mathcal{L}_t(\btheta)\defeq-\sum_{s=1}^{t-1}\sum_{i\in \tau_s}\left[X_{s,i} \log \ell\left(\btheta^{\top}\bphi_s(i) \right)+\left(1-X_{s,i}\right)\right. 
\left.\cdot \log \left(1-\ell\left(\btheta^{\top}\bphi_s(i) \right)\right)\right]+\frac{\lambda_t}{2}\|\btheta\|_2^2.
\end{array}
\end{equation}

Let us define the following notations:

\begin{align}
\alpha\left(\bx, \btheta_1, \btheta_2\right) & :=\int_{v=0}^1 \dot{\ell}\left(v \bx^{\top} \btheta_2+(1-v) \bx^{\top} \btheta_1\right) d v=\alpha\left(\bx, \btheta_2, \btheta_1\right)>0 \label{apdx_eq:alpha}\\
\tilde{\alpha}\left(\bx, \btheta_1, \btheta_2\right) & :=\int_{v=0}^1 (1-v)\dot{\ell}\left(v \bx^{\top} \btheta_2+(1-v) \bx^{\top} \btheta_1\right) d v>0 \label{apdx_eq:tilde_alpha}\\
\mathbf{G}_t\left(\btheta_1, \btheta_2\right) & :=\sum_{s=1}^{t-1}\sum_{i\in\tau_s} \alpha\left(\bphi_s(i), \btheta_1, \btheta_2\right) \bphi_s(i) \bphi_s(i)^{\top}+\lambda_t \mathbf{I}_d \\
\tilde{\mathbf{G}}_t\left(\btheta_1, \btheta_2\right) & :=\sum_{s=1}^{t-1}\sum_{i\in\tau_s} \tilde{\alpha}\left(\bphi_s(i), \btheta_1, \btheta_2\right) \bphi_s(i) \bphi_s(i)^{\top}+\lambda_t \mathbf{I}_d \\
\mathbf{H}_t\left(\btheta\right) & :=\sum_{s=1}^{t-1}\sum_{i\in\tau_s} \dot{\ell}\left(\bphi_s(i)^{\top} \btheta\right) \bphi_s(i) \bphi_s(i)^{\top}+\lambda_t \mathbf{I}_d \\
\bg_t(\btheta)&\defeq\sum_{s=1}^{t-1}\sum_{i\in \tau_s}\ell\left(\btheta^{\top}\bphi_s(i) \right)\bphi_s(i)+\lambda_t\btheta \label{apdx_eq:gt}\\
\mathbf{V}_t & :=\sum_{s=1}^{t-1}\sum_{i\in\tau_s} \bphi_s(i) \bphi_s(i)^{\top}+\kappa \lambda_t \mathbf{I}_d
\end{align}

Note that for all $\bx \in \mathbb{R}^d$ and $\btheta_1, \btheta_2 \in \mathbb{R}^d$, the following equality holds due to the first and the second-order Taylor expansion:
\begin{align}
    \ell\left(\bx^{\boldsymbol{\top}} \btheta_2\right)-\ell\left(\bx^{\boldsymbol{\top}} \btheta_1\right)&=\alpha\left(\bx, \btheta_1, \btheta_2\right) \bx^{\boldsymbol{\top}}\left(\btheta_2-\btheta_1\right)\label{apdx_eq:1_taylor_l}\\
     \ell\left(\bx^{\top} \btheta_2\right)-\ell\left(\bx^{\top} \btheta_1\right)&=\dot{\ell}\left(\bx^{\top} \btheta_1\right) \bx^{\top}\left(\btheta_2-\btheta_1\right)+\left(\bx^{\top}\left(\btheta_2-\btheta_1\right)\right)^2\int_{v=0}^1 (1-v)\ddot{\ell}\left(v \bx^{\top} \btheta_2+(1-v) \bx^{\top} \btheta_1\right) d v\label{apdx_eq:2_taylor_l}
\end{align}

\cref{apdx_eq:1_taylor_l}, \cref{apdx_eq:gt} and \cref{apdx_eq:alpha} allow us to link $\btheta_2-\btheta_1$ with $g_t\left(\btheta_2\right)-g_t\left(\btheta_1\right)$. Namely, it holds that:
\begin{align}
g_t\left(\btheta_2\right)-g_t\left(\btheta_1\right) & =\sum_{s=1}^{t-1}\sum_{i\in \tau_s} \alpha\left(\bphi_s(i), \btheta_1, \btheta_2\right) \bphi_s(i) \bphi_s(i)^{\top}\left(\btheta_2-\btheta_1\right)+\lambda \btheta_2-\lambda \btheta_1\\
& =\mathbf{G}_t\left(\btheta_1, \btheta_2\right)\left(\btheta_2-\btheta_1\right)
\end{align}

Because $\mathbf{G}_t\left(\btheta_1, \btheta_2\right) \succ \mathbf{0}_{d \times d}$, this yields:
\begin{align}\label{apdx_eq:link_G_g}
    \left\|\btheta_1-\btheta_2\right\|_{\mathbf{G}_t\left(\btheta_1, \btheta_2\right)}=\left\|g_t\left(\btheta_1\right)-g_t\left(\btheta_2\right)\right\|_{\mathbf{G}_t^{-1}\left(\btheta_1, \btheta_2\right)}
\end{align}

Additionally, by the first-order Taylor expansion, we can also link $\btheta_1-\btheta_2$ with the difference in loss functions $\mathcal{L}_t\left(\btheta_2\right)-\mathcal{L}_t\left(\btheta_1\right)$:
\begin{align}\label{apdx_eq:loss_talor}
    \mathcal{L}_t\left(\btheta_2\right)-\mathcal{L}_t\left(\btheta_1\right)=\left.\nabla \mathcal{L}_t\right({\btheta_1}) ^{\top}\left(\btheta_2-\btheta_1\right)+\left(\btheta_2-\btheta_1\right)^{\top} \widetilde{\mathbf{G}}_{\mathbf{t}}\left(\btheta_1, \btheta_2\right)\left(\btheta_2-\btheta_1\right).
\end{align}

We have the following inequality to link $\bH_t(\btheta)$ with $\bV_t$:

\begin{align}
    \bH_t(\btheta)\succeq \kappa^{-1} \bV_t.
\end{align}

By \cref{apdx_lem:scc}, we have the following inequality to link $\bG_t(\btheta_1, \btheta_2), \tilde{\bG}_t(\btheta_1, \btheta_2)$ and $\bH_t(\btheta_1), \bH_t(\btheta_2)$ for any $\btheta_1, \btheta_2 \in \Theta$:
\begin{align}
    \bG_t(\btheta_1, \btheta_2) &\succeq (1+2L)^{-1} \bH_t(\btheta_1)\\
    \bG_t(\btheta_1, \btheta_2) &\succeq (1+2L)^{-1} \bH_t(\btheta_2)\\
\tilde{\bG}_t(\btheta_1, \btheta_2) &\succeq (2+2L)^{-1} \bH_t(\btheta_1)\\
\end{align}

By Lemma 9 of \cite{abeille2021instance}, it holds that for any $x,y\in \R$,
\begin{align}
    \dot{\ell}(x)\le \dot{\ell} (y) \exp(\abs{x-y})
\end{align}

Below we define the following functions that scale with $\tilde{O}(d)$.

\begin{align}
    \lambda_t(\delta)&=d\log \left(\frac{4(1+tK)}{\delta}\right)\\
    \gamma_t(\delta)&=\left(L+\frac{3}{2}\right)\sqrt{\lambda_t(\delta)}
\end{align}

\section{Confidence Region and Exploration Bonus}\label{apdx_sec:confidence_region}

\subsection{Concentration Inequality}\label{apdx_sec:proof_concen_ineq}

\begin{proof}[Proof of \cref{lem:est_dist}]
    
Our result is built upon the following concentration inequality.

\begin{proposition}[Theorem 4 of \cite{abeille2021instance}]\label{apdx_thm:prev_concentration}
Let $\left\{\mathcal{F}_t\right\}_{t=1}^{\infty}$ be a filtration. Let $\left\{\bx_t\right\}_{t=1}^{\infty}$ be a stochastic process such that $\bx_t\in \R^d, \norm{\bx_t}_2\le 1$ and $\bx_t$ is $\mathcal{F}_t$ measurable. Let $\left\{\varepsilon_t\right\}_{t=1}^{\infty}$ be a martingale difference sequence such that $\varepsilon_{t}$ is $\mathcal{F}_{t+1}$ measurable. Furthermore, assume that conditionally on $\mathcal{F}_{t-1}$ we have $\left|\varepsilon_{t+1}\right| \leq 1$ almost surely, and note $\sigma_t^2\defeq\mathbb{E}\left[\varepsilon_{t+1}^2 \mid \mathcal{F}_t\right]$. Let $\lambda>0$ and for any $t \geq 1$ define:
$$
\mathbf{H}_t\defeq\sum_{s=1}^{t} \sigma_s^2 \bx_s \bx_s^T+\lambda_t \mathbf{I}_d, \quad \bS_t\defeq\sum_{s=1}^{t} \varepsilon_{s+1} \bx_s .
$$

Then for any $\delta \in(0,1]$ :
$$
\mathbb{P}\left(\exists t \geq 1,\left\|\bS_t\right\|_{\mathbf{H}_t^{-1}} \geq \frac{\sqrt{\lambda_t}}{2}+\frac{2}{\sqrt{\lambda_t}} \log \left(\frac{\operatorname{det}\left(\mathbf{H}_{\mathbf{t}}\right)^{\frac{1}{2}} \lambda_t^{-\frac{d}{2}}}{\delta}\right)+\frac{2}{\sqrt{\lambda_t}} d \log (2)\right) \leq \delta
$$
\end{proposition}

By the optimality condition of the MLE ($\nabla \cL_t(\hat{\btheta}_t)=0$) and  
\cref{eq:grad_MLE}, the following holds:
\begin{align}
    \bg_t(\hat{\btheta}_t)&=\sum_{s=1}^{t-1} \sum_{i\in \tau_s} X_{s,i}\bphi_s(i).
\end{align}

Let $\varepsilon_{s,i}\defeq \left(X_{s,i} - \ell\left({\btheta^*}^{\top}\bphi_s(i) \right)\right)$ and by subtracting $g_t(\btheta^*)=\sum_{s=1}^{t-1} \sum_{i\in \tau_s}  \ell\left({\btheta^*}^{\top}\bphi_s(i) \right)\bphi_s(i) + \lambda_t \btheta^*$, we have:
\begin{align}
    \bg_t\left(\hat{\btheta}_t\right)-\bg_t\left(\btheta^*\right)&=\sum_{s=1}^{t-1} \sum_{i\in \tau_s} \left(X_{s,i} - \ell\left({\btheta^*}^{\top}\bphi_s(i) \right)\right)\bphi_s(i) - \lambda_t \btheta^*\\
    &=\sum_{s=1}^{t-1} \sum_{i\in \tau_s} \varepsilon_{s,i}\bphi_s(i) - \lambda_t \btheta^*.
\end{align}

Let $\bZ_t \defeq \sum_{s=1}^{t-1} \sum_{i\in \tau_s} \varepsilon_{s,i}\bphi_s(i)$, by the direct computation:
\begin{align}
    \norm{\bg_t\left(\hat{\btheta}_t\right)-\bg_t\left(\btheta^*\right)} _{\bH_t^{-1}(\btheta^*)}&=\norm{\bZ_t-\lambda_t \btheta^*} _{\bH_t^{-1}(\btheta^*)}\\
    &\overset{(a)}\le \norm{\bZ_t} _{\bH_t^{-1}(\btheta^*)} + \lambda_t \norm{\btheta^*} _{\bH_t^{-1}(\btheta^*)}\\
    &\overset{(b)} \le \norm{\bZ_t} _{\bH_t^{-1}(\btheta^*)} + \sqrt{\lambda_t} L \label{apdx_eq:ZH}
\end{align}
where inequality (a) is due to the triangle inequality, inequality (b) is due to $\bH_t(\btheta^*) \succeq \lambda_t \bI_d$ by definition and $\norm{\btheta^*}_2\le L$ by \cref{cond:bounded_para}.

To apply \cref{apdx_thm:prev_concentration}, fix $t\in [T]$, we reorder the feature vector $\bphi_s(i)$ and the noise $\varepsilon_{s,i}$ for $(s,i)\in [t-1]\times [m]$ under the lexicographic ordering, i.e., the ordering $\prec_{\text{lex}}$ over $[t-1]\times [m]$ that satisfies $(s_1, i_1)\prec_{\text{lex}}(s_2, i_2)$ if $s_1 < s_2$ or $s_1=s_2, i_1 <i_2$. 

Specifically, for each $j \in [(t-1)m]$, we define
\begin{equation}
\begin{gathered}
\rho(j)=\left\lfloor 1+\frac{j-1}{m}\right\rfloor \in[t-1], \quad \sigma(j)=j-(\rho(j)-1) m \in[m], \\
\tilde{\bphi}_{j}=\bphi_{\rho(j)}(\sigma(j)) \mathbb{I}\left(\sigma(j) \in \tau_{\rho(j)}\right), \quad \tilde{\varepsilon}_{j+1}=\varepsilon_{\rho(j),\sigma(j)}, \quad \tilde{\boldsymbol{H}}_{j}=\sum_{k=1}^j \dot{\ell}\left({\btheta^*}^{\top}\bphi_{\rho(k)}(\sigma(k))\right)\tilde{\bphi}_{k} \tilde{\bphi}_{k}^{\top}+\lambda_t\bI_d.
\end{gathered}
\end{equation}

Let $\cF_j$ denote the $\sigma$-algebra generated by $\{\tilde{\bphi}_k\}_{k=1}^{j} \cup \{\tilde{\varepsilon}_{k+1}\}_{k=1}^{j-1}$.
Recall that conditionally on history $\cH_s$, outcomes $X_{s,i}\sim \text{Bernoulli}\left( \ell\left({\btheta^*}^{\top}\bphi_s(i) \right)\right)$ for $i\in [m]$ are independent Bernoulli random variables. 
We can verify that $\tilde{\bphi}_j$ is $\cF_j$ measurable and each $\tilde{\varepsilon}_{j+1}$ is $[-1,1]$ valued and $\cF_{j+1}$ measurable satisfying:
\begin{align}
 &\E \left[ \tilde{\varepsilon}_{j+1} \mid \cF_j\right] = 0\\
 &\E\left[\tilde{\varepsilon}_{j+1}^2 \mid \cF_j\right] = \Var[X_{\rho(j), \sigma(j)}] = \dot{\ell}\left({\btheta^*}^{\top}\bphi_{\rho(j)}(\sigma(j))\right)
\end{align}

Now for any $t$, we define  
\begin{align}
    \bH_t\defeq\sum_{k=1}^{(t-1)m} \dot{\ell}\left({\btheta^*}^{\top}\bphi_{\rho(k)}(\sigma(k))\right)\tilde{\bphi}_{k} \tilde{\bphi}_{k}^{\top}+\lambda_t\bI_d &= \sum_{s=1}^{t-1}\sum_{i\in\tau_s} \dot{\ell}\left({\btheta^*}^{\top} \bphi_s(i)\right) \bphi_s(i) \bphi_s(i)^{\top}+\lambda_t \mathbf{I}_d\\
    \bS_t\defeq\sum_{j=1}^{(t-1)m} \tilde{\varepsilon}_{j+1}\tilde{\bphi}_j&=\sum_{s=1}^{t-1}\sum_{i\in\tau_s} \varepsilon_{s,i} \bphi_s(i)
\end{align}

Since all conditions of \cref{apdx_thm:prev_concentration} are met and we have with probability at least $1-\delta$:
\begin{align}
    \norm{\bS_t}_{\bH_t^{-1}}&\le \frac{\sqrt{\lambda_t}}{2}+\frac{2}{\sqrt{\lambda_t}} \log \left(\frac{\operatorname{det}\left(\mathbf{H}_{\mathbf{t}}\right)^{\frac{1}{2}} \lambda_t^{-\frac{d}{2}}}{\delta}\right)+\frac{2}{\sqrt{\lambda_t}} d \log (2)
\end{align}

Now, we need to bound $\det(\bH_t)$:
\begin{align}
    \det(\bH_t)&=\det \left(\sum_{s=1}^{t-1}\sum_{i\in\tau_s} \dot{\ell}\left({\btheta^*}^{\top} \bphi_s(i)\right) \bphi_s(i) \bphi_s(i)^{\top}+\lambda_t \mathbf{I}_d\right)\\
    &\le \det \left(\sum_{s=1}^{t-1}\sum_{i\in\tau_s} \bphi_s(i) \bphi_s(i)^{\top}+\lambda_t \mathbf{I}_d\right)\\
    &\overset{(a)}\le (\lambda_t + tK/d)^d
\end{align}
where inequality (a) is due to \cref{apdx_lem:det_and_trace} and the fact that $\norm{\bphi_s(i)}_2\le 1, |\tau_s|\le K$.

Therefore, we have 
\begin{align}\label{apdx_eq:SH}
    \norm{\bS_t}_{\bH_t^{-1}}&\le \frac{\sqrt{\lambda_t}}{2}+\frac{2}{\sqrt{\lambda_t}} \log \left(\frac{\left(\lambda_t+t K / d\right)^{d / 2} \lambda_t^{-d / 2}}{\delta}\right)+\frac{2}{\sqrt{\lambda_t}} d\log 2\\
    &\le \frac{\sqrt{\lambda_t}}{2}+\frac{d}{\sqrt{\lambda_t}} \log \left(\frac{4}{\delta}\left(1+\frac{tK}{\lambda_t d}\right)\right)
\end{align}

Plugging \cref{apdx_eq:SH} into \cref{apdx_eq:ZH} with the fact that $\bZ_t=\bS_t, \bH_t=\bH_t(\btheta^*)$, we have:
\begin{align}
    \left\|\bg_t\left(\hat{\btheta}_t\right)-\bg_t\left(\btheta^*\right)\right\|_{\bH_t^{-1}(\btheta^*)} &\le \left(L+\frac{1}{2}\right)\sqrt{\lambda_t}+\frac{d}{\sqrt{\lambda_t}} \log \left(\frac{4}{\delta}\left(1+\frac{tK}{\lambda_t d}\right)\right)\\
    &\le \left(L+\frac{3}{2}\right)\sqrt{d\log \left(\frac{4(1+tK)}{\delta}\right)}
\end{align}
which concludes the lemma.
\end{proof}

\subsection{Confidence Regions}

\subsubsection{Proof for the Variance-Agnostic Confidence Region in Section \ref{sec:vag_conf_bonus}}\label{apdx_sec:proof_ag_conf}
\begin{proof}[Proof of \cref{lem:var_ag_conf}]

We first build the connection between the $\bH_t(\btheta^*)$ and $\bG_t(\btheta^*,\hat{\btheta}_t)$

\begin{align}
\mathbf{G}_{\mathbf{t}}\left(\btheta^*, \hat{\btheta}_t\right) & =\sum_{s=1}^{t-1}\sum_{i\in\tau_s} \alpha\left(\bphi_s(i), \btheta^*, \hat{\btheta}_t\right) \bphi_s(i) \bphi_s(i)^{\top}+\lambda_t \mathbf{I}_{\mathbf{d}} \\
& \overset{(a)}{\succeq} \sum_{s=1}^{t-1}\sum_{i\in \tau_s}\left(1+\left|\bphi_s(i)^{\top}\left(\btheta^*-\hat{\btheta}_t\right)\right|\right)^{-1} \dot{\ell}\left(\bphi_s(i)^{\top} \btheta^*\right) \bphi_s(i) \bphi_s(i)^{\top}+\lambda_t \mathbf{I}_{\mathbf{d}} \\
& \overset{(b)}{\succeq} \sum_{s=1}^{t-1}\sum_{i\in \tau_s}\left(1+\left\|\bphi_s(i)\right\|_{\mathbf{G}_{\mathbf{t}}^{-1}\left(\btheta^*, \hat{\btheta}_t\right)}\left\|\btheta^*-\hat{\btheta}_t\right\|_{\mathbf{G}_{\mathbf{t}}\left(\btheta^*, \hat{\btheta}_t\right)}\right)^{-1} \dot{\ell}\left(\bphi_s(i)^{\top} \btheta^*\right) \bphi_s(i) \bphi_s(i)^{\top}+\lambda_t \mathbf{I}_{\mathbf{d}} \\
& \succeq\left(1+\lambda_t^{-1 / 2}\left\|\btheta^*-\hat{\btheta}_t\right\|_{\mathbf{G}_{\mathbf{t}}\left(\btheta^*, \hat{\btheta}_t\right)}\right)^{-1} \sum_{s=1}^{t-1}\sum_{i\in \tau_s} \dot{\ell}\left(\bphi_s(i)^{\top} \btheta^*\right) \bphi_s(i) \bphi_s(i)^{\top}+\lambda_t \mathbf{I}_{\mathbf{d}} \\
& \succeq\left(1+\lambda_t^{-1 / 2}\left\|\btheta^*-\hat{\btheta}_t\right\|_{\mathbf{G}_{\mathbf{t}}\left(\btheta^*, \hat{\btheta}_t\right)}\right)^{-1}\left(\sum_{s=1}^{t-1}\sum_{i\in \tau_s} \dot{\ell}\left(\bphi_s(i)^{\top} \btheta^*\right) \bphi_s(i) \bphi_s(i)^{\top}+\lambda_t \mathbf{I}_{\mathbf{d}}\right) \\
& =\left(1+\lambda_t^{-1 / 2}\left\|\btheta^*-\hat{\btheta}_t\right\|_{\mathbf{G}_{\mathbf{t}}\left(\btheta^*, \hat{\btheta}_t\right)}\right)^{-1} \mathbf{H}_{\mathbf{t}}(\btheta^*) \\
& \overset{(c)}{=}\left(1+\lambda_t^{-1 / 2}\left\|g_t(\btheta^*)-g_t\left(\hat{\btheta}_t\right)\right\|_{\mathbf{G}_{\mathbf{t}}^{-1}\left(\btheta^*, \hat{\btheta}_t\right)}\right)^{-1} \mathbf{H}_{\mathbf{t}}(\btheta^*)\label{apdx_eq:link_G_hat_H}
\end{align}
where inequality (a) is due to \cref{apdx_eq:scc1} in \cref{apdx_lem:scc}, inequality (b) is due to Cauchy-Schwarz inequality, and equality (c) is due to \cref{apdx_eq:link_G_g}.

Using this inequality, we have:

\begin{align}
\left\|g_t(\btheta^*)-g_t\left(\hat{\btheta}_t\right)\right\|_{\mathbf{G}_{\mathbf{t}}^{-1}\left(\btheta^*, \hat{\btheta}_t\right)}^2 & \leq\left(1+\lambda_t^{-1 / 2}\left\|g_t(\btheta^*)-g_t\left(\hat{\btheta}_t\right)\right\|_{\mathbf{G}_{\mathbf{t}}^{-1}\left(\btheta^*, \hat{\btheta}_t\right)}\right)\left\|g_t(\btheta^*)-g_t\left(\hat{\btheta}_t\right)\right\|_{\mathbf{H}_{\mathbf{t}}^{-1}(\btheta^*)}^2 \\
& \overset{(a)}{\leq} \lambda_t^{-1 / 2} \gamma_t^2(\delta)\left\|g_t(\btheta^*)-g_t\left(\hat{\btheta}_t\right)\right\|_{\mathbf{G}_{\mathbf{t}}^{-1}\left(\btheta^*, \hat{\btheta}_t\right)}+\gamma_t^2(\delta) 
\end{align}
where inequality (a) is due to \cref{lem:est_dist} which holds with probability at least $1-\delta$.

By \cref{apdx_lem:Poly_ieq}, solving the above quadratic inequality w.r.t.  $\left\|g_t(\btheta^*)-g_t\left(\hat{\btheta}_t\right)\right\|_{\mathbf{G}_{\mathbf{t}}^{-1}\left(\btheta^*, \hat{\btheta}_t\right)}$, we have:

\begin{align}\label{apdx_eq:bound_g_G}
    \left\|g_t(\btheta^*)-g_t\left(\hat{\btheta}_t\right)\right\|_{\mathbf{G}_{\mathbf{t}}^{-1}\left(\btheta^*, \hat{\btheta}_t\right)}\le \frac{\gamma_t^2(\delta)}{\lambda_t^{1/2}} + \gamma_t(\delta)=\left(L^2+4L+\frac{15}{4}\right)\sqrt{\lambda_t}.
\end{align}

Using the above inequality and \cref{apdx_eq:link_G_hat_H}, we have:
\begin{align}
    \bH_t(\btheta^*)\preceq \left(L^2+4L+\frac{19}{4}\right)\mathbf{G}_{\mathbf{t}}\left(\btheta^*, \hat{\btheta}_t\right)
\end{align}

Using the above inequality, we have:
\begin{align}
    \norm{\btheta^*-\hat{\btheta}_t}_{\bH_{t}(\btheta^*)}&\le \sqrt{L^2+4L+\frac{19}{4}}\norm{\btheta^*-\hat{\btheta}_t}_{\mathbf{G}_{\mathbf{t}}\left(\btheta^*, \hat{\btheta}_t\right)}\\
    &\overset{(a)}{=}\sqrt{L^2+4L+\frac{19}{4}}\left\|g_t(\btheta^*)-g_t\left(\hat{\btheta}_t\right)\right\|_{\mathbf{G}_{\mathbf{t}}^{-1}\left(\btheta^*, \hat{\btheta}_t\right)}\\
    &\overset{(b)}{\le} \left(L^2+4L+\frac{19}{4}\right)\sqrt{\lambda_t}
\end{align}
where equality (a) is due to \cref{apdx_eq:link_G_g} and inequality (b) is due to \cref{apdx_eq:bound_g_G}.

Therefore we can derive that:
\begin{align}\label{apdx_eq:vag_conf_bound}
    \norm{\btheta^*-\hat{\btheta}_t}_{\bV_{t}}\overset{(a)}{\le} \sqrt{\kappa}\norm{\btheta^*-\hat{\btheta}_t}_{\bH_{t}^*(\btheta^*)}\le \left(L^2+4L+\frac{19}{4}\right)\sqrt{\kappa\lambda_t}
\end{align}
where inequality (a) is due to $\bH_t(\btheta^*)\succeq \kappa^{-1} \bV_t$.

This concludes the proof of \cref{lem:var_ag_conf}.
\end{proof}

\subsubsection{Proof for the Variance-Adaptive Region in Section \ref{sec:vad_conf_bouns}}\label{apdx_sec:proof_vad_conf}
\begin{proof}[Proof of \cref{lem:var_ad_conf}]
By \cref{apdx_eq:scc1_2} in \cref{apdx_lem:scc}, for any $\btheta_1,\btheta_2 \in \Theta$, we have:
\begin{align}\label{apdx_eq:link_H_G}
    \bG_t(\btheta_1,\btheta_2)&= \sum_{s=1}^{t-1}\sum_{i\in \tau_s} \alpha\left(\bphi_s(i), \btheta_1, \btheta_2\right) \bphi_s(i) \bphi_s(i)^{\top}+\lambda_t \mathbf{I}_{\mathbf{d}} \\
& {\succeq} \sum_{s=1}^{t-1}\sum_{i\in \tau_s}\left(1+2L\right)^{-1} \dot{\ell}\left(\bphi_s(i)^{\top} \btheta_1\right) \bphi_s(i) \bphi_s(i)^{\top}+\lambda_t \mathbf{I}_{\mathbf{d}} \\
&=(1+2L)^{-1} \bH_t(\btheta_1)\\
 \bG_t(\btheta_1,\btheta_2)&= \sum_{s=1}^{t-1}\sum_{i\in \tau_s} \alpha\left(\bphi_s(i), \btheta_1, \btheta_2\right) \bphi_s(i) \bphi_s(i)^{\top}+\lambda_t \mathbf{I}_{\mathbf{d}} \\
& {\succeq} \sum_{s=1}^{t-1}\sum_{i\in \tau_s}\left(1+2L\right)^{-1} \dot{\ell}\left(\bphi_s(i)^{\top} \btheta_2\right) \bphi_s(i) \bphi_s(i)^{\top}+\lambda_t \mathbf{I}_{\mathbf{d}} \\
&=(1+2L)^{-1} \bH_t(\btheta_2)
\end{align}

Using the above inequalities and the fact that $\btheta^*, \hat{\btheta}_{t,H}\in \Theta$, we have 
\begin{align}
    \bH_t(\btheta^*) \preceq (1+2L) \bG_t(\btheta^*, \hat{\btheta} _{t,H}) \\
    \bH_t(\hat{\btheta}_{t,H})  \preceq (1+2L)\bG_t(\btheta^*,\hat{\btheta}_{t,H})
\end{align}

Therefore:
\begin{align}
    &\norm{\btheta^*-\hat{\btheta}_{t,H}}_{\bH_t( \hat{\btheta}_{t,H})}\\
    &\le \sqrt{1+2L}  \norm{\btheta^*-\hat{\btheta}_{t,H}}_{\bG_t(\btheta^*,\hat{\btheta}_{t,H})}\\
    &\overset{(a)}{=}\sqrt{1+2L}  \norm{\bg_t\left(\btheta^*\right)-\bg_t\left(\hat{\btheta}_{t,H}\right)}_{\bG_t^{-1}(\btheta^*,\hat{\btheta}_{t,H})}\\
    &\le\sqrt{1+2L}  \norm{\bg_t\left(\btheta^*\right)-\bg_t\left(\hat{\btheta}_{t}\right)}_{\bG_t^{-1}(\btheta^*,\hat{\btheta}_{t,H})}+\sqrt{1+2L}  \norm{\bg_t\left(\hat{\btheta}_t\right)-\bg_t\left(\hat{\btheta}_{t,H}\right)}_{\bG_t^{-1}(\btheta^*,\hat{\btheta}_{t,H})}\\ 
    &\le (1+2L)  \norm{\bg_t\left(\btheta^*\right)-\bg_t\left(\hat{\btheta}_{t}\right)}_{\bH_t(\btheta^*)}+(1+2L)  \norm{\bg_t\left(\hat{\btheta}_t\right)-\bg_t\left(\hat{\btheta}_{t,H}\right)}_{\bH(\hat{\btheta}_{t,H})}\\
    &\overset{(b)}{\le} 2(1+2L)  \norm{\bg_t\left(\btheta^*\right)-\bg_t\left(\hat{\btheta}_{t}\right)}_{\bH_t(\btheta^*)}\\
    &\overset{(c)}{\le} 2(1+2L)\gamma_t(\delta)
\end{align}
where equality (a) is due to \cref{apdx_eq:link_G_g}, inequality (b) uses the definition of projected MLE $\hat{\btheta}_{t,H} = \arg\min_{\btheta \in \cQ_t}\norm{ g_t(\btheta) - g_t(\hat{\btheta}_t) } _{\bH_t^{-1}(\btheta)}$ and the fact that $\btheta^*\in \cQ_t$ (\cref{eq:Q_region}) under high probability event $\btheta^* \in \cA_s(\delta)$, and inequality (c) is due to \cref{lem:est_dist}.
\end{proof}

\subsubsection{Proof for the Variance-Adaptive Region after the Burn-in Stage in Section \ref{sec:eva_conf_bonus}} \label{apdx_sec:proof_eva_conf}
\begin{proof}[Proof of \cref{lem:Q_region}]
Recall that our feature map $\bphi_t=\bphi$ is time-invariant. In rounds $t=1, ..., T_0$, we define the covariance matrix
\begin{align}
    \bV_t = \sum_{s=1}^{t-1}\bphi(i_s) \bphi(i_s)^{\top} +\kappa\lambda_t \bI_d
\end{align}
where $i_t=\argmax_{i\in[m]} \norm{\bphi(i)}_{\bV_t^{-1}}$ for $t \in [T_0]$ and $\lambda_t=d \log\left( \frac{4(2+t)}{\delta}\right)$ for $t \in [T_0+1]$.

By the similar analysis of \cref{lem:var_ag_conf} where we identify $T=T_0, \tau_s=i_s, K=1$, we have the same bound of \cref{apdx_eq:vag_conf_bound}  with probability at least $1-\delta$,
\begin{align}
     \norm{\btheta^*-\hat{\btheta}_{T_0+1}}_{\bV_{T_0+1}}\le \left(L^2+4L+\frac{19}{4}\right)\sqrt{\kappa\lambda_{T_0+1}}=\left(L^2+4L+\frac{19}{4}\right)\sqrt{\kappa d \log\left( \frac{4(2+T_0)}{\delta}\right)}.
\end{align}

Recall that $\cQ=\left\{\btheta \in \Theta: \norm{\btheta-\hat{\btheta}_{T_0+1}}_{\bV_{T_0+1}} \le \left(L^2+4L+\frac{19}{4}\right)\sqrt{\kappa\lambda_{T_0+1}} \right\}$, we have 
\begin{align}
    \text{diam}(\cQ)&=\max_{i \in [m], \btheta_1, \btheta_2 \in \cQ} \abs{\bphi(i)^{\top} (\btheta_1-\btheta_2)}\\
    &\le\max_{i \in [m], \btheta_1, \btheta_2\in \cQ} \norm{\bphi(i)}_{\bV_{T_0+1}^{-1}} \norm{\btheta_1 -\btheta_2}_{\bV_{T_0+1}}\\
    &\overset{(a)}{\le} \max_{i \in [m], \btheta_1, \btheta_2 \in \cQ} \norm{\bphi(i)}_{\bV_{T_0+1}^{-1}} \left(\norm{\btheta_1-\hat{\btheta}_{T_0+1}}_{\bV_{T_0+1}}+ \norm{\btheta_2-\hat{\btheta}_{T_0+1}}_{\bV_{T_0+1}}\right)\\
    &\overset{(b)}{\le} 2\left(L^2+4L+\frac{19}{4}\right)\sqrt{\kappa\lambda_{T_0+1}} \max_{i\in[m]} \norm{\bphi(i)}_{\bV_{T_0+1}^{-1}}\\
    &\overset{(c)}{\le} 2\left(L^2+4L+\frac{19}{4}\right)\sqrt{\kappa\lambda_{T_0+1}} \frac{\sum_{t=1}^{T_0}\norm{\bphi(i)}_{\bV_{t}^{-1}}}{T_0}\\
    &\overset{(d)}{\le} 2\left(L^2+4L+\frac{19}{4}\right)\frac{\sqrt{\kappa\lambda_{T_0+1}}} {\sqrt{T_0}} \sqrt{\sum_{t=1}^{T_0}\norm{\bphi(i)}^2_{\bV_{t}^{-1}}}\\
    &= 2\left(L^2+4L+\frac{19}{4}\right)\frac{\sqrt{\kappa d \log\left( \frac{4(2+T_0)}{\delta}\right)}} {\sqrt{T_0}} \sqrt{\sum_{t=1}^{T_0}\norm{\bphi(i)}^2_{\bV_{t}^{-1}}}\\
    &\overset{(e)}{\le} 2\left(L^2+4L+\frac{19}{4}\right)\frac{\sqrt{\kappa} 2d \log\left( \frac{4(2+T_0)}{\delta}\right)} {\sqrt{T_0}}\\
    &\overset{}{\le} 2\left(L^2+4L+\frac{19}{4}\right)\frac{\sqrt{\kappa} 2d \log\left( \frac{4(2+T)}{\delta}\right)} {\sqrt{T_0}}\\
    &\overset{(f)}{\le} 1 \label{apdx_eq:D<1}
\end{align}
where inequality (a) is due to the triangle inequality, inequality (b) is due to the definition of $\cQ$, inequality (c) is due to the the fact that $\bV_t \preceq \bV_{T_0+1}$ for $t\in[T_0]$ so that $\norm{\bphi(i_t)}_{\bV_t^{-1}}=\max_{i\in[m]}\norm{\bphi(i)}_{\bV_t^{-1}}\ge \max_{i\in[m]}\norm{\bphi(i)}_{\bV_{T_0+1}^{-1}}$, inequality (d) is due to Cauchy-Schwarz inequality, inequality (e) is due to \cref{apdx_lem:elliptical_potential}, inequality (f) is by $T_0 \defeq \left(4L^2+16L+19\right)^2 \kappa d^2 \log^2\left( \frac{4(2+T)}{\delta}\right)$

\end{proof}

\begin{proof}[Proof of \cref{lem:eva_conf}]

Let $\cQ \in \Theta$ be a region where $\text{diam}(\cQ)\le D$.
Suppose $\btheta^* \in \cQ$.
We first use the second-order Taylor expansion in \cref{apdx_eq:loss_talor} for the loss function $\cL_t$:

\begin{align}
    \mathcal{L}_t\left(\hat{\btheta}_{t,E}\right)-\mathcal{L}_t\left(\btheta^*\right)&=\left.\nabla \mathcal{L}_t\right({\btheta^*}) ^{\top}\left(\hat{\btheta}_{t,E}-\btheta^*\right)+\left(\hat{\btheta}_{t,E}-\btheta^*\right)^{\top} \widetilde{\mathbf{G}}_{\mathbf{t}}\left( \btheta^*,\hat{\btheta}_{t,E}\right)\left(\hat{\btheta}_{t,E}-\btheta^*\right)\\
    &\overset{(a)}{\ge} \left.\nabla \mathcal{L}_t\right({\btheta^*}) ^{\top}\left(\hat{\btheta}_{t,E}-\btheta^*\right)+(2+D)^{-1}\norm{\hat{\btheta}_{t,E}-\btheta^*}_{\bH_t(\btheta^*)}^2\label{apdx_eq:rewrite_1}
\end{align}
where inequality (a) is due to the following inequalities

\begin{align}\label{apdx_eq:link_H_tilde_G}
    \tilde{\bG}_t(\btheta^*,\hat{\btheta}_{t,E})&= \sum_{s=1}^{t-1}\sum_{i\in \tau_s} \tilde{\alpha}\left(\bphi_s(i), \btheta^*, \hat{\btheta}_{t,E}\right) \bphi_s(i) \bphi_s(i)^{\top}+\lambda_t \mathbf{I}_{\mathbf{d}} \\
& \overset{(a)}{\succeq} \sum_{s=1}^{t-1}\sum_{i\in \tau_s}\left(2+D\right)^{-1} \dot{\ell}\left(\bphi_s(i)^{\top} \btheta^*\right) \bphi_s(i) \bphi_s(i)^{\top}+\lambda_t \mathbf{I}_{\mathbf{d}} \\
&=(2+D)^{-1} \bH_t(\btheta^*)
\end{align}
where inequality (a) is due to \cref{apdx_eq:scc2} in \cref{apdx_lem:scc}.

Therefore, we rewrite \cref{apdx_eq:rewrite_1}:
\begin{align}\label{apdx_eq:2+D}
    \norm{\hat{\btheta}_{t,E}-\btheta^*}_{\bH_t(\btheta^*)}^2 \le (2+D)\nabla \mathcal{L}_t({\btheta^*}) ^{\top}\left(\btheta^*-\hat{\btheta}_{t,E}\right) + (2+D)\left(\mathcal{L}_t\left(\hat{\btheta}_{t,E}\right)-\mathcal{L}_t\left(\btheta^*\right)\right)
\end{align}

Since $\hat{\btheta}_{t,E}=\argmin_{\btheta \in \cQ} \mathcal{L}_t(\theta)$ and under the event that $\btheta^* \in \cQ$, we have:

\begin{align}
    \mathcal{L}_t\left(\hat{\btheta}_{t,E}\right)-\mathcal{L}_t\left(\btheta^*\right)\le 0
\end{align}

Combining the above inequality and \cref{apdx_eq:2+D}, we have:
\begin{align}
    \norm{\hat{\btheta}_{t,E}-\btheta^*}_{\bH_t(\btheta^*)}^2 &\le (2+D)\nabla \mathcal{L}_t({\btheta^*}) ^{\top}\left(\btheta^*-\hat{\btheta}_{t,E}\right)\\
    &\le (2+D)\norm{\nabla \mathcal{L}_t({\btheta^*})} _{\bH_t^{-1}(\btheta^*)} \norm{\hat{\btheta}_{t,E}-\btheta^*}_{\bH_t(\btheta^*)}
\end{align}
which can be rewritten as:
\begin{align}
    \norm{\hat{\btheta}_{t,E}-\btheta^*}_{\bH_t(\btheta^*)}&\le (2+D)\norm{\nabla \mathcal{L}_t({\btheta^*})} _{\bH_t^{-1}(\btheta^*)}\\
    &=(2+D)\norm{\bg_t(\btheta^*)-\sum_{s=1}^{t-1}\sum_{i\in \tau_s} X_{s,i} \bphi_s(i)} _{\bH_t^{-1}(\btheta^*)}\\
    &=(2+D)\norm{\bg_t(\btheta^*)-\bg_t(\hat{\btheta}_t)} _{\bH_t^{-1}(\btheta^*)}\\
    &\le (2+D)\gamma_t(\delta)
\end{align}

Now we can conclude the lemma by the fact that $\btheta^* \in \cQ$ and $D \le 1$ due to \cref{apdx_eq:D<1} with probability at least $1-\delta$. 

\end{proof}

\subsection{Optimism and Exploration Bonus}
\subsubsection{Proof for the Variance-Agnostic Exploration Bonus in Section \ref{sec:vag_conf_bonus} }\label{apdx_sec:var_ag_bonus}
\begin{proof}[Proof of \cref{lem:var_ag_bonus}]
Using the first-order Taylor expansion in \cref{apdx_eq:1_taylor_l}, we have:
\begin{align}
    \abs{\ell\left({\btheta^*}^{\top}\bphi_t(i)\right)-\ell\left(\hat{\btheta}_t^{\top}\bphi_t(i)\right)}&=\alpha\left(\bphi_t(i), \hat{\btheta}_t,  \btheta^*\right)\abs{\bphi_t(i)^{\top} \left( \btheta^* - \hat{\btheta}_t\ \right)}\\
    &\overset{(a)}{\le} \frac{1}{4}\abs{\bphi_t(i)^{\top} \left( \btheta^* - \hat{\btheta}_t\ \right)}\\
    &\overset{(b)}{\le} \frac{1}{4}\norm{\bphi_t(i)} _{\bV_t^{-1}} \norm{\btheta^* - \hat{\btheta}_t} _{\bV_t}\\
    &\overset{(c)}{\le} \frac{1}{4}\beta_t(\delta)  \norm{\bphi_t(i)} _{\bV_t^{-1}}
\end{align}
where inequality (a) is due to $\cref{apdx_eq:alpha}$ with the fact that $\dot{\ell}(x)\le 1/4$ for any $x\in \R$, inequality (b) is due to Cauchy-Schwarz inequality, and inequality (c) is due to event $\btheta^* \in \cB_t(\delta)$ in \cref{lem:var_ag_conf}.

This concludes the \cref{lem:var_ag_bonus}.
\end{proof}

\subsubsection{Proof for the Variance-Adaptive Exploration Bonus in Section \ref{sec:vad_conf_bouns} }\label{apdx_sec:proof_vad_bonus}
\begin{proof}[Proof of \cref{lem:var_ad_bonus}]
Using the second-order Taylor expansion in \cref{apdx_eq:2_taylor_l}, we have:
\begin{align}
    &\abs{\ell\left({\btheta^*}^{\top}\bphi_t(i)\right)-\ell\left(\hat{\btheta}_{t,H}^{\top}\bphi_t(i)\right)}\label{apdx_eq:var_ad_bonus_start}\\
    &=\left|\dot{\ell}\left(\bphi_t(i)^{\top}\hat{\btheta}_{t,H}\right) \bphi_t(i)^{\top}\left(\btheta^*-\hat{\btheta}_{t,H}\right)\right.\notag\\
    &+\left.\left(\bphi_t(i)^{\top}\left(\btheta^*-\hat{\btheta}_{t,H}\right)\right)^2\int_{v=0}^1 (1-v)\ddot{\ell}\left(v \bphi_t(i)^{\top} \btheta^*+(1-v) \bphi_t(i)^{\top} \hat{\btheta}_{t,H}\right) d v\right|\\
    &\le \dot{\ell}\left(\bphi_t(i)^{\top}\hat{\btheta}_{t,H}\right) \abs{\bphi_t(i)^{\top}\left(\btheta^*-\hat{\btheta}_{t,H}\right)}\notag\\
    &+\left(\bphi_t(i)^{\top}\left(\btheta^*-\hat{\btheta}_{t,H}\right)\right)^2\int_{v=0}^1 (1-v)\abs{\ddot{\ell}\left(v \bphi_t(i)^{\top} \btheta^*+(1-v) \bphi_t(i)^{\top} \hat{\btheta}_{t,H}\right)} d v\\
    &\overset{(a)}{\le} \dot{\ell}\left(\bphi_t(i)^{\top}\hat{\btheta}_{t,H}\right) \abs{\bphi_t(i)^{\top}\left(\btheta^*-\hat{\btheta}_{t,H}\right)}+\frac{1}{8}\left(\bphi_t(i)^{\top}\left(\btheta^*-\hat{\btheta}_{t,H}\right)\right)^2\label{apdx_eq:var_ad_bonus_mid}\\
    &\overset{(b)}{\le}\dot{\ell}\left(\bphi_t(i)^{\top}\hat{\btheta}_{t,H}\right) \norm{\bphi_t(i)} _{\bH_t^{-1}(\hat{\btheta}_{t,H})}\norm{\btheta^*-\hat{\btheta}_{t,H}} _{\bH_t(\hat{\btheta}_{t,H})}+\frac{1}{8}\norm{\bphi_t(i)} ^2  _{\bH_t^{-1}(\hat{\btheta}_{t,H})}\norm{\btheta^*-\hat{\btheta}_{t,H}} ^2 _{\bH_t(\hat{\btheta}_{t,H})}\\
    &\overset{(c)}\le (2+4L)\gamma_t(\delta)\dot{\ell}\left(\bphi_t(i)^{\top}\hat{\btheta}_{t,H}\right) \norm{\bphi_t(i)} _{\bH_t^{-1}(\hat{\btheta}_{t,H})}+\frac{1}{2} (1+2L)^2 \gamma_t^2(\delta) \norm{\bphi_t(i)} ^2  _{\bH_t^{-1}(\hat{\btheta}_{t,H})}\\
    &\overset{(d)}\le (2+4L)\gamma_t(\delta)\dot{\ell}\left(\bphi_t(i)^{\top}\hat{\btheta}_{t,H}\right) \norm{\bphi_t(i)} _{\bH_t^{-1}(\hat{\btheta}_{t,H})}+\frac{1}{2} (1+2L)^2 \kappa \gamma_t^2(\delta) \norm{\bphi_t(i)} ^2  _{\bV_t^{-1}}
\end{align}
where inequality (a) is due to the self-concordance property that $\abs{\ddot{\ell}(x)}\le \dot{\ell}(x)\le 1/4$ for any $x\in \R$, inequality (b) is due to Cauchy-Schwarz inequality,  inequality (c) is due to the event $\btheta^* \in \cC_t(\delta)$ defined in \cref{lem:var_ad_conf}, and inequality (d) is due to $\cH_{t}(\btheta)\ge \kappa^{-1} \bV_t$
for any $\btheta \in \R^d$.
\end{proof}

This concludes the \cref{lem:var_ad_bonus}.

\subsubsection{Proof for the Variance-Adaptive Exploration Bonus after the Burn-in Stage in Section \ref{sec:eva_conf_bonus} }\label{apdx_sec:proof_eva_bonus}
\begin{proof}[proof of \cref{lem:eva_bonus}]

By \cref{apdx_lem:scc}, we have:

\begin{align}\label{apdx_eq:s_lower_than_hat}
      \dot{\ell}\left( {\btheta^*}^{\top} \bphi_t(i) \right) &\le \exp \left(\abs{\left(  {\btheta^*} - \hat{\btheta}_{t,E}\right)^{\top} \bphi_t(i) } \right) \dot{\ell}\left( {\hat{\btheta}_{t,E}}^{\top} \bphi_t(i) \right)\\
     &\overset{(a)}{\le} e \dot{\ell}\left( {\hat{\btheta}_{t,E}}^{\top} \bphi_t(i) \right)
\end{align}
where inequality (a) is due to $\btheta^*,\hat{\btheta}_{t,E} \in \cQ$ and $\text{diam}(\cQ)\le 1$ which holds with probability at least $1-\delta$ by \cref{lem:Q_region} and the definition of $\hat{\btheta}_{t,E}$ in \cref{line:constrained_opt}.

Similarly, we have

\begin{align}\label{apdx_eq:hat_lower_than_s}
    \dot{\ell}\left( {\hat{\btheta}_{t,E}}^{\top} \bphi_t(i) \right) &\le \exp \left(\abs{\left(  {\btheta^*} - \hat{\btheta}_{t,E}\right)^{\top} \bphi_t(i) } \right) \dot{\ell}\left( {\btheta^*}^{\top} \bphi_t(i) \right)\\
     &\le e \dot{\ell}\left( {\btheta^*}^{\top} \bphi_t(i) \right)
\end{align}

Therefore, by the definition of $\bH_t(\btheta)$, we have: 
\begin{align}\label{apdx_eq:link_H*_H}
    e^{-1}\bH_t(\hat{\btheta}_{t,E}) \le \bH_t({\btheta^*})\le e\bH_t(\hat{\btheta}_{t,E})
\end{align}

Following a similar argument from \cref{apdx_eq:var_ad_bonus_start} to \cref{apdx_eq:var_ad_bonus_mid}, we have:
\begin{align}
    &\abs{\ell\left({\btheta^*}^{\top}\bphi_t(i)\right)-\ell\left(\hat{\btheta}_{t,E}^{\top}\bphi_t(i)\right)}\\
    &\le \dot{\ell}\left(\bphi_t(i)^{\top}\hat{\btheta}_{t,E}\right) \abs{\bphi_t(i)^{\top}\left(\btheta^*-\hat{\btheta}_{t,E}\right)}+\frac{1}{8}\left(\bphi_t(i)^{\top}\left(\btheta^*-\hat{\btheta}_{t,E}\right)\right)^2\\
    &\overset{(a)}{\le}\dot{\ell}\left(\bphi_t(i)^{\top}\hat{\btheta}_{t,E}\right) \norm{\bphi_t(i)} _{\bH_t^{-1}(\btheta^*)}\norm{\btheta^*-\hat{\btheta}_{t,E}} _{\bH_t(\btheta^*)}+\frac{1}{8}\norm{\bphi_t(i)} ^2  _{\bH_t^{-1}(\btheta^*)}\norm{\btheta^*-\hat{\btheta}_{t,E}} ^2 _{\bH_t(\btheta^*)}\\
    & \overset{(b)}{\le}\sqrt{e}\dot{\ell}\left(\bphi_t(i)^{\top}\hat{\btheta}_{t,E}\right) \norm{\bphi_t(i)} _{\bH_t^{-1}(\hat{\btheta}_{t,E})}\norm{\btheta^*-\hat{\btheta}_{t,E}} _{\bH_t(\btheta^*)}+\frac{1}{8}\kappa\norm{\bphi_t(i)} ^2  _{\bV_t^{-1}}\norm{\btheta^*-\hat{\btheta}_{t,E}} ^2 _{\bH_t(\btheta^*)}\\
    &\overset{(c)}{\le} 3\sqrt{e}\gamma_t(\delta)\dot{\ell}\left(\bphi_t(i)^{\top}\hat{\btheta}_{t,E}\right) \norm{\bphi_t(i)} _{\bH_t^{-1}(\hat{\btheta}_{t,E})}+\frac{9}{8}\kappa\gamma_t^2(\delta)\norm{\bphi_t(i)} ^2  _{\bV_t^{-1}}
\end{align}
    where inequality (a) is due to Cauchy-Schwarz inequality, inequality (b) is due to $e\bH_t^{-1}(\hat{\btheta}_{t,E}) \ge \bH_t^{-1}({\btheta^*})$ by \cref{apdx_eq:link_H*_H}, and inequality (c) is due to \cref{lem:eva_conf}.
\end{proof}

\section{Regret upper bound}\label{apdx_sec:regret_upper_bounds}

\subsection{Regret Bound for Alg. \ref{alg:CLogUCB} under 1-norm TPM Smoothness Condition}\label{apdx_sec:regret_ag}

\begin{proof}[Proof of \cref{thm:var_ag_thm}]
Recall that the historical data $\cH_t=(\bphi_s, S_s, \tau_s, (X_{s,i})_{i\in \tau_s})_{s<t} \bigcup \bphi_t$. 
Let $\E_t[\cdot]\defeq\E[\cdot | \cH_t]$.
Let $\bar{\bmu}_t\defeq (\bar{\mu}_{t,1},...,\bar{\mu}_{t,m})$, where $\bar{\mu}_{t,i}$ is the UCB value of \cref{line:clogb_ucb} in \cref{alg:CLogUCB}.

Under the event $\left\{\forall t\ge 1, \btheta^* \in \mathcal{B}_t(\delta)\right\}$, we have:

\begin{align}
\text{Reg}(T)&\overset{(a)}{=}\E\left[\sum_{t \in [T]}\E_t[\alpha\cdot r(S^*_t;\bmu_t)-r(S_t;\bmu_t)]\right]\label{apdx_eq:1-norm_reg_start}\\
&\overset{(b)}{\le}\E\left[\sum_{t \in [T]}\E_t[\alpha\cdot r(S^*_t;\bar{\bmu}_t)-r(S_t;\bmu_t)]\right]\notag\\
&\overset{(c)}{\le} \E\left[\sum_{t \in [T]}\E_t[r(S_t;\bar{\bmu}_t)-r(S_t;\bmu_t)]\right]\notag\\
    &\overset{(d)}{\le}\E\left[\sum_{t \in [T]}\E_t\left[\sum_{i \in [m]}B_1p_i^{\bmu_t,S_t}(\bar{\mu}_{t,i}-\mu_{t,i})\right]\right]\notag\\
    &\overset{(e)}{=} \E\left[\sum_{t \in [T]}\E_t\left[\sum_{i \in \tau_t}B_1(\bar{\mu}_{t,i}-\mu_{t,i})\right]\right]\label{apdx_eq:1-norm_reg_mid}\\
&\overset{(f)}{\le} \E\left[\sum_{t \in [T]}\E_t\left[\sum_{i \in \tau_t}\frac{1}{2}B_1\beta_t(\delta)\norm{\bphi_t(i)}_{\bV_t^{-1}}\right]\right]\notag\\
&\overset{(g)}{=}  \frac{1}{2}B_1\beta_T(\delta)\E\left[\sum_{t \in [T]}\sum_{i \in \tau_t}\norm{\bphi_t(i)}_{\bV_{t}^{-1}}\right]\notag\\
&\overset{(h)}{\le}  \frac{1}{2}B_1\beta_T \E\left[\sqrt{KT\sum_{t \in [T]}\sum_{i \in \tau_t}\norm{\bphi_t(i)}^2_{\bV_{t}^{-1}}}\right] \notag\\
&\overset{(i)}{\le}\frac{1}{2}B_1\beta_T(\delta) \sqrt{2KT \log \left( T + \lbdd \right)}\notag\\
&=\frac{1}{2}B_1\left(L^2+4L+\frac{19}{4}\right) \sqrt{\kappa \lbd} \sqrt{4KT \lbdd}\notag\\
&\le B_1\left(L^2+4L+\frac{19}{4}\right) \sqrt{\kappa KT}  \lbdd.
   \end{align}
 where inequality (a) is due to the regret definition and the tower rule of expectation, 
 inequality (b) is due to monotonicity condition (\cref{cond:mono}) and the fact that $\mu_{t,i}\le\bar{\mu}_{t,i}$ by \Cref{lem:var_ag_bonus}, 
 inequality (c) is due to $\alpha r(S^*;\bar{\bmu}_t)\le\alpha \max_{S\in \cS}r(S;\bar{\bmu}_t)\le r(S_t;\bar{\bmu}_t)$ by the definition of the $\alpha$-approximation oracle, 
 inequality (d) is due to $r(S_t;\bar{\bmu}_t)-r(S_t;\bmu_t)\le\sum_{i \in [m]}B_1p_i^{\bmu_t,S_t}(\bar{\mu}_{t,i}-\mu_{t,i})$ by 1-norm TPM Condition (\cref{cond:TPM}), 
 inequality(e) is due to $\E_t[i\in \tau_t]=p_{i}^{\bmu_t, \bS_t}$ and the fact that $S_t, \bmu_t, \bar{\bmu}_t$ are $\cH_t$ measurable, 
 inequality (f) is due to $\bar{\mu}_{t,i}-\bar{\mu}_{t,i}\le \frac{1}{2}\beta_t(\delta)\norm{\bphi_t(i)}_{\bV_t^{-1}}$ by \cref{lem:var_ag_bonus}, 
 inequality (g) is due to the tower rule of expectation,
 inequality (h) by Cauchy Schwarz inequality, 
 and (i) by the elliptical potential lemma (\Cref{apdx_lem:elliptical_potential}). 

Let $r_{\max}=\max_{t\in[T]} \alpha r(S_t^*;\bmu_t)$. Now we set $\delta=\frac{1}{T}$ and consider the regret caused by the failure of event $\left\{\forall t\ge 1, \btheta^* \in \mathcal{B}_t(\delta)\right\}$, we have:
\begin{align}
    \text{Reg}(T)&\le r_{\max} T \cdot \frac{1}{T} + B_1\left(L^2+4L+\frac{19}{4}\right) \sqrt{\kappa KT}  \lbddt\\
    &= O\left(B_1 d\sqrt{\kappa K T}\log(KT)\right)
\end{align}

\end{proof}

\subsection{Regret Bound for Alg. \ref{alg:VA_CLogB} under 1-Norm TPM Smoothness Condition and TPVM Smoothness Condition}\label{apdx_sec:proof_thm_ad}
\begin{proof}[Proof of \cref{thm:var_ad_thm1}]
Recall that the historical data $\cH_t=(\bphi_s, S_s, \tau_s, (X_{s,i})_{i\in \tau_s})_{s<t} \bigcup \bphi_t$. 
Let $\E_t[\cdot]\defeq\E[\cdot | \cH_t]$.
Let $\bar{\bmu}_t\defeq (\bar{\mu}_{t,1},...,\bar{\mu}_{t,m})$, where $\bar{\mu}_{t,i}$ is the UCB value of \cref{line:va_clogb_ucb} in \cref{alg:VA_CLogB}.

Suppose the event $\{\forall t \geq 1, \btheta^* \in \cA_t(\delta)\} \leq \gamma_t(\delta)\}$ and the event $\{\forall t \geq 1, \btheta^* \in \cC_t(\delta)\}$ hold.

By the first-order Taylor expansion, for any $\btheta_1, \btheta_2 \in \cA_t(\delta)$ and any vector $\bx=\bphi_t(i)$ for all $t\in [T], i\in [m]$, we have:

\begin{align}
    \dot{\ell} \left( \btheta_2^{\top} \bx \right) &=  \dot{\ell} \left( \btheta_1^{\top} \bx\right) + \bx^{\top} (\btheta_2-\btheta_1) \int_{v=0}^1 \ddot{\ell}\left(v \bx^{\top} \btheta_2+(1-v) \bx^{\top} \btheta_1\right) d v\\
    &\le \dot{\ell} \left( \btheta_1^{\top} \bx\right) + \abs{\bx^{\top} (\btheta_2-\btheta_1)} \int_{v=0}^1 \abs{\ddot{\ell}\left(v \bx^{\top} \btheta_2+(1-v) \bx^{\top} \btheta_1\right)} d v\\
    &\overset{(a)}{\le} \dot{\ell} \left( \btheta_1^{\top} \bx\right) + \abs{\bx^{\top} (\btheta_2-\btheta_1)} \int_{v=0}^1 \dot{\ell}\left(v \bx^{\top} \btheta_2+(1-v) \bx^{\top} \btheta_1\right) d v\\
    &\le \dot{\ell} \left( \btheta_1^{\top} \bx\right) + \frac{1}{4}\abs{\bx^{\top} (\btheta_2-\btheta_1)}\\
    &\le \dot{\ell} \left( \btheta_1^{\top} \bx\right) + \frac{1}{4} \norm{\bx}_{\bG^{-1}_t(\btheta_1, \btheta_2)} \norm{\btheta_2-\btheta_1}_{\bG_t(\btheta_1, \btheta_2)}\\
    &\overset{(b)}{=} \dot{\ell} \left( \btheta_1^{\top} \bx\right) + \frac{1}{4} \norm{\bx}_{\bG^{-1}_t(\btheta_1, \btheta_2)} \norm{\bg_t(\btheta_2)-\bg_t(\btheta_1)}_{\bG^{-1}_t(\btheta_1, \btheta_2)}\\
    &\le \dot{\ell} \left( \btheta_1^{\top} \bx\right) + \frac{1}{4} \norm{\bx}_{\bG^{-1}_t(\btheta_1, \btheta_2)} \left(\norm{\bg_t(\btheta_2)-\bg_t(\hat{\btheta}_t)}_{\bG^{-1}_t(\btheta_1, \btheta_2)}+\norm{\bg_t(\btheta_1)-\bg_t(\hat{\btheta}_t)}_{\bG^{-1}_t(\btheta_1, \btheta_2)}\right)\\
    &\overset{(c)}{\le} \dot{\ell} \left( \btheta_1^{\top} \bx\right) + \frac{\sqrt{1+2L}}{4} \norm{\bx}_{\bG^{-1}_t(\btheta_1, \btheta_2)} \left(\norm{\bg_t(\btheta_2)-\bg_t(\hat{\btheta}_t)}_{\bH^{-1}_t(\btheta_2)}+\norm{\bg_t(\btheta_1)-\bg_t(\hat{\btheta}_t)}_{\bH^{-1}_t(\btheta_1)}\right)\\
    &\overset{(d)}{\le} \dot{\ell} \left( \btheta_1^{\top} \bx\right) + \frac{\sqrt{1+2L}}{2}\gamma_t(\delta)\norm{\bx}_{\bG^{-1}_t(\btheta_1, \btheta_2)}\\
    &\overset{(e)}{\le} \dot{\ell} \left( \btheta_1^{\top} \bx\right) + \frac{\sqrt{1+2L}}{2}\sqrt{\kappa}\gamma_t(\delta)\norm{\bx}_{\bV^{-1}_t}
\end{align}
where inequality (a) is due to the self-concordant property $\abs{\ddot{\ell}(x)}\le \dot{\ell}(x)$ for any $x \in \R$,
inequality (b) is due to \cref{apdx_eq:link_G_g}, 
inequality (c) is due to the same derivation of \cref{apdx_eq:link_H_G}, 
inequality (d) is due to the $\btheta_1, \btheta_2 \in \cA_t(\delta)$,
inequality (e) is due to $\bG_t(\btheta_1, \btheta_2)\ge \kappa^{-1} \bV_t$.

Define $\tilde{\btheta}_{t,i}\defeq \arg\min_{\btheta \in \mathcal{A}_t(\delta)}\dot{\ell}(\bphi_t(i)^{\top}\btheta)$. By the above inequality and the fact that $\tilde{\btheta}_{t,i}, \hat{\btheta}_{t,H}\in \cA_{t}(\delta)$, we have: 
\begin{align}\label{apdx_eq:link_dot_l}
    \dot{\ell} \left( \hat{\btheta}_{t,H}^{\top}\bphi_t(i)\right) {\le} \dot{\ell} \left( \tilde{\btheta}_{t,i}^{\top}\bphi_t(i)\right) + \frac{\sqrt{1+2L}}{2}\sqrt{\kappa}\gamma_t(\delta)\norm{\bphi_t(i)}_{\bV^{-1}_t}
\end{align}


Define $\bL_t\defeq\sum_{s=1}^{t-1}\sum_{i\in\tau_s}\dot{\ell}\left(\bphi_s(i)^{\top}\tilde{\btheta}_{s,i}\right)\bphi_s(i)\bphi_s(i)^{\top}+\lambda_t \mathbf{I}_d$. Since the region $\cQ_t$ in \cref{eq:Q_region} is equivalent to
\begin{align}
    \cQ_t = \left\{\btheta \in \Theta: \dot{\ell}\left(\btheta^{\top}\bphi_s(i)\right) \ge \min_{\btheta' \in \cA_s(\delta)}\dot{\ell}\left({\btheta'}^{\top}\bphi_s(i)\right) \text{ for all } i \in \tau_s, s\in [t]\right\},
\end{align} and using the fact that $\btheta^*, \hat{\btheta}_{t,H} \in \cQ_t$, we have for any $s\in [t], i\in \tau_s$:
\begin{align}
\dot{\ell}\left( {\btheta^*}^{\top} \bphi_{s}(i) \right) &\ge \dot{\ell} \left( \tilde{\btheta}_{s,i} ^{\top}\bphi_{s}(i)\right)\label{apdx_eq:lb_theta_s}\\
    \dot{\ell}\left( \hat{\btheta}_{t,H}^{\top} \bphi_{s}(i) \right) &\ge \dot{\ell} \left( \tilde{\btheta}_{s,i} ^{\top}\bphi_{s}(i)\right)\label{apdx_eq:lb_theta_hat}
\end{align}

Therefore, we have:
\begin{align}\label{apdx_eq:link_H_L}
    \cH_t(\hat{\btheta}_{t,H}) &=\sum_{s=1}^{t-1}\sum_{i\in\tau_s} \dot{\ell}\left(\bphi_s(i)^{\top} \hat{\btheta}_{t,H}\right) \bphi_s(i) \bphi_s(i)^{\top}+\lambda_t \mathbf{I}_d\\
    &\ge \sum_{s=1}^{t-1}\sum_{i\in\tau_s} \dot{\ell}\left(\bphi_s(i)^{\top} \tilde{\btheta}_{s,i}\right) \bphi_s(i) \bphi_s(i)^{\top}+\lambda_t \mathbf{I}_d = \bL_t
\end{align}

Now we use the same derivation from \cref{apdx_eq:1-norm_reg_start} to \cref{apdx_eq:1-norm_reg_mid}:

\begin{align}
    &\text{Reg}(T)\le \E\left[\sum_{t \in [T]}\E_t\left[\sum_{i \in \tau_t}B_1(\bar{\mu}_{t,i}-\mu_{t,i})\right]\right]\overset{(a)}{\le}\E\left[\sum_{t \in [T]}\E_t\left[\sum_{i \in \tau_t}B_1(2\rho_{t,H}(i))\right]\right]\notag\\
    &=\E\left[\sum_{t \in [T]}\E_t\left[\sum_{i \in \tau_t}B_1\left((4+8L) \dot{\ell}\left(\hat{\btheta}_{t,H}^{\top}\bphi_t(i)\right)\|\bphi_t(i)\|_{\mathbf{H}_t^{-1}(\hat{\btheta}_{t,H})} \gamma_t(\delta)+(1+2L)\kappa \gamma_t^2(\delta)\|\bphi_t(i)\|_{\mathbf{V}_t^{-1}}^2\right)\right]\right]\\
    &\overset{(b)}{\le} \E\left[\sum_{t \in [T]}\E_t\left[\sum_{i \in \tau_t}B_1\left((4+8L) \dot{\ell}\left(\tilde{\btheta}_{t,i}^{\top}\bphi_t(i)\right)\|\bphi_t(i)\|_{\mathbf{H}_t^{-1}(\hat{\btheta}_{t,H})} \gamma_t(\delta)\right.\right.\right.\notag\\
    &\quad\quad+\left.\left.\left. 3(1+2L)^2\kappa \gamma_t^2(\delta)\|\bphi_t(i)\|_{\mathbf{V}_t^{-1}}^2\right)\right]\right]\\
     &\overset{(c)}{\le} \E\left[\sum_{t \in [T]}\E_t\left[\sum_{i \in \tau_t}B_1\left((4+8L) \dot{\ell}\left(\tilde{\btheta}_{t,i}^{\top}\bphi_t(i)\right)\|\bphi_t(i)\|_{\mathbf{L}_t^{-1}} \gamma_t(\delta) +3(1+2L)^2\kappa \gamma_t^2(\delta)\|\bphi_t(i)\|_{\mathbf{V}_t^{-1}}^2\right)\right]\right]\\
      &\overset{(d)}{=} \E\left[\sum_{t \in [T]}\sum_{i \in \tau_t}B_1\left((4+8L) \dot{\ell}\left(\tilde{\btheta}_{t,i}^{\top}\bphi_t(i)\right)\|\bphi_t(i)\|_{\mathbf{L}_t^{-1}} \gamma_t(\delta) +3(1+2L)^2\kappa \gamma_t^2(\delta)\|\bphi_t(i)\|_{\mathbf{V}_t^{-1}}^2\right)\right] \label{apdx_eq:not_tight_1} \\
&\overset{}{\le}  (4+8L)B_1 \gamma_T(\delta)\E\left[\sum_{t \in [T]}\sum_{i \in \tau_t} \sqrt{\dot{\ell}\left(\tilde{\btheta}_{t,i}^{\top} \bphi_t(i)\right)} \|\bphi_t(i)\|_{\mathbf{L}_t^{-1}}\right]\notag\\
&+3(1+2L)^2B_1 \kappa\gamma^2_T(\delta)\E\left[\sum_{t \in [T]}\sum_{i \in \tau_t}\|\bphi_t(i)\|^2_{\mathbf{V}_t^{-1}}\right] \label{apdx_eq:not_tight_2} \\
&\overset{(e)}{=}  (4+8L)B_1 \gamma_T(\delta)\E\left[\sum_{t \in [T]}\sum_{i \in \tau_t}\|\tilde{\bphi}_t(i)\|_{\mathbf{L}_t^{-1}}\right]+3(1+2L)^2B_1 \kappa\gamma^2_T(\delta)\E\left[\sum_{t \in [T]}\sum_{i \in \tau_t}\|\bphi_t(i)\|^2_{\mathbf{V}_t^{-1}}\right]\\
&\le  (4+8L)B_1 \gamma_T(\delta)\sqrt{KT}\E\left[\sqrt{\sum_{t \in [T]}\sum_{i \in \tau_t}\|\tilde{\bphi}_t(i)\|^2_{\mathbf{L}_t^{-1}}}\right]+3(1+2L)^2B_1 \kappa\gamma^2_T(\delta)\E\left[\sum_{t \in [T]}\sum_{i \in \tau_t}\|\bphi_t(i)\|^2_{\mathbf{V}_t^{-1}}\right]\\
&\overset{(g)}{\le}  (4+8L)(2L+3)B_1 \sqrt{KT} \lbdd \notag\\
&+3(1+2L)^2(2L+3)^2 B_1\kappa  \left(\lbdd \right)^2 \\ 
&= O\left(B_1 d\sqrt{KT}\log (KT)+ B_1 \kappa d^2\log ^2(KT)\right).
\end{align}
where inequality (a) is due to variance-adaptive exploration bonus lemma (\cref{lem:var_ag_bonus}), 
inequality (b) is due to \cref{apdx_eq:link_dot_l},
inequality (c) is due to \cref{apdx_eq:link_H_L},
inequality (d) is due to the tower law of expectation,
inequality (e) is by defining $\tilde{\bphi}_t(i)\defeq\sqrt{\dot{\ell}\left(\tilde{\btheta}_{t,i}^\top\bphi_{t}(i)\right)}\bphi_t(i)$,
inequality (g) is by the elliptical potential lemma (\cref{apdx_lem:elliptical_potential})

Let $r_{\max}=\max_{t\in[T]} \alpha r(S_t^*;\bmu_t)$. Now we set $\delta=\frac{1}{2T}$ and consider the regret caused by the failure of the event $\{\forall t \geq 1, \btheta^* \in \cA_t(\delta)\} \leq \gamma_t(\delta)\}$ and the event $\{\forall t \geq 1, \btheta^* \in \cC_t(\delta)\}$. We have:
\begin{align}
    \text{Reg}(T)&\le r_{\max}T \cdot \frac{2}{2T} +  (4+8L)(2L+3)B_1 \sqrt{KT} \lbdd \notag\\
&+3(1+2L)^2(2L+3)^2 B_1\kappa  \left(\lbdd \right)^2 \\
    &= O\left(B_1 d\sqrt{KT}\log (KT)+ B_1 \kappa d^2\log ^2(KT)\right).
\end{align}

\end{proof}

\begin{proof}[Proof of \cref{thm:var_ad_thm2}]
Suppose the event $\{\forall t \geq 1, \btheta^* \in \cA_t(\delta)\} \leq \gamma_t(\delta)\}$ and the event $\{\forall t \geq 1, \btheta^* \in \cC_t(\delta)\}$ hold, 
then we have:

\begin{align}\label{apdx_eq:TPVM_start}
&\text{Reg}(T)\overset{(a)}{=} \E\left[\sum_{t=1}^T  \alpha r(S_t^*; \bmu_t)- r(S_t; \bmu_t)\right]\\
&\overset{(b)}{\le} \E\left[\sum_{t=1}^T \alpha r(S^*_t; \bar{\bmu}_t)- r(S_t; \bmu_t)\right]\\
&\overset{(c)}{\le} \E\left[\sum_{t=1}^T r(S_t; \bar{\bmu}_t)- r(S_t; \bmu_t)\right]
\end{align}
where inequality (a) is by definition, 
inequality (b) follows from Condition~\ref{cond:mono} and \Cref{lem:var_ad_bonus}, 
inequality (c) is due to $\alpha r(S^*;\bar{\bmu}_t)\le\alpha \max_{S\in \cS}r(S;\bar{\bmu}_t)\le r(S_t;\bar{\bmu}_t)$ by the definition of the $\alpha$-approximation oracle.

By \cref{lem:var_ad_bonus}, we upper bound the value of $\bar{\bmu}_{t}$:
\begin{align}
    \bar{\mu}_{t,i} &\le \mu_{t,i} + 2\rho_{t,H}\\
    &= \mu_{t,i} +(4+8L)\gamma_t(\delta)\dot{\ell}\left(\bphi_t(i)^{\top}\hat{\btheta}_{t,H}\right) \norm{\bphi_t(i)} _{\bH_t^{-1}(\hat{\btheta}_{t,H})}+ (1+2L)^2 \kappa \gamma_t^2(\delta) \norm{\bphi_t(i)} ^2  _{\bV_t^{-1}}\\
    &\overset{(a)}{\le} \mu'_{t,i} \defeq (4+8L)\gamma_t(\delta)\dot{\ell}\left(\bphi_t(i)^{\top}\tilde{\btheta}_{t,i}\right) \norm{\bphi_t(i)} _{\bH_t^{-1}(\hat{\btheta}_{t,H})}+ 3(1+2L)^2 \kappa \gamma_t^2(\delta) \norm{\bphi_t(i)} ^2  _{\bV_t^{-1}} 
\end{align}
where inequality (a) is due to \cref{apdx_eq:link_dot_l} by letting $\btheta_2= \hat{\btheta}_{t,H}, \btheta_1=\tilde{\btheta}_{t,i}$ which all belong to $\cA_t(\delta)$.

Let $\bmu'_t=(\mu'_{t,1},..., \mu'_{t,m})$ and by monotonicity,
\begin{align}
    \text{Reg}(T)&\le \E\left[\sum_{t=1}^T r(S_t; \bar{\bmu}_t)- r(S_t; \bmu_t)\right]\\
    &\le \E\left[\sum_{t=1}^T r(S_t; \bmu'_t)- r(S_t; \bmu_t)\right]\\
    &\overset{(a)}{\le} \E\left[\sum_{t=1}^T B_v \sqrt{\sum_{i \in [m]}(p_i^{\bmu_t,S_t})^{\lambda}\frac{(4+8L)^2\gamma^2_t(\delta)\dot{\ell}\left(\bphi_t(i)^{\top}\tilde{\btheta}_{t,i}\right)^2 \norm{\bphi_t(i)}^2 _{\bH_t^{-1}(\hat{\btheta}_{t,H})}}{(1-\mu_{t,i})\mu_{t,i}}}\right]\notag\\
    &+ \E\left[\sum_{t\in [T]}B_1\sum_{i\in [m]}  p_i^{\bmu_t,S_t} 3(1+2L)^2 \kappa \gamma_t^2(\delta) \norm{\bphi_t(i)} ^2  _{\bV_t^{-1}}\right]\\
    &\overset{(b)}{=} \underbrace{\E\left[\sum_{t=1}^T B_v(4+8L) \gamma_t(\delta) \sqrt{\sum_{i \in [m]}(p_i^{\bmu_t,S_t})^{\lambda} \frac{\dot{\ell}^2\left(\bphi_t(i)^{\top}\tilde{\btheta}_{t,i}\right)}{\dot{\ell}\left(\bphi_t(i)^{\top}\btheta^*\right)} \norm{\bphi_t(i)}^2 _{\bH_t^{-1}(\hat{\btheta}_{t,H})}}\right]}_{\text{term 1}}\notag\\
    &+ \underbrace{\E\left[\sum_{t\in [T]}B_1\sum_{i\in [m]} p_i^{\bmu_t,S_t} 3(1+2L)^2 \kappa \gamma_t^2(\delta) \norm{\bphi_t(i)} ^2  _{\bV_t^{-1}}\right]}_{\text{term 2}}\\
\end{align}
where inequality (a) is due to \cref{cond:TPVMm} by setting $\zeta_i=(4+8L)\gamma_t(\delta)\dot{\ell}\left(\bphi_t(i)^{\top}\tilde{\btheta}_{t,i}\right) \norm{\bphi_t(i)} _{\bH_t^{-1}(\hat{\btheta}_{t,H})}, \eta_i=3(1+2L)^2 \kappa \gamma_t^2(\delta) \norm{\bphi_t(i)} ^2  _{\bV_t^{-1}}$,
inequality (b) is due to $\dot{\ell}\left(\bphi_t(i)^{\top}\btheta^*\right)=\left(1- \ell\left(\bphi_t(i)^{\top}\btheta^*\right) \right)\ell\left(\bphi_t(i)^{\top}\btheta^*\right)  =(1-\mu_{t,i})\mu_{t,i}$.

For term 1, we have:
\begin{align}
    \text{term 1} &\overset{(a)}{\le}  B_v(4+8L) \gamma_T(\delta)\E\left[\sum_{t=1}^T  \sqrt{\sum_{i \in [m]}p_i^{\bmu_t,S_t} \frac{\dot{\ell}^2\left(\bphi_t(i)^{\top}\tilde{\btheta}_{t,i}\right)}{\dot{\ell}\left(\bphi_t(i)^{\top}\btheta^*\right)}  \norm{\bphi_t(i)}^2 _{\bH_t^{-1}(\hat{\btheta}_{t,H})}}\right]\\
    &\overset{(b)}{\le}  B_v(4+8L) \gamma_T(\delta)\sqrt{T}\E\left[\sqrt{\sum_{t=1}^T  \sum_{i \in [m]} p_i^{\bmu_t,S_t} \frac{\dot{\ell}^2 \left(\bphi_t(i)^{\top}\tilde{\btheta}_{t,i}\right)}{\dot{\ell}\left(\bphi_t(i)^{\top}\btheta^*\right)}  \norm{\bphi_t(i)}^2 _{\bH_t^{-1}(\hat{\btheta}_{t,H})}}\right]\\
    &\overset{(c)}{\le}  B_v(4+8L) \gamma_T(\delta)\sqrt{T}\sqrt{\E\left[\sum_{t=1}^T  \sum_{i \in [m]} p_i^{\bmu_t,S_t}  \frac{\dot{\ell}^2\left(\bphi_t(i)^{\top}\tilde{\btheta}_{t,i}\right)}{\dot{\ell}\left(\bphi_t(i)^{\top}\btheta^*\right)}  \norm{\bphi_t(i)}^2 _{\bH_t^{-1}(\hat{\btheta}_{t,H})}\right]}\\
    &\overset{(d)}{=}  B_v(4+8L) \gamma_T(\delta)\sqrt{T}\sqrt{\E\left[\sum_{t=1}^T  \sum_{i \in \tau_t}\frac{\dot{\ell}^2\left(\bphi_t(i)^{\top}\tilde{\btheta}_{t,i}\right)}{\dot{\ell}\left(\bphi_t(i)^{\top}\btheta^*\right)}  \norm{\bphi_t(i)}^2 _{\bH_t^{-1}(\hat{\btheta}_{t,H})}\right]}
\end{align}
where inequality (a) is due to $\lambda \ge 1$, inequality (b) is due to Cauchy-Schwarz inequality, 
inequality (c) is due to Jensen's inequality,
equality (d) is due to $\E_t[i\in \tau_t]=p_{i}^{\bmu_t, \bS_t}$ and the fact that $S_t, \bH_t, \tilde{\btheta}_{t,i}, \hat{\btheta}_{t,H}$ are $\cH_t$ measurable.

By \cref{apdx_eq:lb_theta_s} and \cref{apdx_eq:link_H_L}, we have ${{\dot{\ell}\left(\bphi_t(i)^{\top} \btheta^*\right)} \ge \dot{\ell} \left(\bphi_t(i)^{\top}\tilde{\btheta}_{t,i}\right)} $ and  $\bH_t(\hat{\btheta}_{t,H})\ge \bL_t$. Therefore:

\begin{align}
    \text{term 1} &\le B_v(4+8L) \gamma_T(\delta)\sqrt{T}\sqrt{\E\left[\sum_{t=1}^T  \sum_{i \in \tau_t}\dot{\ell}\left(\bphi_t(i)^{\top}\tilde{\btheta}_{t,i}\right)  \norm{\bphi_t(i)}^2 _{\bL_t^{-1}}\right]}\\
    &\overset{(a)}{=} B_v(4+8L) \gamma_T(\delta)\sqrt{T}\sqrt{\E\left[\sum_{t=1}^T  \sum_{i \in \tau_t}  \norm{\tilde{\bphi}_t(i)}^2 _{\bL_t}\right]}\\
    &\overset{(b)}{\le} (4+8L)(2L+3) B_v\sqrt{T}\lbdd. 
\end{align}
where inequality (a) is due to the definition of $\tilde{\bphi}_t(i)\defeq\sqrt{\dot{\ell}\left(\tilde{\btheta}_{t,i}^\top\bphi_{t}(i)\right)}\bphi_t(i)$, 
inequality (b) is due to the elliptical potential lemma (\cref{apdx_lem:elliptical_potential}).

\begin{align}
    \text{term 2} &= \E\left[\sum_{t\in [T]}B_1\sum_{i\in \tau_t}  3(1+2L)^2 \kappa \gamma_t^2(\delta) \norm{\bphi_t(i)} ^2  _{\bV_t^{-1}}\right]\\
    &\le  3(1+2L)^2 \kappa \gamma_T^2(\delta)\E\left[\sum_{t\in [T]}\sum_{i\in \tau_t}  \norm{\bphi_t(i)} ^2  _{\bV_t^{-1}}\right]\\
    &\overset{(a)}{\le}  3(1+2L)^2 (2L+3)^2 B_1\kappa \left( \lbdd \right)^2
\end{align}
where inequality (a) is due to the elliptical potential lemma (\cref{apdx_lem:elliptical_potential}).
\end{proof}

Let $r_{\max}=\max_{t\in[T]} \alpha r(S_t^*;\bmu_t)$. Now we set $\delta=\frac{1}{2T}$ and consider the regret caused by the failure of the event $\{\forall t \geq 1, \btheta^* \in \cA_t(\delta)\} \leq \gamma_t(\delta)\}$ and the event $\{\forall t \geq 1, \btheta^* \in \cC_t(\delta)\}$. We have:
\begin{align}
    \text{Reg}(T)&\le r_{\max}T \cdot \frac{2}{2T} + \text{term 1} + \text{term 2}\\
    &\le r_{\max} + (4+8L)(2L+3) B_v\sqrt{T}\lbdd \notag\\
    &+  3(1+2L)^2 (2L+3)^2 B_1\kappa \left( \lbdd \right)^2 \\
    &= O\left( B_v d\sqrt{T} \log(KT) + B_1 \kappa d^2 \log ^2(KT)\right)
\end{align}

\subsection{Regret Bound for Alg. \ref{alg:EVA_CLogB} under 1-Norm TPM Smoothness Condition and TPVM Smoothness Condition}\label{apdx_sec:proof_thm_eva}
\begin{proof}[Proof of \cref{thm:eva_thm1}]

Let $r_{\max}=\max_{t\in[T]} \alpha r(S_t^*;\bmu_t)\le B_1 K$.
In this proof, we first bound the regret in the burn-in stage
\begin{align}
    \text{term 0}&=\sum_{t=1}^{T_0} r_{\max} = r_{\max} \cdot \left(4L^2+16L+19\right)^2 \kappa d^2 \log^2\left( \frac{4(2+T)}{\delta}\right)\\
    &\le B_1 K \left(4L^2+16L+19\right)^2 \kappa d^2 \log^2\left( \frac{4(2+T)}{\delta}\right)\label{apdx_eq:burn_in_reg}
\end{align}

Next, we bound the regret after the burn-in stage for $t= T_0+1,..., T$.

Suppose the event $\{\btheta^* \in \cQ\} \cap \{\text{diam}(\cQ) \le 1\}$ and the event $\{\forall t\ge T_0+1, \btheta^* \in \mathcal{D}_t(\delta)\}$ hold.

\textbf{Under the 1-norm TPM condition:}

We use the same derivation from \cref{apdx_eq:1-norm_reg_start} to \cref{apdx_eq:1-norm_reg_mid}:

\begin{align}
    &\text{Reg}(T)\le \E\left[\sum_{t=T_0+1}^T\E_t\left[\sum_{i \in \tau_t}B_1(\bar{\mu}_{t,i}-\mu_{t,i})\right]\right]\overset{(a)}{\le}\E\left[\sum_{t=T_0+1}^T\E_t\left[\sum_{i \in \tau_t}B_1(2\rho_{t,E}(i))\right]\right]\notag\\
    &=\E\left[\sum_{t=T_0+1}^T\E_t\left[\sum_{i \in \tau_t}B_1\left(6\sqrt{e} \dot{\ell}\left(\hat{\btheta}_{t,E}^{\top}\bphi_t(i)\right)\|\bphi_t(i)\|_{\mathbf{H}_t^{-1}(\hat{\btheta}_{t,E})} \gamma_t(\delta)+\frac{9}{4}\kappa \gamma_t^2(\delta)\|\bphi_t(i)\|_{\mathbf{V}_t^{-1}}^2\right)\right]\right]\\
    &\overset{(b)}{\le} \E\left[\sum_{t=T_0+1}^T\E_t\left[\sum_{i \in \tau_t}B_1\left(6e^2 \dot{\ell}\left({\btheta^*}^{\top}\bphi_t(i)\right)\|\bphi_t(i)\|_{\mathbf{H}_t^{-1}({\btheta^*})} \gamma_t(\delta) + \frac{9}{4}\kappa \gamma_t^2(\delta)\|\bphi_t(i)\|_{\mathbf{V}_t^{-1}}^2\right)\right]\right]\\
    &{\le} \E\left[\sum_{t=T_0+1}^T\E_t\left[\sum_{i \in \tau_t}B_1\left(6e^2 \sqrt{\dot{\ell}\left({\btheta^*}^{\top}\bphi_t(i)\right)}\|\bphi_t(i)\|_{\mathbf{H}_t^{-1}({\btheta^*})} \gamma_t(\delta) + \frac{9}{4}\kappa \gamma_t^2(\delta)\|\bphi_t(i)\|_{\mathbf{V}_t^{-1}}^2\right)\right]\right]\\
    &\overset{(c)}{=} \E\left[\sum_{t=T_0+1}^T\sum_{i \in \tau_t}B_1\left(6e^2 \sqrt{\dot{\ell}\left({\btheta^*}^{\top}\bphi_t(i)\right)}\|\bphi_t(i)\|_{\mathbf{H}_t^{-1}({\btheta^*})} \gamma_t(\delta) + \frac{9}{4}\kappa \gamma_t^2(\delta)\|\bphi_t(i)\|_{\mathbf{V}_t^{-1}}^2\right)\right]\\
&\overset{}{\le}  6e^2 B_1 \gamma_T(\delta) \E\left[ \sum_{t=T_0+1}^T \sum_{i \in \tau_t} \sqrt{\dot{\ell} \left({\btheta^*}^{\top} \bphi_t(i)\right)} \|\bphi_t(i)\|_{\mathbf{H}_t^{-1}({\btheta^*})}\right] + \frac{9}{4} B_1 \kappa \gamma^2_T(\delta)\E\left[\sum_{t=T_0+1}^T \sum_{i \in \tau_t} \|\bphi_t(i)\|^2_{\mathbf{V}_t^{-1}}\right]\\
&\overset{(d)}{=}  6e^2 B_1 \gamma_T(\delta) \E\left[ \sum_{t=T_0+1}^T \sum_{i \in \tau_t}  \|\tilde{\bphi}_t(i)\|_{\mathbf{H}_t^{-1}({\btheta^*})}\right] + \frac{9}{4} B_1 \kappa \gamma^2_T(\delta)\E\left[\sum_{t=T_0+1}^T \sum_{i \in \tau_t} \|\bphi_t(i)\|^2_{\mathbf{V}_t^{-1}}\right]\\
&\le  6e^2 B_1 \gamma_T(\delta)\sqrt{KT}\E\left[\sqrt{\sum_{t=T_0+1}^T\sum_{i \in \tau_t}\|\tilde{\bphi}_t(i)\|^2_{\mathbf{H}_t^{-1}(\btheta^*)}}\right]+\frac{9}{4}B_1 \kappa\gamma^2_T(\delta)\E\left[\sum_{t=T_0+1}^T\sum_{i \in \tau_t}\|\bphi_t(i)\|^2_{\mathbf{V}_t^{-1}}\right]\\
&\overset{(e)}{\le}  6e^2 (2L+3)B_1 \sqrt{KT} \lbdd +\frac{9}{4}(2L+3)^2 B_1\kappa  \left(\lbdd \right)^2 \\ 
&= O\left(B_1 d\sqrt{KT}\log (KT)+ B_1 \kappa d^2\log ^2(KT)\right).
\end{align}
where inequality (a) is due to variance-adaptive exploration bonus lemma after burn-in stage (\cref{lem:eva_bonus}), 
inequality (b) is due to \cref{apdx_eq:hat_lower_than_s} and \cref{apdx_eq:link_H*_H},
inequality (c) is due to the tower law of expectation,
equality (d) is by defining $\tilde{\bphi}_t(i)\defeq\sqrt{\dot{\ell}\left({\btheta^*}^\top\bphi_{t}(i)\right)}\bphi_t(i)$,
inequality (e) is by the elliptical potential lemma (\cref{apdx_lem:elliptical_potential})

Recall that $r_{\max}=\max_{t\in[T]} \alpha r(S_t^*;\bmu_t)\le B_1 K$. Now we set $\delta=\frac{1}{2T}$. Consider the regret (term 0) caused by the burn-in stage in \cref{apdx_eq:burn_in_reg} and the failure of the event $\{\btheta^* \in \cQ\} \cap \{\text{diam}(\cQ) \le 1\}$ and the event $\{\forall t\ge T_0+1, \btheta^* \in \mathcal{D}_t(\delta)\}$. We have:
\begin{align}
    \text{Reg}(T)&\le r_{\max}T \cdot \frac{2}{2T} + \text{term 0} + 6e^2 (2L+3)B_1 \sqrt{KT} \lbddt \notag\\
    &+\frac{9}{4}(2L+3)^2 B_1\kappa  \left(\lbddt \right)^2\\
    &\le B_1 K + B_1 K \left(4L^2+16L+19\right)^2 \kappa d^2 \log^2\left( {8(2+T)T}\right) \notag\\
    &+ 6e^2 (2L+3)B_1 \sqrt{KT} \lbddt \notag\\
    &+\frac{9}{4}(2L+3)^2 B_1\kappa  \left(\lbddt \right)^2 \\
    &= O\left( B_v d\sqrt{T} \log(KT) + B_1 \kappa K d^2 \log ^2(KT)\right)\label{apdx_eq:eva_reg_tpm}
\end{align}

\textbf{Under the 1-norm TPVM condition}:

We follow the same derivation of \cref{apdx_eq:TPVM_start}:

\begin{align}
    \text{Reg}(T)&\overset{(c)}{\le} \E\left[\sum_{t=1}^T r(S_t; \bar{\bmu}_t)- r(S_t; \bmu_t)\right]
\end{align}

By \cref{lem:var_ad_bonus}, we upper bound the value of $\bar{\bmu}_{t}$:
\begin{align}
    \bar{\mu}_{t,i} &\le \mu_{t,i} + 2\rho_{t,E}\\
    &= 6\sqrt{e} \dot{\ell}\left(\hat{\btheta}_{t,E}^{\top}\bphi_t(i)\right)\|\bphi_t(i)\|_{\mathbf{H}_t^{-1}(\hat{\btheta}_{t,E})} \gamma_t(\delta)+\frac{9}{4}\kappa \gamma_t^2(\delta)\|\bphi_t(i)\|_{\mathbf{V}_t^{-1}}^2\\
    &\overset{(a)}{\le} \mu'_{t,i} \defeq 6e^2 \dot{\ell}\left({\btheta^*}^{\top}\bphi_t(i)\right)\|\bphi_t(i)\|_{\mathbf{H}_t^{-1}({\btheta^*})} \gamma_t(\delta) + \frac{9}{4}\kappa \gamma_t^2(\delta)\|\bphi_t(i)\|_{\mathbf{V}_t^{-1}}^2 
\end{align}
where inequality (a) is due to  \cref{apdx_eq:hat_lower_than_s} and \cref{apdx_eq:link_H*_H}.

Let $\bmu'_t=(\mu'_{t,1},..., \mu'_{t,m})$ and by monotonicity,
\begin{align}
    \text{Reg}(T)&\le \E\left[\sum_{t=T_0+1}^T r(S_t; \bar{\bmu}_t)- r(S_t; \bmu_t)\right]\\
    &\le \E\left[\sum_{t=T_0+1}^T r(S_t; \bmu'_t)- r(S_t; \bmu_t)\right]\\
    &\overset{(a)}{\le} \E\left[\sum_{t=T_0+1}^T B_v \sqrt{ \sum_{i \in [m]} (p_i^{\bmu_t,S_t})^{\lambda} \frac{(6e^2)^2\gamma^2_t(\delta) \dot{\ell}\left(\bphi_t(i)^{\top}{\btheta^*}\right)^2 \norm{\bphi_t(i)}^2 _{\bH_t^{-1}({\btheta^*})}}{(1-\mu_{t,i})\mu_{t,i}}}\right]\notag\\
    &+ \E\left[\sum_{t=T_0+1}^T B_1\sum_{i\in [m]}  p_i^{\bmu_t,S_t}\frac{9}{4} \kappa \gamma_t^2(\delta) \norm{\bphi_t(i)} ^2  _{\bV_t^{-1}}\right]\\
    &\overset{(b)}{=} \underbrace{\E\left[\sum_{t=T_0+1}^T B_v(6e^2) \gamma_t(\delta) \sqrt{\sum_{i \in [m]}(p_i^{\bmu_t,S_t})^{\lambda} {\dot{\ell}\left(\bphi_t(i)^{\top}\btheta^*\right)} \norm{\bphi_t(i)}^2 _{\bH_t^{-1}(\btheta^*)}}\right]}_{\text{term 1}}\notag\\
    &+ \underbrace{\E\left[\sum_{t=T_0+1}^T B_1\sum_{i\in [m]} p_i^{\bmu_t,S_t}\frac{9}{4} \kappa \gamma_t^2(\delta) \norm{\bphi_t(i)} ^2  _{\bV_t^{-1}}\right]}_{\text{term 2}}\\
\end{align}
where inequality (a) is due to \cref{cond:TPVMm} by setting $\zeta_i=6e^2 \dot{\ell}\left({\btheta^*}^{\top}\bphi_t(i)\right)\|\bphi_t(i)\|_{\mathbf{H}_t^{-1}({\btheta^*})} \gamma_t(\delta), \eta_i=\frac{9}{4}\kappa \gamma_t^2(\delta)\|\bphi_t(i)\|_{\mathbf{V}_t^{-1}}^2 $,
equality (b) is due to $\dot{\ell}\left(\bphi_t(i)^{\top}\btheta^*\right)=\left(1- \ell\left(\bphi_t(i)^{\top}\btheta^*\right) \right)\ell\left(\bphi_t(i)^{\top}\btheta^*\right)  =(1-\mu_{t,i})\mu_{t,i}$.

For term 1, we have:
\begin{align}
    \text{term 1} &\overset{(a)}{\le}  B_v(6e^2) \gamma_T(\delta)\E\left[\sum_{t=T_0+1}^T  \sqrt{\sum_{i \in [m]}p_i^{\bmu_t,S_t} {\dot{\ell}\left(\bphi_t(i)^{\top}\btheta^*\right)}  \norm{\bphi_t(i)}^2 _{\bH_t^{-1}({\btheta^*})}}\right]\\
    &\overset{(b)}{\le}  B_v(6e^2) \gamma_T(\delta)\sqrt{T}\E\left[\sqrt{\sum_{t=T_0+1}^T \sum_{i \in [m]}p_i^{\bmu_t,S_t} {\dot{\ell}\left(\bphi_t(i)^{\top}\btheta^*\right)}  \norm{\bphi_t(i)}^2 _{\bH_t^{-1}({\btheta^*})}}\right]\\
    &\overset{(c)}{\le}  B_v(6e^2) \gamma_T(\delta)\sqrt{T}\sqrt{\E\left[\sum_{t=T_0+1}^T  \sum_{i \in [m]}p_i^{\bmu_t,S_t} {\dot{\ell}\left(\bphi_t(i)^{\top}\btheta^*\right)}  \norm{\bphi_t(i)}^2 _{\bH_t^{-1}({\btheta^*})}\right]}\\
    &\overset{(d)}{=}  B_v(6e^2) \gamma_T(\delta)\sqrt{T}\sqrt{\E\left[\sum_{t=T_0+1}^T  \sum_{i \in \tau_t} {\dot{\ell}\left(\bphi_t(i)^{\top}\btheta^*\right)}  \norm{\bphi_t(i)}^2 _{\bH_t^{-1}({\btheta^*})}\right]}\\
    &\overset{(e)}{=}  B_v(6e^2) \gamma_T(\delta)\sqrt{T}\sqrt{\E\left[\sum_{t=T_0+1}^T  \sum_{i \in \tau_t}  \norm{\tilde{\bphi}_t(i)}^2 _{\bH_t^{-1}({\btheta^*})}\right]}\\
    &\overset{(f)}{\le}  (6e^2)(2L+3) B_v\sqrt{T}\lbdd
\end{align}
where inequality (a) is due to $\lambda \ge 1$, inequality (b) is due to Cauchy-Schwarz inequality, 
inequality (c) is due to Jensen's inequality,
equality (d) is due to $\E_t[i\in \tau_t]=p_{i}^{\bmu_t, \bS_t}$ and the fact that $S_t, \bH_t, \tilde{\btheta}_{t,i}, \hat{\btheta}_{t,H}$ are $\cH_t$ measurable,
equality (e) is by defining $\tilde{\bphi}_t(i)\defeq\sqrt{\dot{\ell}\left({\btheta^*}^\top\bphi_{t}(i)\right)}\bphi_t(i)$.

\begin{align}
    \text{term 2} &= \E\left[\sum_{t=T_0+1}^TB_1\sum_{i\in \tau_t}  \frac{9}{4} \kappa \gamma_t^2(\delta) \norm{\bphi_t(i)} ^2  _{\bV_t^{-1}}\right]\\
    &\le  \frac{9}{4} \kappa \gamma_T^2(\delta)\E\left[\sum_{t=T_0+1}^T\sum_{i\in \tau_t}  \norm{\bphi_t(i)} ^2  _{\bV_t^{-1}}\right]\\
    &\overset{(a)}{\le}  \frac{9}{4} (2L+3)^2 B_1\kappa \left( \lbdd \right)^2
\end{align}
where inequality (a) is due to the elliptical potential lemma (\cref{apdx_lem:elliptical_potential}).

Recall that $r_{\max}=\max_{t\in[T]} \alpha r(S_t^*;\bmu_t)\le B_1 K$. Now we set $\delta=\frac{1}{2T}$. Consider the regret caused by the burn-in stage (term 0) in \cref{apdx_eq:burn_in_reg} and the failure of the event $\{\btheta^* \in \cQ\} \cap \{\text{diam}(\cQ) \le 1\}$ and the event $\{\forall t\ge T_0+1, \btheta^* \in \mathcal{D}_t(\delta)\}$. We have:
\begin{align}
    \text{Reg}(T)&\le r_{\max}T \cdot \frac{2}{2T} + \text{term 0} + \text{term 1} + \text{term 2}\\
    &\le B_1 K + B_1 K \left(4L^2+16L+19\right)^2 \kappa d^2 \log^2\left( \frac{4(2+T)}{\delta}\right) \notag\\
    &+ (6e^2)(2L+3) B_v\sqrt{T}\lbdd \notag\\
    &+  \frac{9}{4} (2L+3)^2 B_1\kappa \left( \lbdd \right)^2 \\
    &= O\left( B_v d\sqrt{T} \log(KT) + B_1 \kappa K d^2 \log ^2(KT)\right)\label{apdx_eq:eva_reg_tpvm}
\end{align}


\end{proof}

\section{Auxiliary Lemmas}\label{apdx_sec:aux}

The following lemma is extracted from Lemma 10 in \citet{abbasi2011improved}.
\begin{lemma}[Determinant-Trace Inequality]\label{apdx_lem:det_and_trace}
Let $\bx_1, ..., x_t \in \R^d$ be a sequence with $\norm{\bx_i}_2\le L$, $1\le i \le t$. 
Let $\bV_t=\lambda \bI_d + \sum_{i=1}^{t} \bx_i \bx_i^{\top}$ for some $\lambda>0$.
Then,
\begin{align*}
 \det(\bV_{t}) \le  (\lambda + tL^2/d)^d.
\end{align*}
\end{lemma}

The following lemma is extracted from Proposition 7 in \citet{abeille2021instance}. 
\begin{lemma}[Polynomial Inequality]\label{apdx_lem:Poly_ieq}
    Let $b, c \in \mathbb{R}^{+}$, and $x \in \mathbb{R}$. The following implication holds:
$$
x^2 \leq b x+c \Longrightarrow x \leq b+\sqrt{c}
$$
\end{lemma}

The following lemma is from Lemma 7, Lemma 8, and Lemma 9 in \citet{abeille2021instance}.
\begin{lemma}[Self-concordance control]\label{apdx_lem:scc}
If the link function $\ell$ is twice continuously differentiable and $\abs{\ddot{\ell}}\le \dot{\ell}$, then we have the following list of inequalities: 

Let $\bx \in \R^d$ with $\norm{\bx}_2\le 1$, for any $\btheta_1, \btheta_2 \in \R^d$. Then
\begin{align}\label{apdx_eq:scc1}
    \alpha(\bx,\btheta_1, \btheta_2) &\ge \frac{\dot{\ell}(\bx ^{\top} \btheta)}{1+\abs{\bx^{\top}(\btheta_1-\btheta_2)}}, \quad \text {    for    } \btheta \in \{\btheta_1, \btheta_2\}\\
    &\ge \frac{\dot{\ell}(\bx ^{\top} \btheta)}{1+2L} \quad \text { when } \btheta_1, \btheta_2 \in \Theta. \label{apdx_eq:scc1_2}
\end{align}

Similarly, we have
\begin{align}\label{apdx_eq:scc2}
    \tilde{\alpha}(\bx,\btheta_1, \btheta_2) &\ge \frac{\dot{\ell}(\bx ^{\top} \btheta_1)}{2+\abs{\bx^{\top}(\btheta_1-\btheta_2)}},\\
    &\ge \frac{\dot{\ell}(\bx ^{\top} \btheta_1)}{2+2L} \quad \text { when } \btheta_1, \btheta_2 \in \Theta.
\end{align}

Finally, we have 
\begin{align}\label{apdx_eq:scc3}
    \dot{\ell}( \bx^{\top} \btheta_2) \exp \left( -\abs{\bx^{\top} \left( \btheta_1-\btheta_2 \right) } \right) \le \dot{\ell}( \bx^{\top} \btheta_1) \le \dot{\ell}( \bx^{\top} \btheta_2) \exp \left( \abs{\bx^{\top} \left( \btheta_1-\btheta_2 \right) } \right)
\end{align}
\end{lemma}

The following lemma is extracted from Lemma A.3 in \cite{li2016contextual}.
\begin{lemma}[Determinant-norm Inequality]\label{apdx_lem:det_and_norm}
    Let $\bx_1, ..., \bx_n \in \R^d$ be a sequence. Then we have 
    \begin{align*}
        \det\left(\bI + \sum_{i=1}^n{\bx_i \bx_i^{\top}}\right)\ge 1+\sum_{i=1}^n\norm{\bx_i}_{2}^2.
    \end{align*}
\end{lemma}
\begin{proof}
Denote the eigenvalues of $I+\sum_{i=1}^n \bx_i \bx_i^{\top}$ by $1+$ $\alpha_1, \ldots, 1+\alpha_d$ with $\alpha_j \geq 0,1 \leq j \leq d$. Then
$$
\begin{aligned}
& \operatorname{det}\left(I+\sum_{i=1}^n \bx_i \bx_i^{\top}\right)=\prod_{j=1}^d\left(1+\alpha_j\right) \\
\geq & 1+\sum_{j=1}^d \alpha_j=1-d+\sum_{i=1}^d\left(1+\alpha_i\right) \\
= & 1-d+\operatorname{tr}\left(I+\sum_{i=1}^n \bx_i \bx_i^{\top}\right)\\
=&1-d+d+\sum_{i=1}^n\left\|\bx_i\right\|_2^2 \\
= & 1+\sum_{i=1}^n\left\|\bx_i\right\|_2^2 .
\end{aligned}
$$
\end{proof}

\begin{lemma}[Elliptical potential with multiple arms and time varying regularization.]\label{apdx_lem:elliptical_potential}
Let $m,K,T>0$ and $\{\tau_t\}_{t=1}^{\infty}$ be a sequence of indices such that $\tau_t \subseteq [m]$ and $|\tau_t|\le K$ for all $t \in  [T]$.
Let $\{\bx_{t,i}\}_{t\in [T], i\in \tau_t}$ be a sequence in $\R^d$ such that $\norm{\bx_{t,i}}_2\le 1$ for all $t\in[T], i\in \tau_t$.
Let $\{\lambda_t\}_{t \in [T]}$ be an increasing sequence in $\R^+$ s.t., $\lambda_1 \ge K$. 
For $t\ge 1$ define $\bV_t=\sum_{s=1}^{t-1}\sum_{i\in \tau_s} \bx_{s,i}\bx_{s,i}^{\top} + \lambda_t \bI_d$. Then for $t\in [T]$, we have:
\begin{align}
    \sum_{s=1}^{t}\sum_{i \in \tau_s}\norm{\bx_{s,i}}_{\bV_s^{-1}}^2 \le 2d\log(\lambda_{t+1}+t)
\end{align}
\end{lemma}
\begin{proof}

  \begin{align}
        \det(\bV_{t+1})&\overset{(a)}{=}\det\left(\lambda_{t+1} \bI_d + \sum_{s=1}^{t-1}\sum_{i\in \tau_s} \bx_{s,i}\bx_{s,i}^{\top} + \sum_{i\in \tau_{t}}\bx_{t,i}\bx_{t,i}^{\top}\right)\notag\\
        &\overset{(b)}{\ge}\det\left(\lambda_{t} \bI_d + \sum_{s=1}^{t-1}\sum_{i\in \tau_s} \bx_{s,i}\bx_{s,i}^{\top} + \sum_{i\in \tau_{t}}\bx_{t,i}\bx_{t,i}^{\top}\right) \notag\\
        &=\det\left(\bV_{t} + \sum_{i\in \tau_{t}}\bx_{t,i}\bx_{t,i}^{\top}\right) \notag\\
        &\overset{(c_1)}{=}\det(\bV_t^{1/2}) \det\left(\bI_d + \sum_{i\in \tau_{t}}\bV_{t}^{-1/2}\bx_{t,i} (\bV_{t}^{-1/2}\bx_{t,i})^{\top}\right) \det(\bV_t^{1/2})\notag\\
        &\overset{(c_2)}{=}\det(\bV_{t})\cdot \det\left(\bI_d + \sum_{i\in \tau_{t}}\bV_{t}^{-1/2}\bx_{t,i} (\bV_{t}^{-1/2}\bx_{t,i})^{\top}\right)\notag\\
        &\overset{(d)}{\ge} \det(\bV_{t})\cdot \left(1 + \sum_{i\in \tau_{t}}\norm{\bx_{t,i}}_{\bV_{t}^{-1}}^2\right)\notag\\
        &\overset{(e)}{\ge} \det(\lambda_1 \bI_d) \prod_{s=1}^t \left(1+\sum_{i \in \tau_s}\norm{\bx_{s,i}}_{\bV_s^{-1}}^2\right)\label{apdx_eq:prod_and_det},
    \end{align}
    where inequality (a) is due to the definition of $\bV_{t+1}$,
    inequality (b) is due to $\lambda_{t+1} \ge \lambda_t$ by assumption,
    equality ($c_1$) and ($c_2$) are due to $\det(\bA\bB)=\det(\bA)\det(\bB)$,
    inequality (d) is due to \Cref{apdx_lem:det_and_norm}, 
    inequality (e) follows from repeatedly applying (d).

    Since $\norm{\bx_{s,i}}_{\bV_s^{-1}}^2\le \frac{\norm{\bx_{s,i}}^2}{\lambda_{\min}(\bV_s)}\le 1/\lambda_s\le 1/\lambda_1$, where $\lambda_{\min}(\bV_s)$ is the minimum eigenvalue of $\bV_s$, we have $\sum_{i \in \tau_s}\norm{\bx_{s,i}}_{\bV_s^{-1}}^2\le \frac{K}{\lambda_1}$. Using the fact that $2\log (1+x)\ge x$ for any $x\in[0,1]$, we have:
    \begin{align*}
        &\sum_{s=1}^{t}\frac{1}{\max(1, K/\lambda_1)}\sum_{i \in \tau_s}\norm{\bx_{s,i}}_{\bV_s^{-1}}^2\\
        &{\le} 2\sum_{s=1}^t\log \left(1+\frac{1}{\max(1, K/\lambda_1)}\sum_{i \in \tau_s}\norm{\bx_{s,i}}_{\bV_s^{-1}}^2\right)\\
        &{\le} 2\sum_{s=1}^t\log \left(1+\sum_{i \in \tau_s}\norm{\bx_{s,i}}_{\bV_s^{-1}}^2\right)\\
        &= 2\log \prod_{s=1}^t\left(1+\sum_{i \in \tau_s}\norm{\bx_{s,i}}_{\bV_s^{-1}}^2\right)\\
        &\overset{(a)}{\le} 2\log\left(\frac{\det(\bV_{t+1})}{\det(\lambda_1 \bI_d)}\right)\\
        &\overset{(b)}{\le} 2\log \left(\frac{(\lambda_{t+1} + Kt/d)^d}{\lambda_1^d}\right)=2d\log \left(\frac{\lambda_{t+1}}{\lambda_1}+\frac{Kt}{\lambda_1 d}\right),
    \end{align*}
    where inequality (a) follows from \Cref{apdx_eq:prod_and_det}, inequality (b) follows from \Cref{apdx_lem:det_and_trace}.

    By the assumption that $\lambda_1 \ge K$, we have:
    \begin{align}
        \sum_{s=1}^{t}\sum_{i \in \tau_s}\norm{\bx_{s,i}}_{\bV_s^{-1}}^2\le 2{\max(1, K/\lambda_1)} d\log \left(\frac{\lambda_{t+1}}{\lambda_1}+\frac{Kt}{\lambda_1 d}\right) \le 2d\log(\lambda_{t+1}+t)
    \end{align}
    \end{proof}

\section{Extended Experiments}\label{app:extended}

\begin{figure*}[!th]
    \centering
    \begin{subfigure}[t]{0.45\textwidth}
    \includegraphics[width=0.99\textwidth]{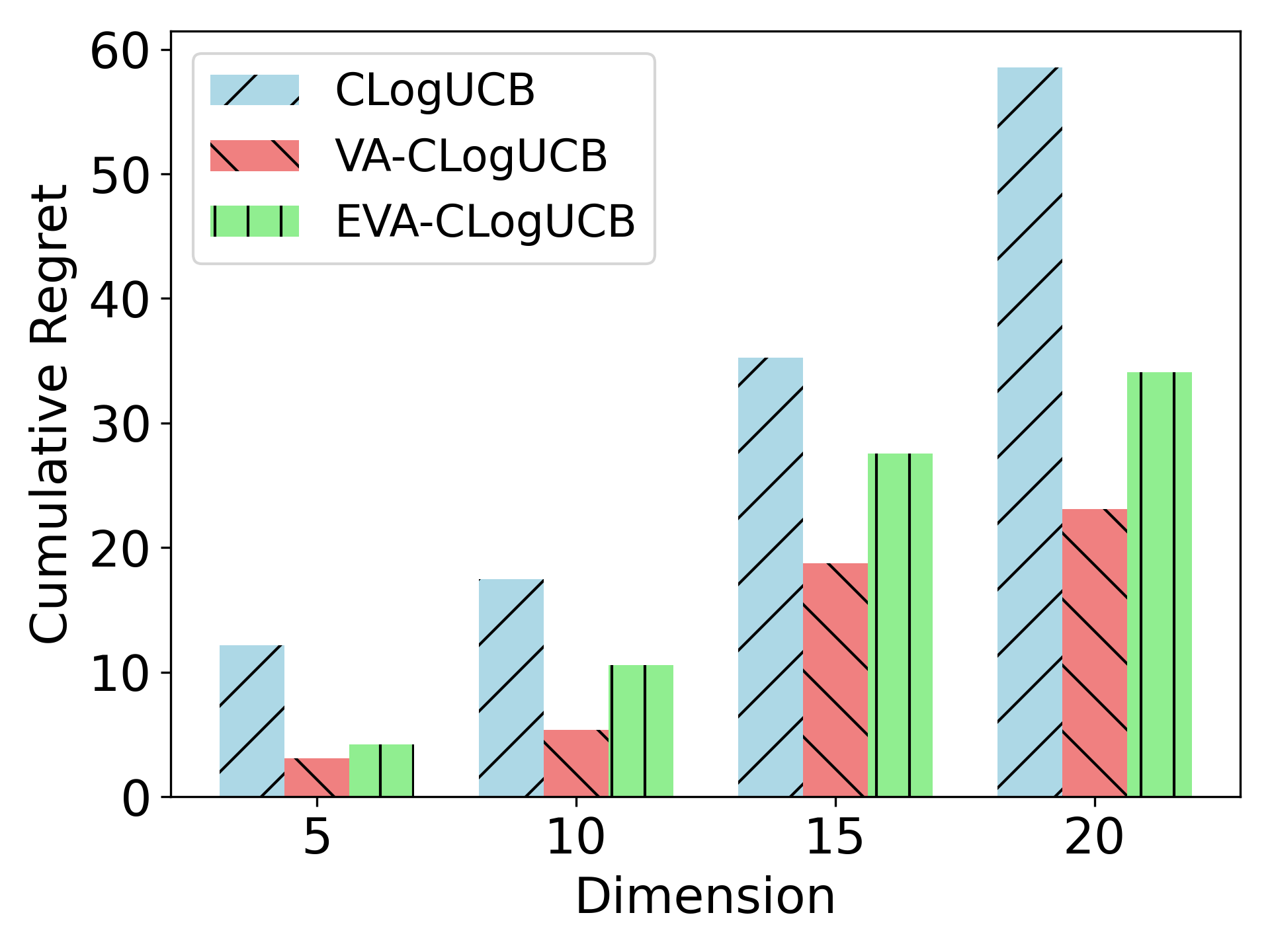}
    \caption{Regret under Varying  $d$ with $T=1000$}
    \label{fig:regret_d}
    \end{subfigure}
    \begin{subfigure}[t]{0.45\textwidth}
    \includegraphics[width=0.99\textwidth]{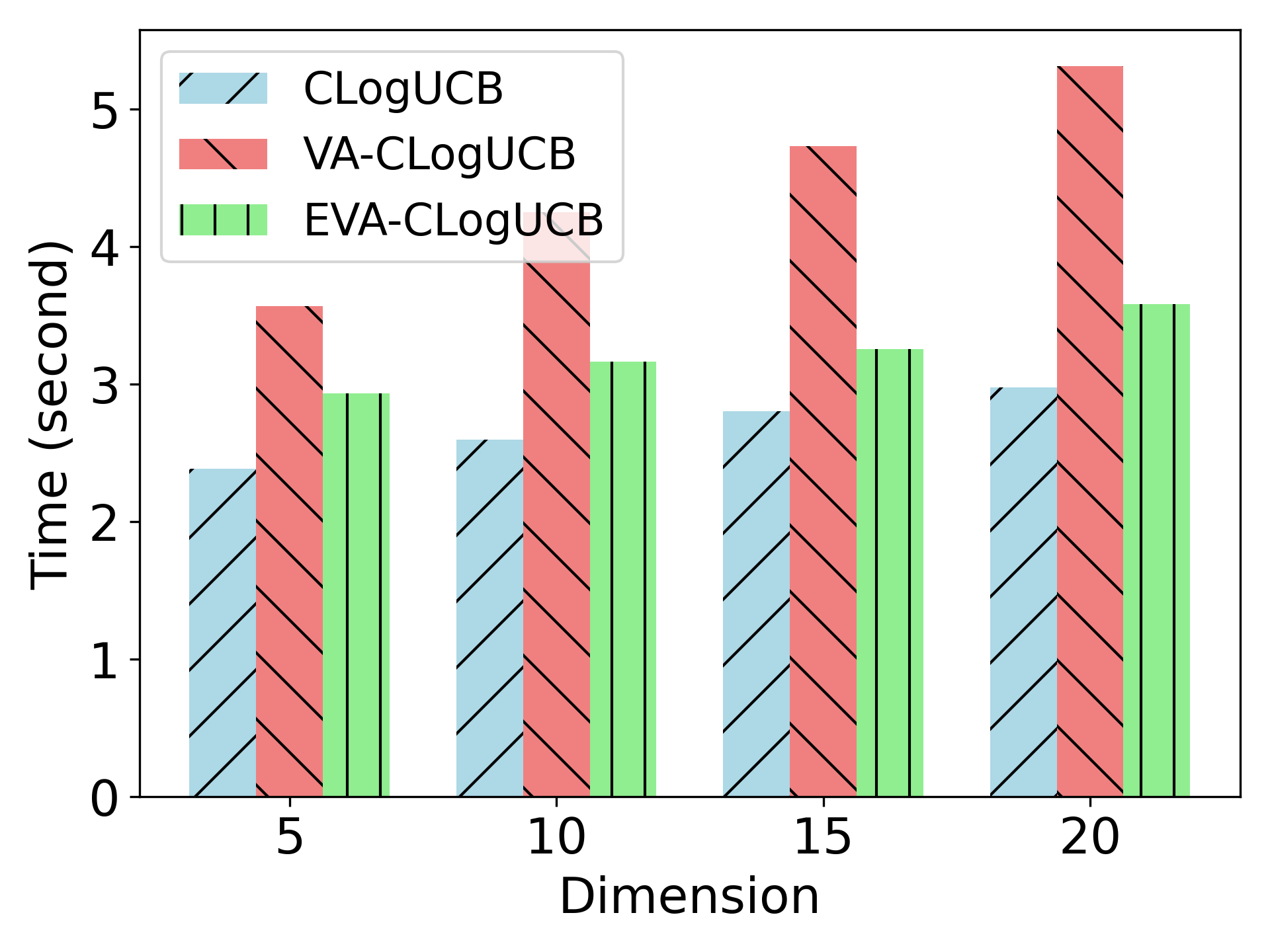}
    \caption{Running Time under Varying $d$ with $T=1000$}
    \label{fig:running_time}
    \end{subfigure}
\caption{Regret vs. running time performance under different dimensions.}
  \label{fig:reg}
  \vspace{-0.1in}
\end{figure*}

\rev{\textbf{Ablation Study on Different Instance Scale.}
Following the settings in Section \ref{sec:Cascading}, we conduct experiments on cascading bandits for the online learning-to-rank application at another different instance scale, where the goal is to select \(K=5\) items from a set of \(m=20\) to maximize the reward.
Fig.~\ref{fig:regret_d} and Fig.~\ref{fig:running_time} compare the regret and total running time after 1000 rounds of our algorithms, i.e., CLogUCB, VA-CLogUCB, and EVA-CLogUCB, across various dimensions \(d\), confirming our theoretical predictions. VA-CLogUCB consistently achieves the lowest regret, reducing it by at least 9.09\% compared to CLogUCB and by 6.93\% compared to EVA-CLogUCB. In terms of running time, CLogUCB is the fastest, completing tasks at least 23.08\% quicker than VA-CLogUCB and 13.84\% faster than EVA-CLogUCB. EVA-CLogUCB offers a balanced performance, positioning itself between the two in both regret and running time. Additionally, as \(d\) increases from 5 to 20, EVA-CLogUCB's running time improvement over VA-CLogUCB grows from 17.78\% to 35.6\%.}

\begin{table}[h]
    \centering
     \caption{Comprehensive comparison of cumulative regret and running time (s) for all different algorithms under $T=2000$  ($K=10, K=15,$ and $K=20$).}
    \label{tab:regret_time_K10_K15_K20}
    \resizebox{1\columnwidth}{!}{
    \begin{tabular}{|c|c|c|c|c|c|c|c|}
        \hline
        \makebox[0.16\textwidth][c]{\diagbox{\tiny{Result}}{\tiny{Algorithm}}} & \scriptsize \underline{\textbf{VA-CLogUCB}} & \scriptsize \underline{\textbf{CLogUCB}} & \scriptsize \underline{\textbf{EVA-CLogUCB}} & \scriptsize \textbf{C$^3$UCB} & \scriptsize \textbf{$\epsilon$-greedy} & \scriptsize \textbf{CUCB} & \scriptsize \textbf{VAC$^2$-UCB} \\ \hline \hline
        \scriptsize \textbf{Regret \tiny{(K=10)}} & 192.1 & 455.3 & 366.1 & 1042.1 & 1418.4 & 1491.0 & 849.3 \\ \hline
        \scriptsize \textbf{Time \tiny{(K=10)}} & 235.1 & 146.3 & 208.0 & 117.5 & 54.6 & 42.3 & 119.9 \\ \hline \hline
        \scriptsize \textbf{Regret \tiny{(K=15)}} & 158.9 & 375.1 & 307.9 & 859.1 & 1080.2 & 1163.2 & 719.3 \\ \hline
        \scriptsize \textbf{Time \tiny{(K=15)}} & 264.1 & 169.7 & 225.3 & 130.1& 67.1  & 55.7 & 133.9 \\ \hline \hline
        \scriptsize \textbf{Regret \tiny{(K=20)}} & 137.9 & 318.4 & 258.0 & 728.4 & 831.2 & 860.3 & 609.1 \\ \hline
        \scriptsize \textbf{Time \tiny{(K=20)}} & 303.6 & 187.8 & 255.3 &148.1  & 83.0 & 72.7 & 152.6 \\ \hline
    \end{tabular}}
\end{table}

\begin{figure*}[!th]
    \centering
    \begin{subfigure}[t]{0.495\textwidth}
    \includegraphics[width=1\textwidth]{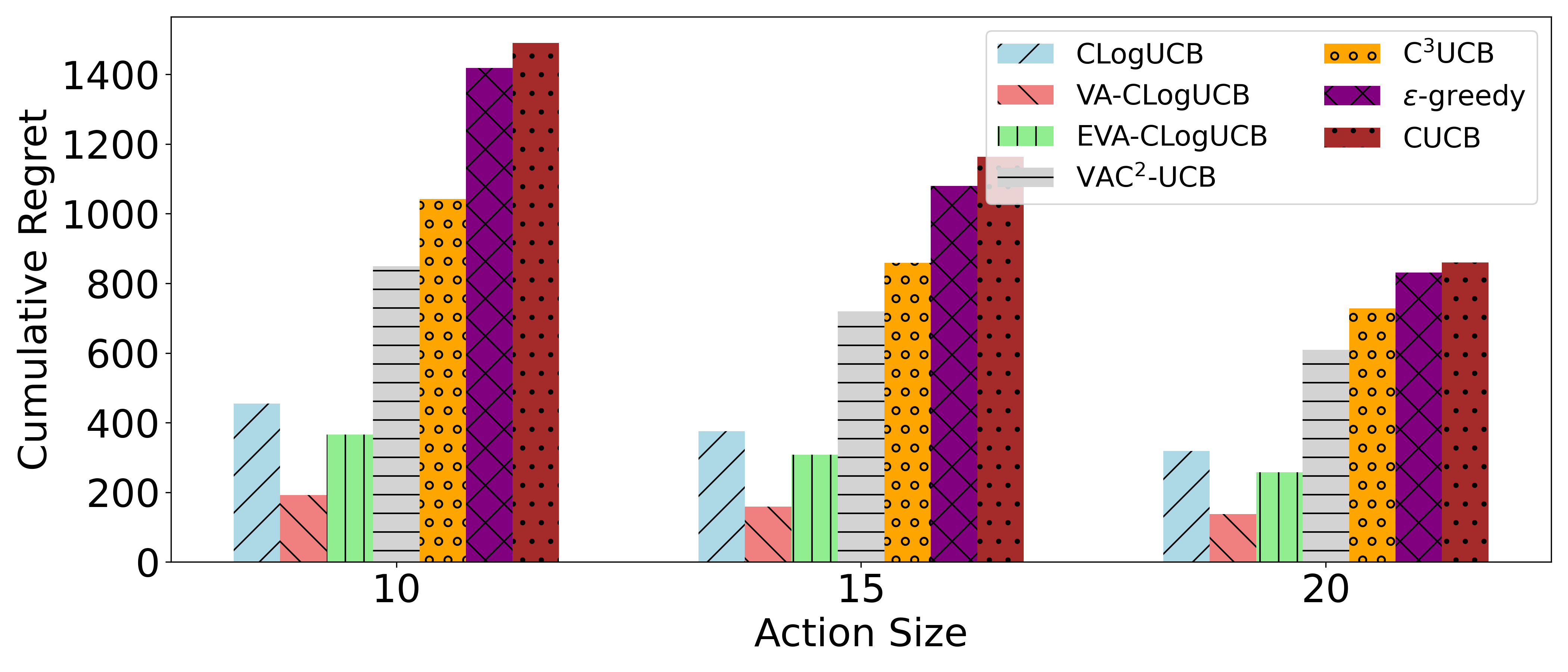}
    \caption{Regret under Varying  $K$ with $T=2000$}
    \label{fig:regret_k_all}
    \end{subfigure}
    \begin{subfigure}[t]{0.495\textwidth}
    \includegraphics[width=1\textwidth]{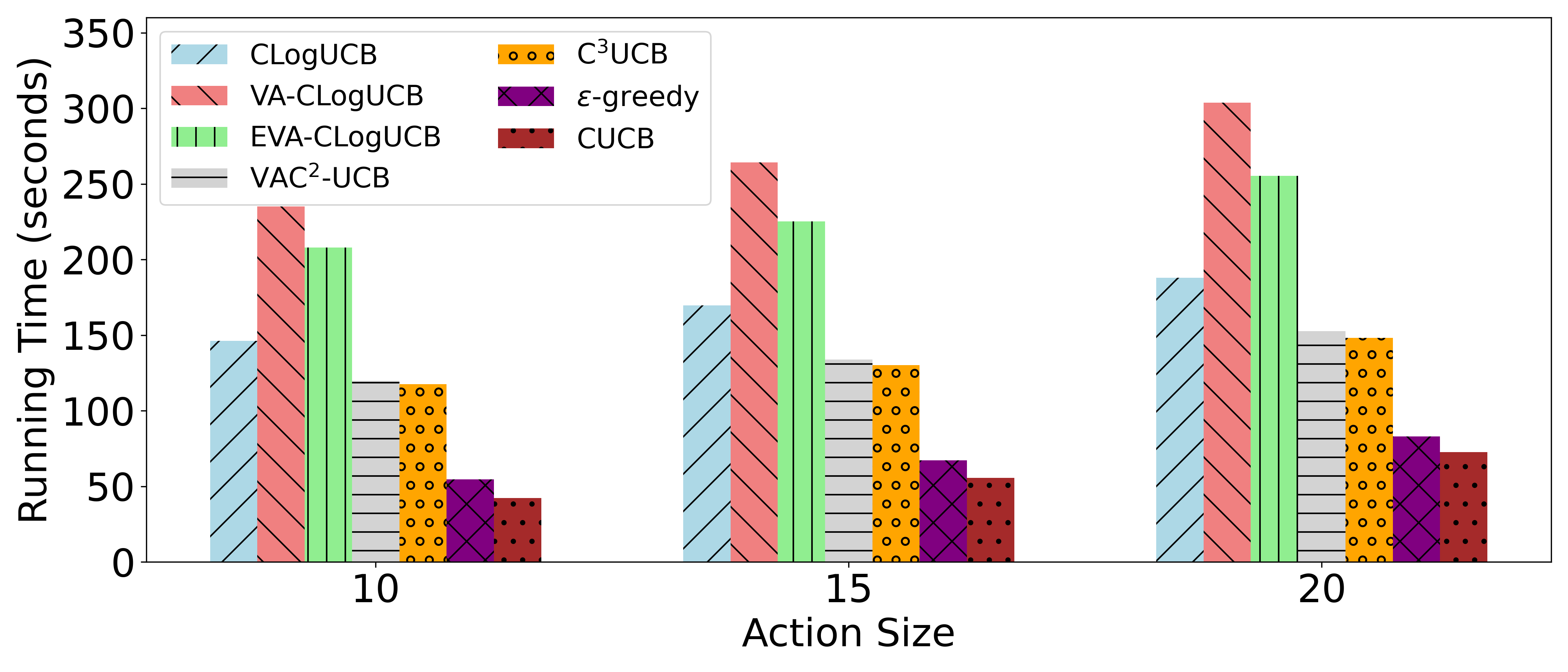}
    \caption{Running Time under Varying $K$ with $T=2000$}
    \label{fig:running_time_k_all}
    \end{subfigure}
\caption{Comprehensive comparison on cumulative regret and running time performance for all algorithms.}
  \label{fig:all_k}
  \vspace{-0.1in}
\end{figure*}

\rev{\textbf{Comprehensive Comparison on Regret and Running Time.}  Table \ref{tab:regret_time_K10_K15_K20} and Fig. \ref{fig:all_k} present the cumulative regret and runing time for all benchmark methods corresponding to the results in Fig. \ref{fig:regret_k}-\ref{fig:running_time1}. The logistic bandits exhibit longer running times because, unlike other methods that iterate through a Gram matrix for easily solving the ridge regression problem, they employ the maximum
likelihood estimation (MLE) and the Newton gradient method with projection operations \cite{filippi2010parametric,zhang2024online}. This approach leads to increased running time, consistent with observations in prior studies on logistic bandits \cite{zhang2024online}. Consistent with Section \ref{sec:simulations}, our proposed algorithms can significantly reduce the cumulative regret for the CLogB problem under different action sizes $K$.}

\begin{figure}[!th]
    \centering
    \includegraphics[width=0.94\textwidth]{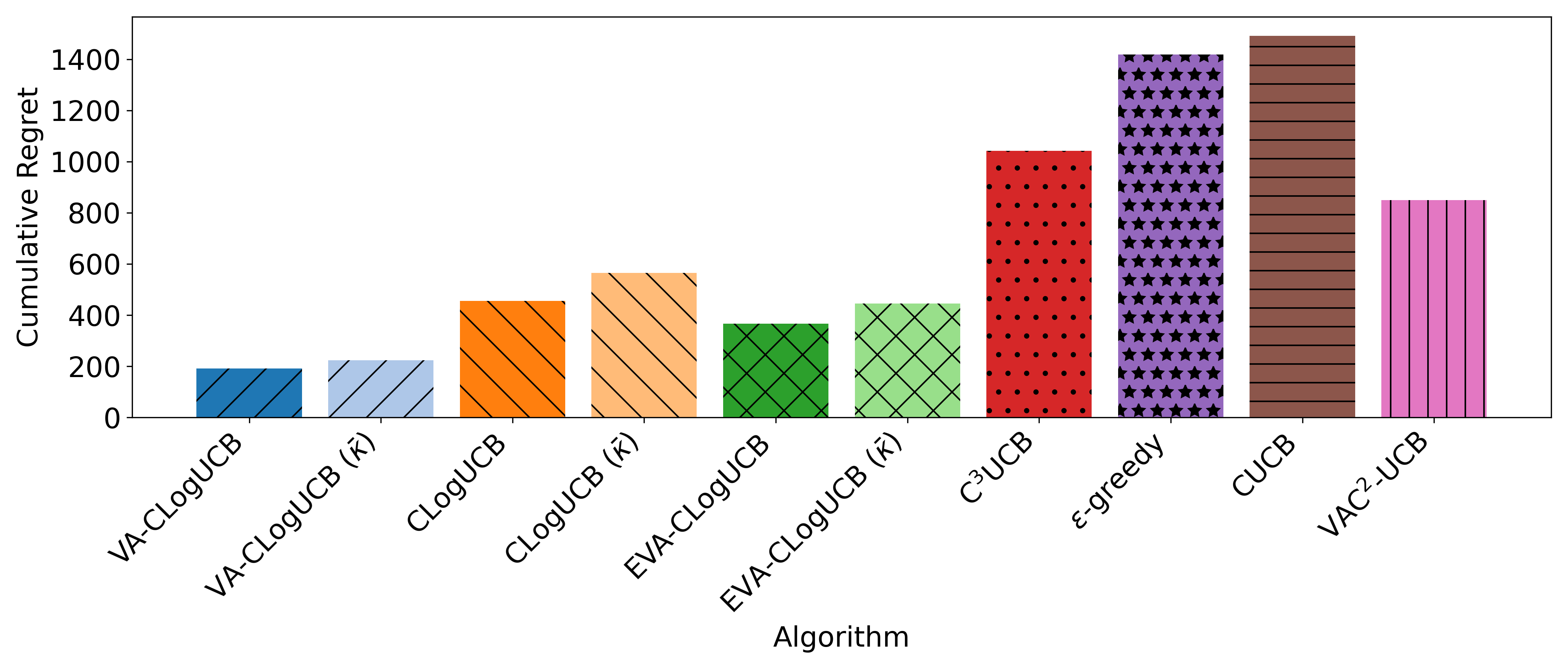}
    \caption{Regret comparison across various algorithms, including those utilizing upper bounds on $\kappa$.}
    \label{fig:kappa}
\end{figure}

\rev{\textbf{Impact of Using Upper Bounds on $\kappa$.}
We conduct experiments using the upper bound $\bar{\kappa} = 4\exp(L)$, where $L$ is an upper bound on the $\ell_2$ norm of the underlying parameter $\btheta^*$, to evaluate the effect of overestimating $\kappa$ on algorithm performance.
Following the setting in Section \ref{sec:Cascading}, the \(l_2\)-norm is defined as:
$
\|\btheta^*\|_2 = \sqrt{(\theta_1^*)^2 + (\theta_2^*)^2 + \ldots + (\theta_d^*)^2}
$, where each component \(\theta_i^*\) is independently drawn from the uniform distribution \(U(-1, 1)\). Therefore, the upper bound for the \(l_2\)-norm of \(\btheta^*\) is
$
\|\btheta^*\|_2 \leq \sqrt{d \cdot 1^2} = \sqrt{d}
$.
Fig. \ref{fig:kappa} presents the cumulative regret performance for $K=10$ and $m=600$. Our results indicate that using a conservative estimate, such as $\bar{\kappa} = 4\exp(L)$, leads to increased regret compared to the true $\kappa$. However, the performance degradation is within an acceptable range, and the impact of $\kappa$ follows the trend: CLogUCB > EVA-CLogUCB > VA-CLogUCB, consistent with the findings in Table \ref{tab:clogb_res}.
This demonstrates that our algorithm maintains robustness even with approximate knowledge of the nonlinearity parameter.}




\end{document}